\documentclass{article}

\usepackage{times}
\usepackage{graphicx}
\usepackage{sidecap}
\usepackage[small]{caption}
\usepackage{subcaption}
\usepackage{amsmath}
\usepackage{amsthm}
\newtheorem{theorem}{Theorem}

\usepackage{multicol}

\usepackage{setspace}
\AtBeginDocument{%
  \addtolength\abovedisplayskip{-0.25\baselineskip}%
  \addtolength\belowdisplayskip{-0.25\baselineskip}%
  \addtolength\abovedisplayshortskip{-0.25\baselineskip}%
  \addtolength\belowdisplayshortskip{-0.25\baselineskip}%
}

\usepackage{pifont}
\usepackage{xspace}

\DeclareMathOperator{\argmin}{argmin}
\DeclareMathOperator{\argmax}{argmax}
\newcommand{\angles}[1]{\langle #1 \rangle}

\usepackage[round]{natbib}

\usepackage{appendix}

\newcommand{\rightcomment}[1]{\(\triangleright\) {\small \it #1}}
\newcommand{\eqcomment}[1]{\addtocounter{equation}{1}\tag*{\rightcomment{#1}\quad(\theequation)}}
\usepackage{suffix}
\WithSuffix\newcommand\eqcomment*[1]{\tag*{\rightcomment{#1}}}

\usepackage{hyperref}

\usepackage[accepted]{icml2019}
\usepackage[noend]{algpseudocode}

\renewcommand\algorithmicdo{:}
\renewcommand\algorithmicthen{:}
\algnewcommand{\IfThen}[2]{\State \algorithmicif\ #1\ \algorithmicthen\ #2}
\algnewcommand{\IfThenElse}[3]{\State \algorithmicif\ #1\ \algorithmicthen\ #2\ \algorithmicelse\ #3}
\algrenewcommand{\algorithmiccomment}[1]{\hfill \rightcomment{#1}}
\algnewcommand{\LineComment}[1]{\State \rightcomment{#1}}
\algnewcommand{\LinesComment}[1]{\State \rightcomment{\parbox[t]{\linewidth-\leftmargin-\widthof{\(\triangleright\) }}{#1}}\smallskip}
\algnewcommand\algorithmicinput{{\bfseries Input:}}
\algnewcommand\INPUT{\item[\algorithmicinput]}
\algnewcommand\algorithmicoutput{{\bfseries Output:}}
\algnewcommand\OUTPUT{\item[\algorithmicoutput]}
\makeatletter
\newcounter{algorithmicH}
\let\oldalgorithmic\algorithmic
\renewcommand{\algorithmic}{%
  \stepcounter{algorithmicH}
  \oldalgorithmic}
\renewcommand{\theHALG@line}{ALG@line.\thealgorithmicH.\arabic{ALG@line}}
\makeatother
\makeatletter
\newcommand{\algmargin}{\the\ALG@thistlm}
\makeatother
\algnewcommand{\Statepar}[1]{\State\parbox[t]{\dimexpr\linewidth-\algmargin}{\strut #1\strut}}

\newcommand{\pluseq}{\mathrel{+\!\!=}}

\usepackage{xpatch}
\xapptocmd\normalsize{
 \abovedisplayskip=11pt plus 3pt minus 9pt
 \abovedisplayshortskip=0pt plus 3pt
 \belowdisplayskip=11pt plus 3pt minus 9pt
 \belowdisplayshortskip=6.5pt plus 3.5pt minus 3pt
}{}{}

\usepackage{amsmath}
\usepackage{amssymb}
\usepackage{mathrsfs}
\usepackage{mathtools, cuted}
\usepackage{tabularx, booktabs}
\newcolumntype{C}{>{\centering\arraybackslash}X}
\newcolumntype{R}{>{\raggedleft\arraybackslash}X}
\newcolumntype{S}{>{\raggedleft\arraybackslash\hsize=.5\hsize}X}
\usepackage{latexsym}
\usepackage{url}
\usepackage{xspace}
\usepackage{bm,array}
\usepackage{amsfonts}
\usepackage{enumitem}
\usepackage{mathtools}
\usepackage{bm}
\usepackage{esvect}
\usepackage[noabbrev,capitalize]{cleveref}
\crefname{equation}{equation}{equations}
\crefname{section}{section}{sections}
\crefname{footnote}{footnote}{footnotes}   
\crefname{line}{line}{lines}   

\usepackage{bbm}
\renewcommand{\vec}[1]{{\boldsymbol{\mathbf{#1}}}}

\newcommand{\defn}[1]{\textbf{#1}}
\newcommand{\defeq}{\mathrel{\stackrel{\textnormal{\tiny def}}{=}}}
\newcommand{\E}[2][]{\mathbb{E}_{{#1}}[#2]}
\newcommand{\Real}{\mathbb{R}}

\newcommand{\Uniform}{\mathrm{Unif}}
\newcommand{\Exp}{\mathrm{Exp}}
\newcommand{\Categorical}{\mathrm{Categorical}}
\renewcommand{\th}{\textsuperscript{th}\xspace}
\newcommand{\bos}{\textsc{bos}\xspace}
\newcommand{\eos}{\textsc{eos}\xspace}

\newcommand{\comp}{\vec{x} \sqcup \vec{z}}

\newcommand{\obs}{\vec{x}}
\newcommand{\unobs}{\vec{z}}
\newcommand{\truth}{{\unobs^{*}}}
\newcommand{\estimate}{\hat{\unobs}}
\newcommand{\seqspace}{\mathcal{Z}}

\newcommand{\supk}{^{(k)}}
\newcommand{\ali}{\vec{a}}
\newcommand{\alihat}{\hat{\ali}}
\newcommand{\zunion}{\vec{z}_{\sqcup}}

\newcommand{\history}{\mathcal{H}}
\newcommand{\future}{\mathcal{F}}
\newcommand{\kt}[2]{#1{\scriptstyle @}#2}

\newcommand{\target}{p(\unobs \mid \obs)}
\newcommand{\proposal}{q(\unobs \mid \obs)}
\newcommand{\model}{p_{\text{model}}}
\newcommand{\pmiss}{p_{\text{miss}}}
\newcommand{\Obs}{\mathit{X}}
\newcommand{\Unobs}{\mathit{Z}}
\newcommand{\Comp}{\mathit{Y}}
\newcommand{\pmi}{\rho}
\newcommand{\inten}[2]{\lambda_{{#1}}(#2 \mid \history(#2))}
\newcommand{\intenboth}[2]{\lambda_{{#1}}(#2 \mid \history(#2), \future(#2))}
\newcommand{\intenbothq}[2]{\lambda^q_{{#1}}(#2 \mid \history(#2), \future(#2))}

\newcommand{\stat}{\underline}

\newcommand{\state}{\vec{h}}
\newcommand{\stateb}{\bar{\vec{h}}}
\newcommand{\cell}{\vec{c}}
\newcommand{\cellb}{\bar{\vec{c}}}
\newcommand{\inputgate}{\vec{i}}
\newcommand{\inputgateb}{\bar{\vec{i}}}
\newcommand{\forgetgate}{\vec{f}}
\newcommand{\forgetgateb}{\bar{\vec{f}}}
\newcommand{\zgate}{\vec{z}}
\newcommand{\zgateb}{\bar{\vec{z}}}
\newcommand{\outputgate}{\vec{o}}
\newcommand{\outputgateb}{\bar{\vec{o}}}
\newcommand{\statcell}{\stat{\vec{c}}}
\newcommand{\statcellb}{\bar{\stat{\vec{c}}}}
\newcommand{\statinputgate}{\stat{\vec{i}}}
\newcommand{\statinputgateb}{\bar{\stat{\vec{i}}}}
\newcommand{\statforgetgate}{\stat{\vec{f}}}
\newcommand{\statforgetgateb}{\bar{\stat{\vec{f}}}}
\newcommand{\decaygate}{\vec{\delta}}
\newcommand{\decaygateb}{\bar{\vec{\delta}}}

\newcommand{\loss}{L}

\usepackage[usenames,dvipsnames,svgnames,table]{xcolor}
\usepackage[disable]{todonotes} 
\makeatletter
\newcommand*\iftodonotes{\if@todonotes@disabled\expandafter\@secondoftwo\else\expandafter\@firstoftwo\fi}
\makeatother

\newcommand{\Fixme}[2][]{\noindent}
\newcommand{\Notewho}[3][]{\noindent}
\newcommand{\Jason}[2][]{\noindent}
\newcommand{\Hongyuan}[2][]{\noindent}
\newcommand{\Guanghui}[2][]{\noindent}
\newlength{\extramargin}
\usepackage{calc}
\iftodonotes{
  \usepackage[paperwidth=\paperwidth+\extramargin*2,marginparwidth=\marginparwidth+\extramargin,width=\textwidth,height=\textheight]{geometry}
  \addtolength{\marginparwidth}{-7mm}
}{}

\newcommand{\cutforspace}[1]{}
\renewcommand{\paragraph}[1]{\par\textbf{#1}\quad}

\usepackage{tikz}
\usetikzlibrary{shapes.geometric}
\newcommand{\orangecircle}{\begin{tikzpicture} \draw [fill=orange, thick] (0,0) circle (0.1); \end{tikzpicture}\xspace}
\newcommand{\greendiamond}{\begin{tikzpicture} \draw [fill=Green, thick, rotate=45] (0,0) rectangle (0.15, 0.15); \end{tikzpicture}\xspace}
\newcommand{\purplesquare}{\begin{tikzpicture} \draw [fill=MediumOrchid, thick] (0,0) rectangle (0.17, 0.17); \end{tikzpicture}\xspace}
\newcommand{\bluehexagon}{\begin{tikzpicture} \path node[regular polygon, regular polygon sides=6, fill=cyan, draw, thick, scale=0.7] (hexagon) {}; \end{tikzpicture}\xspace}
\newcommand{\greensolid}{\begin{tikzpicture} \draw [Green, ultra thick] (0,1) -- (0.5,1); \end{tikzpicture}\xspace}
\newcommand{\greendash}{\begin{tikzpicture} \draw [Green, dashed, ultra thick] (0,1) -- (0.5,1); \end{tikzpicture}\xspace}
\newcommand{\purplesolid}{\begin{tikzpicture} \draw [MediumOrchid, ultra thick] (0,1) -- (0.5,1); \end{tikzpicture}\xspace}
\newcommand{\purpledash}{\begin{tikzpicture} \draw [MediumOrchid, dashed, ultra thick] (0,1) -- (0.5,1); \end{tikzpicture}\xspace}
\newcommand{\greendiamondfaint}{\begin{tikzpicture} \draw [fill=Green, opacity=0.3, thick, rotate=45] (0,0) rectangle (0.15, 0.15); \end{tikzpicture}\xspace}
\newcommand{\bluedot}{\begin{tikzpicture} \draw [fill=blue] (0,0) circle (0.1); \end{tikzpicture}\xspace}
\newcommand{\reddot}{\begin{tikzpicture} \path node[regular polygon, regular polygon sides=3, fill=red, draw, red, thick, scale=0.5] (hexagon) {}; \end{tikzpicture}\xspace}
\usetikzlibrary{arrows,decorations.markings}
\usetikzlibrary{arrows}
\newcommand{\blueline}{\begin{tikzpicture} \draw[arrows={-angle 60}, white, thick, rotate=180, opacity=1.0] (0,0.00) -- (0.5,0.00); \draw[arrows={-angle 60}, blue, thick, rotate=180] (0,-0.06) -- (0.5,-0.06); \end{tikzpicture}\xspace}
\newcommand{\redline}{\begin{tikzpicture} \draw[dashed, arrows={-angle 60}, white, thick, rotate=180, opacity=1.0] (0,0.00) -- (0.5,0.00); \draw[dashed, arrows={-angle 60}, red, thick, rotate=180] (0,-0.06) -- (0.5,-0.06); \end{tikzpicture}\xspace}

\newcommand{\heldout}{test\space}

\usepackage{verbatim}

\icmltitlerunning{Imputing Missing Events in Continuous-Time Event Streams}

\begin{document}
\twocolumn[
\icmltitle{Imputing Missing Events in Continuous-Time Event Streams}

\icmlsetsymbol{equal}{*}

\begin{icmlauthorlist}
	\icmlauthor{Hongyuan Mei}{to}
	\icmlauthor{Guanghui Qin}{goo}
	\icmlauthor{Jason Eisner}{to}
\end{icmlauthorlist}

\icmlaffiliation{to}{Department of Computer Science, Johns Hopkins University, USA}
\icmlaffiliation{goo}{Department of Physics, Peking University, China}

\icmlcorrespondingauthor{Hongyuan Mei}{hmei@cs.jhu.edu}

\icmlkeywords{Missing events, neural Hawkes, particle smoothing, MBR decoding}

\vskip 0.3in
]

\printAffiliationsAndNotice{}

\begin{abstract}\label{sec:abstract}
Events in the world may be caused by other, {\em unobserved} events.
We consider sequences of events in continuous time.  
Given a probability model of \emph{complete} sequences, we propose particle smoothing---a form of sequential importance sampling---to impute the missing events in an \emph{incomplete} sequence.
We develop a trainable family of proposal distributions based on a type of  bidirectional continuous-time LSTM: Bidirectionality lets the proposals condition on future observations, not just on the past as in particle filtering.  Our method can sample an ensemble of possible complete sequences (particles), from which we 
form a single consensus prediction that has low Bayes risk under our chosen loss metric.
We experiment in multiple synthetic and real domains, using different missingness mechanisms, and modeling the complete sequences in each domain with a neural Hawkes process \citep{mei-17-neuralhawkes}.  On held-out incomplete sequences, our method is effective at inferring the ground-truth unobserved events, with particle smoothing consistently improving upon particle filtering.
\end{abstract}

\section{Introduction}\label{sec:intro}
\defn{Event streams} of discrete events in continuous time are often {\em partially} observed.  We would like to impute the missing events $\unobs$.  Suppose we know the prior distribution $\model$ of complete event streams, as well as the ``missingness mechanism'' $\pmiss(\unobs \mid \text{complete stream})$, which stochastically determines which of the events will not be observed.  One can then use use Bayes' Theorem, as spelled out in \cref{eqn:target} below, to define the posterior distribution $\target$ given just the observed events $\obs$.\footnote{Bayes' Theorem can be applied even if $\pmiss$ is a missing-not-at-random (MNAR) mechanism, as is common in this setting.  MNAR is only tricky if we know \emph{neither} $\model$ \emph{nor} $\pmiss$.}

\begin{figure*}[!ht]
	\begin{center}
		\caption[]{Stochastically imputing a taxi's pick-up events (\greendiamond) given its observed drop-off events (\purplesquare).  At this stage, we are trying to determine the next event after the \purplesquare at time $t_1$---either an unobserved event at $t_{1,1} \in (t_1,t_2)$ or the next observed event at $t_2$.}
		\vspace{12pt}

		\begin{subfigure}[b]{0.99\linewidth}
			\begin{subfigure}[t]{0.32\linewidth}
				\renewcommand\thesubfigure{\alph{subfigure}1}
				\includegraphics[width=0.99\linewidth]{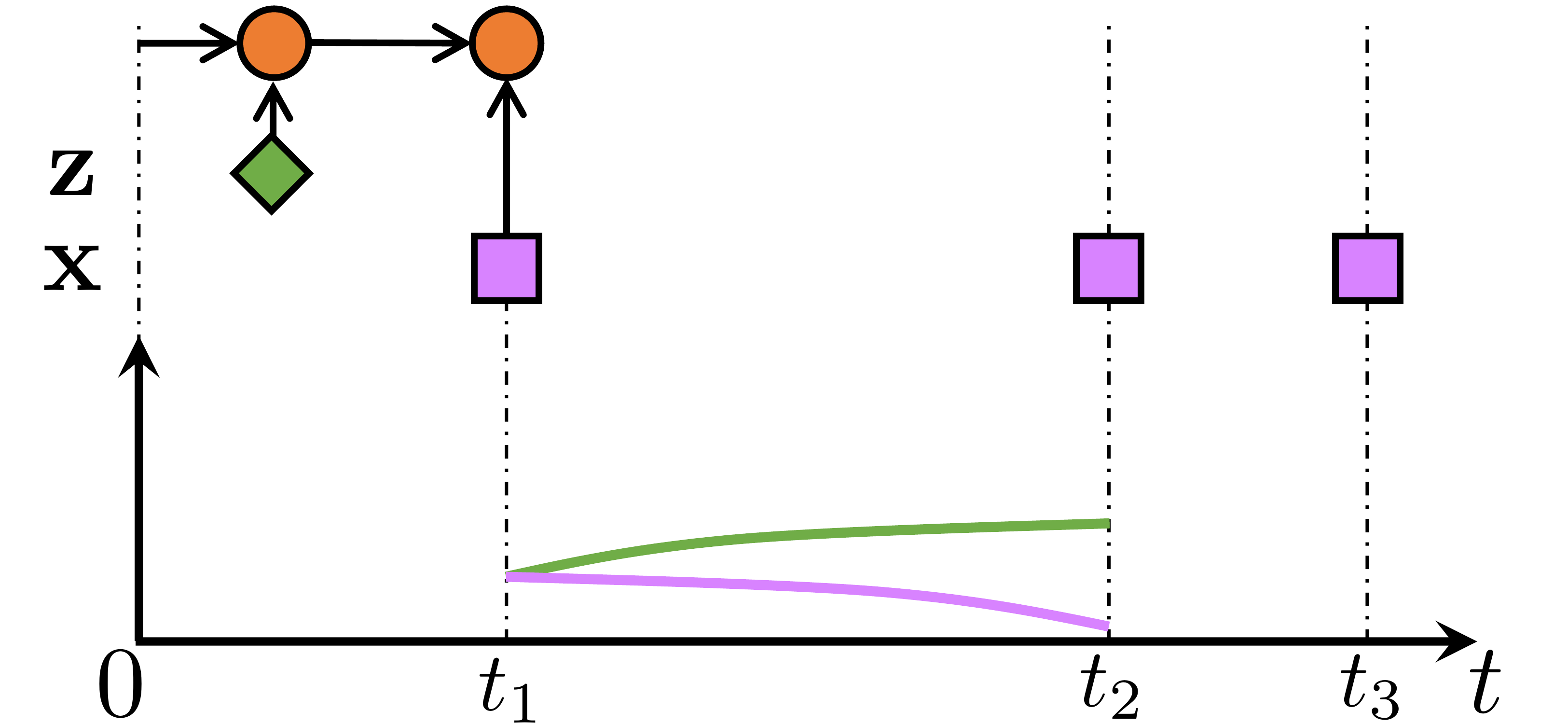}
				\caption{Both intensities are low (i.e., passengers are scarce at this time of day), 
					so no event happens to be proposed in $(t_1,t_2)$.}\label{fig:pf1}
			\end{subfigure}
			~
			\begin{subfigure}[t]{0.32\linewidth}
				\addtocounter{subfigure}{-1}
				\renewcommand\thesubfigure{\alph{subfigure}2}
				\includegraphics[width=0.99\linewidth]{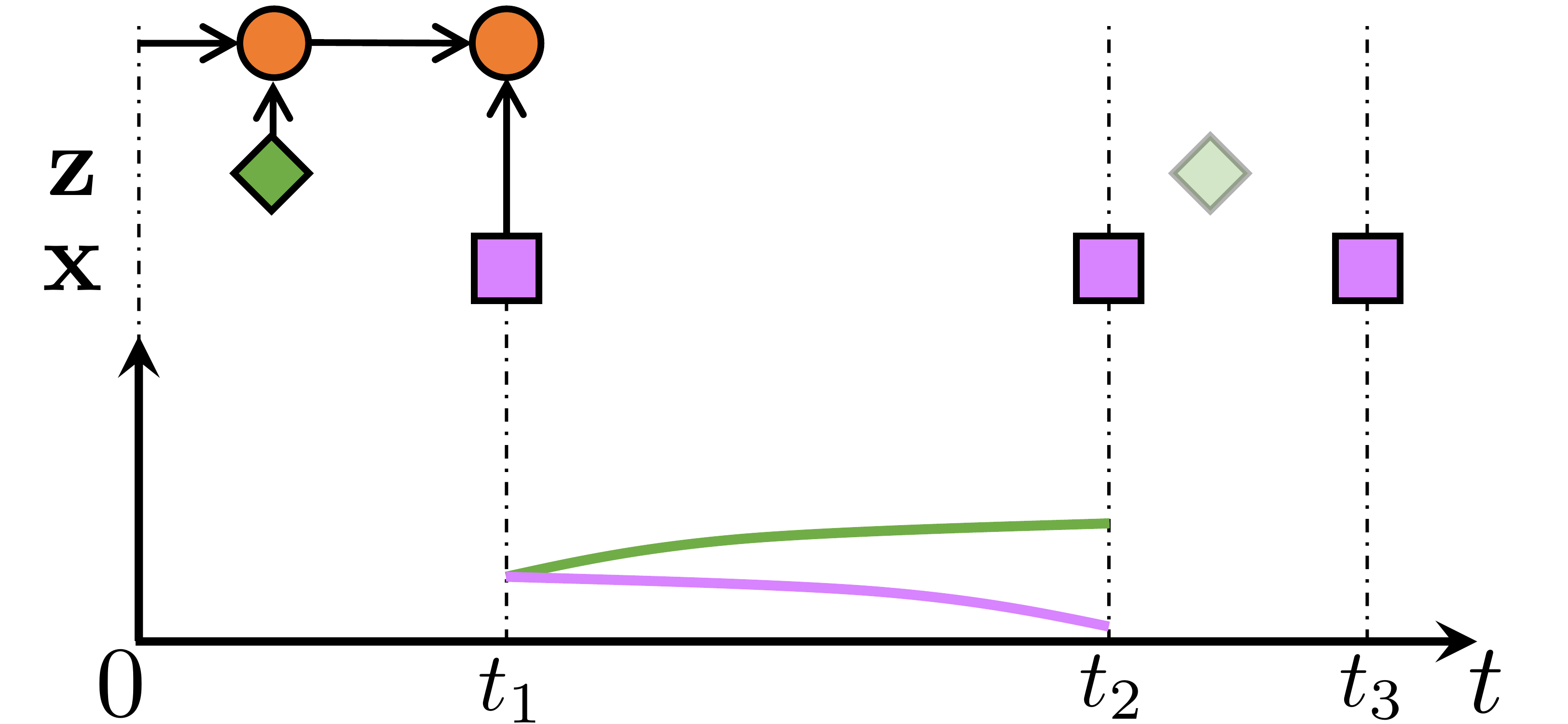}
				\caption{Specifically,
					the next proposed event (\greendiamondfaint) would be \emph{somewhere} after $t_2$,
					without bothering to determine its time precisely.
				}\label{fig:pf2}
			\end{subfigure}
			~
			\begin{subfigure}[t]{0.32\linewidth}
				\addtocounter{subfigure}{-1}
				\renewcommand\thesubfigure{\alph{subfigure}3}
				\includegraphics[width=0.99\linewidth]{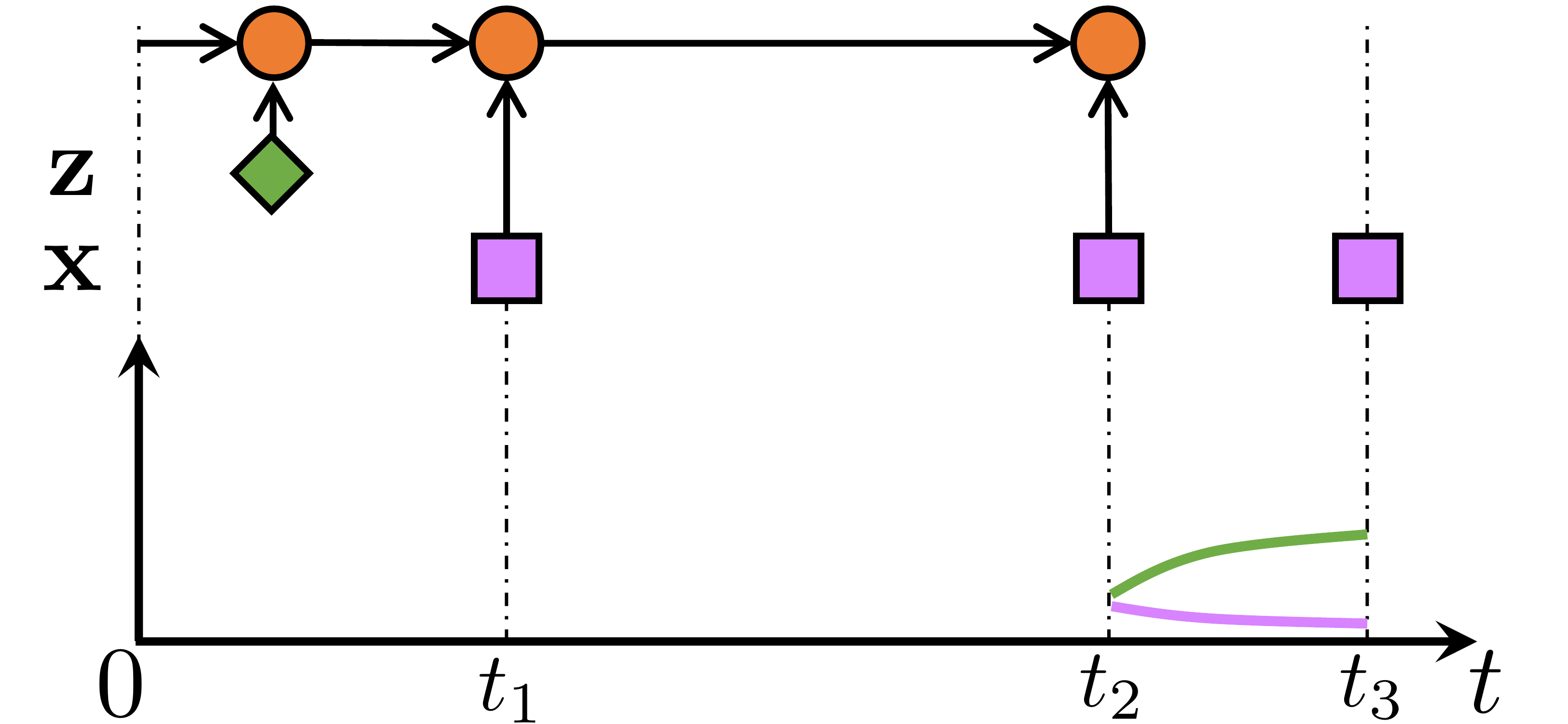}
				\caption{Thus, the next event is $\kt{\purplesquare}{t_2}$;
                                  we feed it into the LSTM, preempting \greendiamondfaint, which is discarded (\cref{line:preempt} of \cref{alg:sis}).
				}\label{fig:pf3}
			\end{subfigure}
			\addtocounter{subfigure}{-1}
			\caption{
				\textbf{Particle filtering (\cref{sec:nhpf})}. We show part of the process of drawing one particle.  Above left, the neural Hawkes process's LSTM has already read the proposed and observed events at times $\leq t_1$.  Its resulting state \orangecircle determines the model intensities \greensolid and \purplesolid of the two event types \greendiamond and \purplesquare, from which the sampler (\Cref{alg:sis} in \cref{sec:smc_details})
                                          determines that there is no unobserved event in $(t_1, t_2)$.  Above right,
				we continue to extend the particle by feeding $\kt{\purplesquare}{t_2}$ into the LSTM and 
				proposing subsequent events based on the new intensities after $t_2$.
				But because \purplesolid was low at $t_2$, the $\kt{\purplesquare}{t_2}$ was unexpected, and that results in downweighting the particle (\cref{line:downweight} of \cref{alg:sis}).  Downweighting recognizes belatedly that proposing no event in $(t_1,t_2)$ has committed us to a particle that will be improbable under the posterior, because its complete sequence includes consecutive drop-offs ($\kt{\purplesquare}{t_1}, \kt{\purplesquare}{t_2}$) far apart in time. 
			}\label{fig:pf}
		\end{subfigure}

		\vspace{8pt}

		\begin{subfigure}[b]{0.99\linewidth}
			\begin{subfigure}[t]{0.32\linewidth}
				\renewcommand\thesubfigure{\alph{subfigure}1}
				\includegraphics[width=0.99\linewidth]{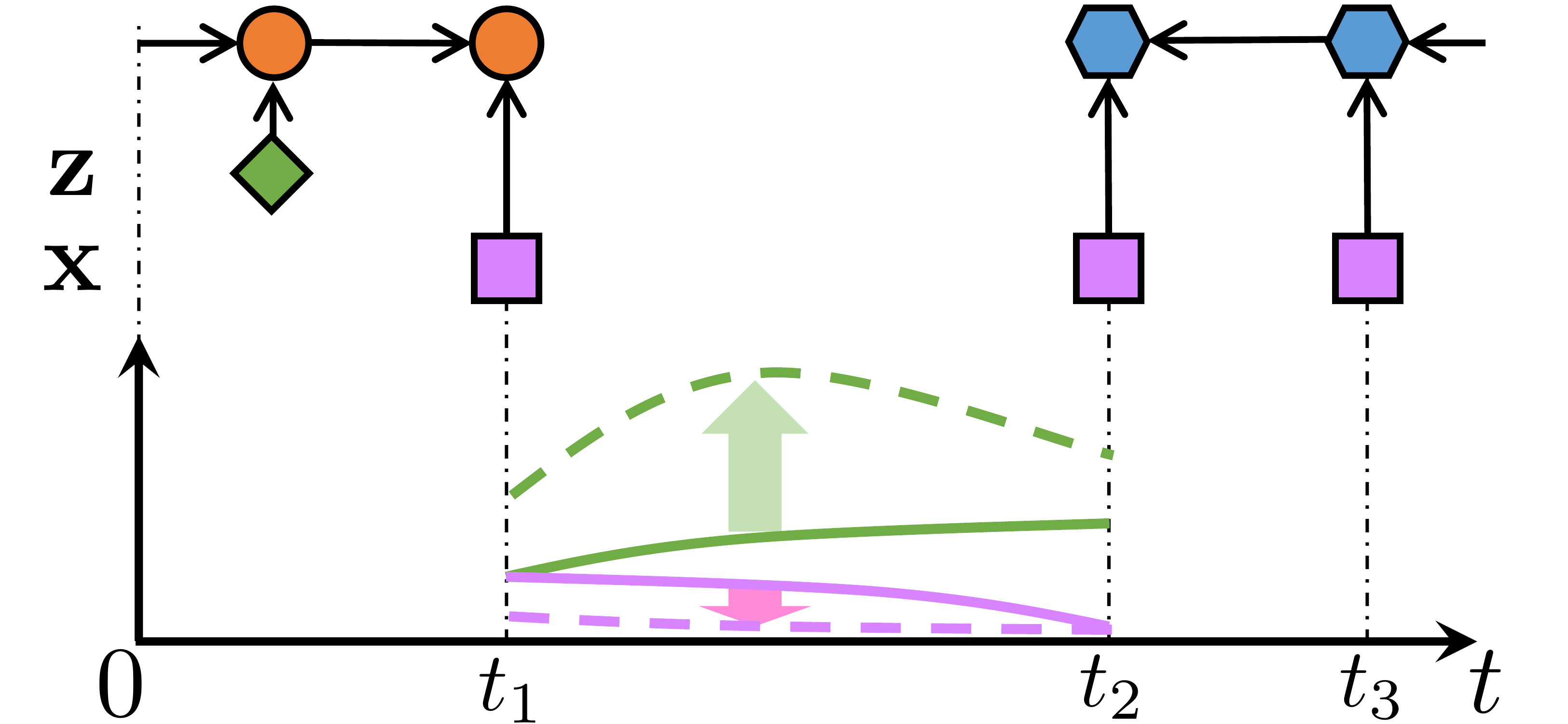}
				\caption{Since a drop-off at $t_2$ strongly suggests a pick-up before $t_2$, considering the future increases the intensity of pick-up on $(t_1,t_2)$ from \greensolid to \greendash (while decreasing that of drop-off from \purplesolid to \purpledash).}\label{fig:ps1}
			\end{subfigure}
			~
			\begin{subfigure}[t]{0.32\linewidth}
				\addtocounter{subfigure}{-1}
				\renewcommand\thesubfigure{\alph{subfigure}2}
				\includegraphics[width=0.99\linewidth]{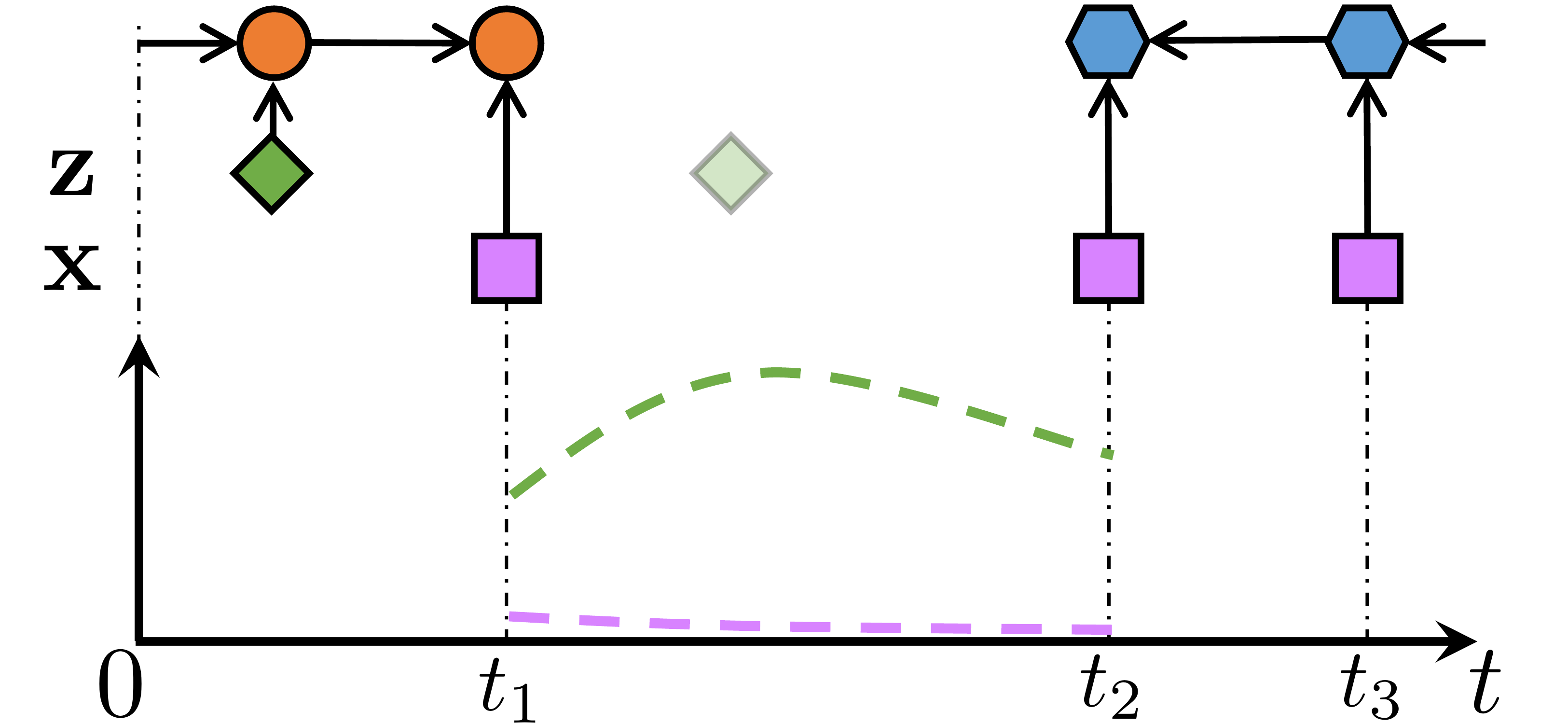}
				\caption{  Consequently, the next proposed event is more likely to be a pick-up in $(t_1,t_2)$ than it was in \cref{fig:pf}. If we stochastically generate such an event $\kt{\greendiamond}{t_{1,1}}$, it is fed into the original \orangecircle LSTM. }\label{fig:ps2}
			\end{subfigure}
			~
			\begin{subfigure}[t]{0.32\linewidth}
				\addtocounter{subfigure}{-1}
				\renewcommand\thesubfigure{\alph{subfigure}3}
				\includegraphics[width=0.99\linewidth]{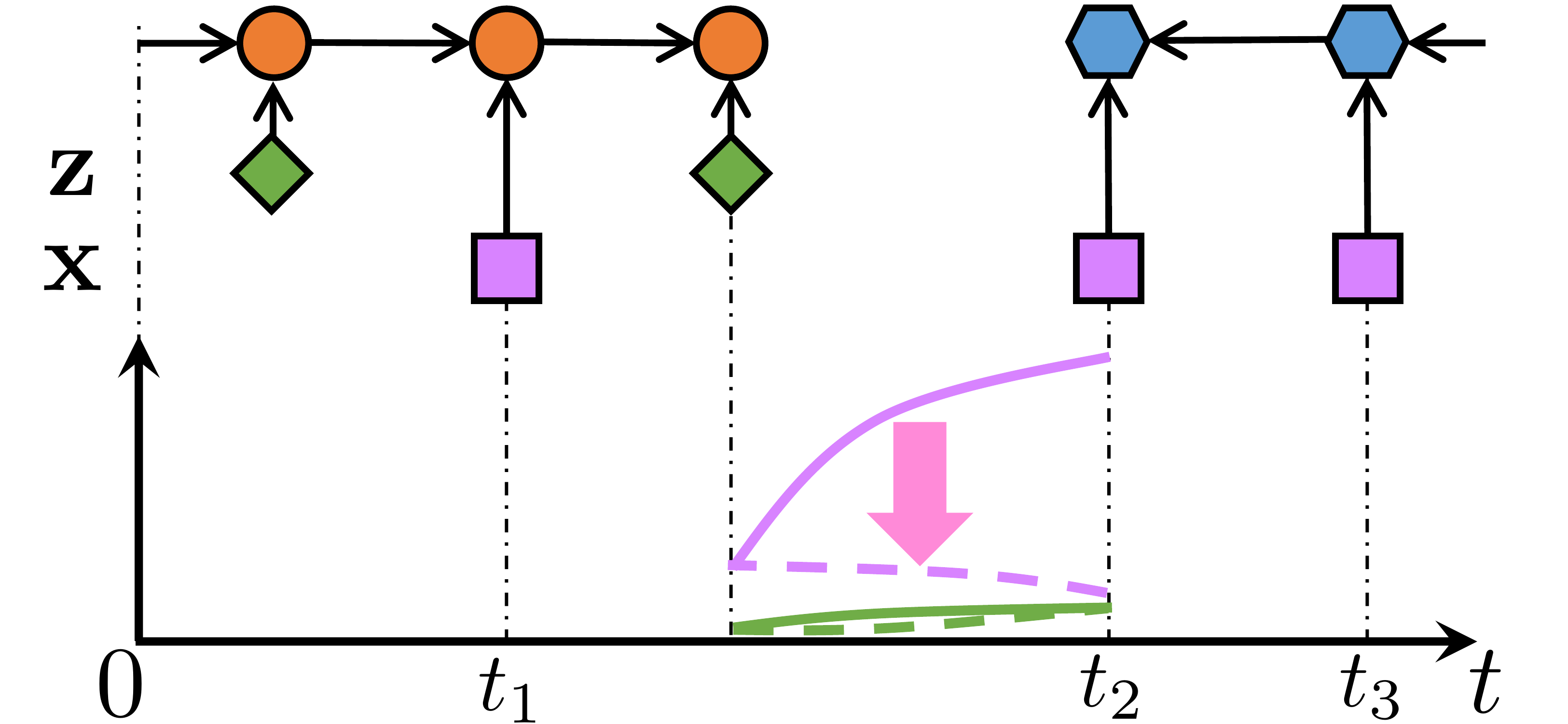}
				\caption{The updated state \orangecircle determines the new \emph{model} intensities \greensolid and \purplesolid,
					and also combines with \bluehexagon to determine the new \emph{proposal} intensities \greendash and \purpledash, which are used to sample the next event.
				}\label{fig:ps3}
			\end{subfigure}
			\addtocounter{subfigure}{-1}
			\caption{
				\textbf{Particle smoothing (\cref{sec:nhps})} samples from a better-informed proposal distribution: a second LSTM (\cref{sec:r2l}) reads the future observations from right to left, and its state \bluehexagon is used \emph{together} with \orangecircle to determine the proposal intensities \greendash and \purpledash.
			}\label{fig:ps}
		\end{subfigure}
		\vspace{-12pt}
		\label{fig:method}
	\end{center}
\end{figure*}

\paragraph{Why is this important?} The ability to impute $\unobs$ is useful in many applied domains, for example:
\begin{itemize}[noitemsep,nolistsep]
\item {\em Medical records.} Some patients record detailed symptoms, self-administered medications, diet, and sleep.  Imputing these events for other patients would produce an augmented medical record that could improve diagnosis, prognosis, treatment, and counseling.

Similar remarks apply to users of life-tracking apps (e.g., MyFitnessPal) who forget to log some of their daily activities (e.g., meals, sleep and exercise). 

\item {\em Competitive games.} In poker or StarCraft, a player lacks full information about what her opponents have acquired (cards) or done (build mines and train soldiers).  Accurately imputing hidden actions from ``what I did'' and ``what I observed others doing'' can help the player make good decisions.  Similar remarks apply to practical scenarios (e.g., military) where multiple actors compete and/or cooperate.
\item {\em User interface interactions.} Cognitive events are usually unobserved. For example, users of an online news provider (e.g., Bloomberg Terminal) may have read and remembered a displayed headline whether or not they clicked on it. Such events are expensive to observe (e.g., via gaze tracking or asking the user).  Imputing them given the observed events (e.g., other clicks) would facilitate personalization.
\item Other partially observed event streams arise in {\em online shopping}, {\em social media}, etc.
\end{itemize}

\paragraph{Why is it challenging?}
It is computationally difficult to reason about the posterior distribution $p(\unobs \mid \obs)$.  Even for a simple $\model$ like a Hawkes process \citep{hawkes-71}, Markov chain Monte Carlo (MCMC) methods are needed, and these methods obtain an efficient transition kernel only by exploiting special properties of the process \citep{shelton-18-missing}. Unfortunately, such properties no longer hold for the more flexible neural models that we will use in this paper \citep{du-16-recurrent,mei-17-neuralhawkes}.

\paragraph{What is our contribution?}
We are, to the best of our knowledge, the first to develop general sequential Monte Carlo (SMC) methods to approximate the posterior distribution over incompletely observed draws
  from a neural point process.
We begin by sketching the approach.

\citet{mei-17-neuralhawkes} give an algorithm to sample a complete sequence from a neural Hawkes process.  Each event in turn is sampled given the complete history of previous events. However, this algorithm only samples from the prior over complete sequences.
We first adapt it into a \defn{particle filtering} algorithm that samples from the posterior given all the observed events.  The basic idea (\cref{fig:pf}) is to draw the events in sequence as before, but now we \emph{force} any observed events to be ``drawn'' at the appropriate times.  That is, we add the observed events to the sequence as they happen (and they duly affect the distribution of subsequent events). There is an associated cost: if we are forced to draw an observed event that is \emph{improbable} given its past history, we must downweight the resulting complete sequence accordingly, because evidently the particular past history that we sampled was inconsistent with the observed event, and hence cannot be part of a high-likelihood complete sequence.  Using this method, we sample many sequences (or \defn{particles}) of different {\em relative} weights. 
This method applies to any temporal point process.\footnote{As long as the number of events is finite with probability 1, and it is tractable to compute the log-likelihood of a complete sequence and to estimate the log-likelihoods of its prefixes.}
\citet{linderman-17-bayesian} apply it to the classical Hawkes process.

Alas, this approach is computationally inefficient.  Sampling a complete sequence that is actually probable under the posterior requires great luck, as the proposal distribution must have the good fortune to draw only events that happen to be consistent with future observations.  Such lucky particles would appropriately get a high weight relative to other particles.  The problem is that we will rarely get such particles at all (unless we sample very many).

To get a more accurate picture of the posterior, this paper draws each event from a smarter distribution that is conditioned on the future observations (rather than drawing the event in ignorance of the future and then downweighting the particle if the future does not turn out as hoped).

This idea is called \defn{particle smoothing} \citep{doucet-09-particle}.
How does it work in our setting? The neural Hawkes process defines the distribution of the next event using the state of a \defn{continuous-time LSTM} that has read the past history from left to right.  When sampling a proposed event, we now use a modified distribution (\cref{fig:ps}) that \emph{also} considers the state of a second continuous-time LSTM that has read the future observations from right to left. As this modified distribution is still imperfect---merely a proposal distribution---we still have to reweight our particles to match the actual posterior under the model.  But this reweighting is not as drastic as for particle filtering, because the new proposal distribution was constructed and trained to resemble the actual posterior.  
Our proposal distribution could also be used with other
point process models by replacing the left-to-right LSTM state with other informative statistics of the past history.

\paragraph{What other contributions?}
We introduce an appropriate evaluation loss metric for event stream reconstruction, and then design a \defn{consensus decoder} that outputs a single low-risk prediction of the missing events by combining the sampled particles (instead of picking one of them).

\vspace{-5pt}
\section[Preliminaries]{Preliminaries\footnote{Our conventions regarding capitalization, boldface, etc.  are inherited from the notation of \protect{\citet[section 2]{mei-17-neuralhawkes}}.
}}\label{sec:notation}
\vspace{-3pt}
\subsection{Partially Observed Event Streams}\label{sec:partial}
We consider a missing-data setting \citep{little-rubin-1987}.  We are given a fixed time interval $[0,T)$ over which events can be observed.  An event of type $k \in  \{1, 2, \ldots, K\}$ at time $t \in [0,T)$ is denoted by an ordered pair written mnemonically as $\kt{k}{t}$.  
Each possible outcome in our probability distributions is a complete event sequence in which each event is designated as either ``observed'' or ``missing.'' 

We observe only the \defn{observed events}, denoted by $\obs = \{\kt{k_1}{t_1}, \kt{k_2}{t_2}, \ldots, \kt{k_I}{t_I}\}$, where $0 = t_0 < t_1 < t_2 < \ldots < t_I < t_{I+1} = T$.
We are given the observation interval $[0,T)$ in the form of two \defn{boundary events} $\kt{k_0}{t_0}$ and $\kt{k_{I+1}}{t_{I+1}}$ at its endpoints, where $k_0\!=\!0$ and $k_{I+1}\!=\!K+1$.  

Let $\kt{k_{i,0}}{t_{i,0}}$ be an alternative notation for the observed event $\kt{k_i}{t_i}$.  Following this observed event (for any $0 \leq i \leq I$), there are $J_i \geq 0$ \defn{unobserved events} $\unobs=\{\kt{k_{i,1}}{t_{i,1}}, \kt{k_{i,2}}{t_{i,2}}, \ldots, \kt{k_{i,J_i}}{t_{i,J_i}}\}$, where $t_{i,0} < t_{i,1} < \ldots < t_{i,J_i} < t_{i+1}$.  We must guess this unobserved sequence including its length $J_i$.  Let $\sqcup$ denote disjoint union.  Our hypothesized \defn{complete event sequence} $\comp$ is thus $\{\kt{k_{i,j}}{t_{i,j}}: 0 \leq i \leq I+1, 0 \leq j \leq J_i\}$, where $t_{i,j}$ increases strictly with the pair $\angles{i,j}$ in  lexicographic order.\footnote{\label{fn:setseq}In general we should allow $t_{i,j}$ to increase \emph{non}-strictly with $\angles{i,j}$. But equality happens to have probability 0 under the neural Hawkes model. So it is convenient to exclude it here, simplifying notation by allowing $\obs, \unobs,\history(t)$ to be sets, not sequences.}

In this paper, we will attempt to guess all of $\unobs$ jointly by sampling it from the posterior distribution
\begin{align*}
&p(\Unobs=\unobs \mid \Obs=\obs)  \\
&\qquad \propto \model(\Comp=\comp) \cdot \pmiss(\Unobs=\unobs \mid \Comp=\comp) 
\end{align*}
of a process that \emph{first} generates the complete sequence $\comp$ from a complete data model $\model$ (given $[0,T)$), and \emph{then} determines which events to censor with the possibly stochastic \defn{missingness mechanism} $\pmiss$. 
The random variables $\Obs$, $\Unobs$, and $\Comp$ refer respectively to the sets of observed events, missing events, and all events over $[0,T)$.  Thus $\Comp = \Obs\sqcup\Unobs$.  
Under the \cutforspace{probability }distributions we will consider,
$|\Comp|$ is almost surely finite.  
Notice that $\unobs$ denotes the set of missing events in $\Comp$ and $Z=\unobs$ denotes
the fact that they are missing.
That said, we will abbreviate our notation above in the standard way:%
\begin{align}\label{eqn:target}
p(\unobs \mid \obs) \;\propto\; \model(\comp) \cdot \pmiss(\unobs \mid \comp)
\end{align}

Note that $\comp$ is simply an undifferentiated sequence of $\kt{k}{t}$ pairs; the subscripts $\angles{i,j}$ are in effect assigned by $\pmiss$, which partitions $\comp$ into $\obs$ and $\unobs$.
To explain a sequence of 50 observed events, one hypothesis is that
$\model$ generated 73 events and then $\pmiss$ selected 23 of them to
be missing (as $\unobs$), leaving the 50 observed events (as $\obs$).

In many missing data settings, the second factor of \cref{eqn:target} can be ignored because (for the given $\obs$) it is known to be a constant function of $\unobs$.  Then the missing data are said to be \defn{missing at random (MAR)}.  For event streams, however, the second factor is generally not constant in $\unobs$ but varies with the \emph{number} of missing events $|\unobs|$.  Thus, our unobserved events are normally \defn{missing not at random (MNAR)}.  See discussion in \cref{sec:miss-mech} and \cref{sec:miss_details}.

\subsection{Choice of $\model$}\label{sec:nhp}
We need a multivariate point process model $\model(\comp)$.  We choose the \defn{neural Hawkes process} \citep{mei-17-neuralhawkes}, which has proven flexible and effective at modeling many real-world event streams.  

Whether an event happens at time $t \in [0,T)$ depends on the \defn{history} $\history(t) \defeq \{\kt{k'}{t'} \in \comp : t' < t\}$---the set of all {\em observed} and {\em unobserved} events before $t$.
Given this  history, the neural Hawkes process defines an \defn{intensity}
$\lambda_k(t \mid \history(t)) \in \Real_{\geq 0}$, which may be
thought of as the \emph{instantaneous rate} at time $t$ of events of type $k$:
\begin{align}\label{eqn:hawkes_mod_a}
	\lambda_k(t \mid \history(t)) &= f_k(\vec{v}_k^\top \state(t))
\end{align}
Here $f_k$ is a softplus function with $k$-specific scaling parameter.  The vector $\vec{h}(t) \in (-1,1)^D$ summarizes $(\history(t),t)$. It is the hidden state at time $t$ of a \defn{continuous-time LSTM} that previously read the events in $\history(t)$ \emph{as they happened}.  The state of such an LSTM evolves endogenously as it waits between events, so the state $\vec{h}(t)$ reflects not only the sequence of past events but also their \emph{timing}, including the gap between the last event in $\history(t)$ and $t$.

As \citet{mei-17-neuralhawkes} explain, the probability of an event of type $k$ in the interval $[t,t+dt)$, divided by $dt$, approaches \eqref{eqn:hawkes_mod_a} as $dt \to 0^+$.  Thus, $\lambda_k$ is similar to the intensity function of an inhomogeneous Poisson process.  Yet it is not a fixed parameter: the $\lambda_k$ function for times $\geq t$ is affected by the previously sampled events $\history(t)$.
See \cref{sec:nhp_details}.

\section{Particle Methods}\label{sec:method}
It is often intractable to sample {\em exactly} from $\target$, because 
$\obs$ and $\unobs$ can be interleaved with each other.  As an alternative, we can use normalized importance sampling, drawing many $\unobs$ values from a \defn{proposal distribution} $\proposal$ and weighting them in proportion to $\frac{\target}{\proposal}$.
\cref{fig:method} shows the key ideas in terms of an example.
Full details are spelled out in \cref{alg:sis} in \cref{sec:smc_details}.  

\Cref{alg:sis} is a \defn{Sequential Monte Carlo (SMC)}
approach \citep{moral-97-nonlinear,liu-98-sequential,doucet-00-sequential,doucet-09-particle}. 
It returns an \defn{ensemble of weighted particles} $\seqspace_M =
\{(\unobs_m, w_m)\}_{m=1}^{M}$.  Each particle $\unobs_m$ is sampled
from the \defn{proposal distribution} $\proposal$, which is defined to
support sampling via a \emph{sequential} procedure that draws one unobserved
event at a time.
The corresponding $w_m$ are \defn{importance weights}, which
are defined as follows\cutforspace{\footnote{\label{fn:droppmiss}Omitting $\pmiss$ from \cref{eqn:weight_b}---because the missingness mechanism is unknown---would require a MAR assumption that this factor is constant.  That is unlikely, as noted above.}} (and built up factor-by-factor in \cref{alg:sis}):
\begin{align}
\label{eqn:weight_b}
w_m &\propto \frac{\model(\comp_m)\;\pmiss(\unobs_m \mid \comp_m)}{q(\unobs_m \mid \obs)} \geq 0
\end{align}

where the normalizing constant is chosen to make $\sum_{m=1}^M w_m = 1$.
\Cref{eqn:target,eqn:weight_b} imply that we would have $w_m = 1/M$ if we could set $\proposal$ equal to $\target$, so that the particles were IID samples from the desired posterior distribution.
In practice, $q$ will not equal $p$, but will be easier than $p$ to sample from.  To correct for the mismatch, the importance weights $w_m$ are higher for particles that $q$ proposes less often than $p$ would have proposed them.  

The distribution implicitly formed by the ensemble, $\hat{p}(\unobs)$,
approaches $\target$ as $M \rightarrow \infty$ \citep{doucet-09-particle}.
Thus, for large $M$, the ensemble may be used to estimate the expectation of {\em any}
function $f(\unobs)$, via
\begin{align}\label{eqn:expectation}
  \E[\target]{f(\unobs)}
  &\approx \E[\hat{p}]{f(\unobs)}
   = {\textstyle \sum_{m=1}^M w_m f(\unobs_m) }
\end{align}

$f(\unobs)$ may be a function that summarizes properties of the complete stream $\comp$ on $[0,T)$, or predicts {\em future} events on $[T,\infty)$ using the sufficient statistic $\history(T) = \comp$.

In the subsections below, we will describe two specific proposal distributions $q$ that are appropriate for the neural Hawkes process,
as we sketched in \cref{sec:intro}.  These distributions define intensity functions $\lambda^q$ over time intervals.  

The trickiest part of \cref{alg:sis} (at \cref{line:thinning}) is to sample the next unobserved event from the proposal distribution $q$.  Here we use the \defn{thinning algorithm} \citep{lewis-79-sim,liniger-09-hawkes,mei-17-neuralhawkes}.  Briefly, this is a rejection sampling algorithm whose own proposal distribution uses a {\em constant} intensity $\lambda^*$, making it a homogeneous Poisson process (which is easy to sample from).  A event proposed by the Poisson process at time $t$ is accepted with probability $\lambda^q(t)/\lambda^* \leq 1$.  If it is rejected, we move on to the next event proposed by the Poisson process, continuing until we either accept such an unobserved event or are preempted by the arrival of the next observed event.

After each step, one may optionally {\em resample} a new set of particles from $\{\unobs_m\}_{m=1}^{M}$ (the $\textsc{Resample}$ procedure in \cref{alg:sis}).  This trick tends to discard low-weight particles and clone high-weight particles, so that the algorithm can explore multiple continuations of the high-weight particles.

\vspace{-2pt}
\subsection{Particle Filtering}\label{sec:nhpf}

We already have a neural Hawkes process $\model$ that was trained on complete data. This model uses a neural net to define an intensity function $\lambda^p_k(t \mid \history(t))$ for \emph{any} history $\history(t)$ of events before $t$ and each event type $k$.  

The simplest proposal distribution uses this intensity function to draw the unobserved events.  More precisely, for each $i = 0,1,\ldots,I$, for each $j = 0, 1, 2, \ldots$,  let the next event $\kt{k_{i,j+1}}{t_{i,j+1}}$ be the first event generated by any of the $K$ intensity functions $\lambda_k(t \mid \history(t))$ over the interval $t \in (t_{i,j},t_{i+1})$, where $\history(t)$ consists of all observed and unobserved events up through $\kt{k_{i,j}}{t_{i,j}}$.  If no event is generated on this interval, then the next event is $\kt{k_{i+1}}{t_{i+1}}$.  This is implemented by \cref{alg:sis} with $\mathit{smooth}=\textbf{false}$.

\vspace{-2pt}
\subsection{Particle Smoothing}\label{sec:nhps}
As motivated in \cref{sec:intro}, we would rather draw each unobserved event according to $\intenboth{k}{t}$ where the \defn{future} $\future(t) \defeq \{ \kt{k_i}{t_i} : t < t_i \leq T \}$ consists of all {\em observed} events that happen after $t$. 
Note the asymmetry with $\history(t)$, which includes observed but also unobserved events.

We use a \defn{right-to-left continuous-time LSTM} to summarize the future $\future(t)$ for any time $t$ into another hidden state vector $\stateb(t) \in (-1,1)^{D'}$.  Then we parameterize the proposal intensity using an extended variant of \cref{eqn:hawkes_mod_a}:
\begin{align}\label{eqn:both}
\rule{0pt}{12pt}\intenbothq{k}{t}
 &= f_k(\vec{v}_k^\top (\state(t) + \vec{B} \stateb(t) ) ) 
\end{align}

This extra machinery is used by \cref{alg:sis} when $\mathit{smooth} = \textbf{true}$.
Intuitively, the left-to-right $\state(t)$, as explained in \citet{mei-17-neuralhawkes}, reads the history $\history(t)$ and computes sufficient statistics for predicting events at times $\geq t$ given $\history(t)$.  But we wish to predict these events given $\history(t)$ \emph{and} $\future(t)$.  \Cref{eqn:both} approximates this Bayesian update using the right-to-left $\stateb(t)$, which is trained to carry 

back relevant information about future observations $\future(t)$.

This is a kind of neuralized forward-backward algorithm.  \citet{lin-eisner-2018-naacl} treat the discrete-time analogue, explaining why a neural forward $\model$ no longer admits tractable exact proposals as does a hidden Markov model
\citep{rabiner-89-tutorial} or linear dynamical system \citep{rauch-65-maximum}.  Like them, we fall back on training an approximate proposal distribution. Regardless of $\model$, particle smoothing is to particle filtering
as Kalman smoothing is to Kalman filtering 
\citep{kalman-60-filter,kalman-61-new}.

Our right-to-left LSTM has the same architecture as the left-to-right LSTM used
in our $\model$ (\cref{sec:nhp}), 
but a separate parameter vector.
For any time $t \in (0, T)$, it arrives at $\stateb(t)$ by reading {\em only} the {\em observed} events $\{\kt{k_i}{t_i}: t < t_i \leq T \}$, i.e., $\future(t)$, in {\em reverse} chronological order.  Formulas are given in \cref{sec:r2l}.  This architecture seemed promising for reading an {\em incomplete} sequence of events from right to left, as \citet[section 6.3]{mei-17-neuralhawkes} had already found that this architecture is predictive when used to read incomplete sequences from left to right.  

\subsubsection{Training the Proposal Distribution}\label{sec:train}
The particle smoothing proposer $q$ can be trained to approximate $\target$ by minimizing a \defn{Kullback-Leibler (KL) divergence}. Its left-to-right LSTM is fixed at $\model$, so its
trainable parameters $\vec{\phi}$ are just the parameters of the right-to-left LSTM together with the matrix $\vec{B}$ from \cref{eqn:both}.
Though $\target$ is unknown,
the gradient of \defn{inclusive KL divergence} between $\proposal$ and $\target$ is\looseness=-1
\begin{align}\label{eqn:in_kl}
  \nabla_{\phi} \text{KL}(p \mid\mid q )
  = \E[\unobs \sim \target]{ - \nabla_{\phi} \log \proposal }
\end{align}
and the gradient of \defn{exclusive KL divergence} is:
\begin{subequations}\label{eqn:ex_kl}
\begin{align}
  &\nabla_{\phi} \text{KL}(q \mid\mid p )
  = \E[\unobs \sim q]{\nabla_{\phi} \left( \textstyle{\frac{1}{2}} \left(\log \proposal - b \right)^2 \right) } \\
  &b
  = \log \model(\comp) + \log \pmiss(\unobs \mid \comp)
\end{align}
\end{subequations}
where $\log \model(\comp)$ is given in \cref{sec:nhp_details}, $\log \proposal$ is given in \cref{sec:proposal}, and $\pmiss(\unobs \mid \comp)$ is assumed to be known to us for any given pair of $\obs$ and $\unobs$.

Minimizing inclusive KL divergence aims at high recall---$\proposal$ is adjusted to assign high probabilities to all of the good hypotheses (according to $\target$).  Conversely, minimizing exclusive KL divergence aims at high precision---$\proposal$ is adjusted to assign low probabilities to poor reconstructions, so that they will not be proposed.
We seek to minimize the linearly combined divergence
\begin{equation}\label{eqn:combined_kl}
  \text{Div}  = \beta\,\text{KL}(p \| q ) + ( 1 - \beta ) \text{KL}(q \| p )\ \text{with}\ \beta \in [0, 1]
\end{equation}
and training is early-stopped when the divergence stops decreasing on the held-out development set. 

But how do we measure these divergences between $\proposal$ and $\target$?  Of course, we actually want the {\em expected} divergence when the observed sequence $\obs \sim {}$ the true distribution. 
Thus, we sample $\obs$ by starting with a \emph{fully observed} sequence from our training examples and then sampling a partition $\obs,\unobs$ from the known missingness mechanism $\pmiss$.\footnote{To get more data for training $q$, we could sample more partitions of the fully observed sequence. In this paper, we only sample one partition.  Note that the fully observed sequence is a real observation from the true complete data distribution (not the model).}
The inclusive expectation in \eqref{eqn:in_kl} uses this $\obs$ and $\unobs$.  
For the exclusive expectation in \eqref{eqn:ex_kl}, we keep this $\obs$ but sample a new $\unobs$ from our proposal distribution $q(\cdot \mid \obs)$.  

Notice that minimizing exclusive divergence here is essentially the REINFORCE algorithm \cite{williams92reinforce}, which is known to have large variance. 
In practice, when tuning our hyperparameters (\cref{sec:training}), $\beta=1$ in \eqref{eqn:combined_kl} gave the best results.  That is---perhaps unsurprisingly---our experiments effectively avoided REINFORCE altogether and placed {\em all} the weight on the inclusive KL, which has no variance issue. More training details including a bias and variance discussion can be found in \cref{sec:training}.

\Cref{sec:mcem} discusses situations where training on incomplete data by EM is possible.

\section{A Loss Function and Decoding Method}\label{sec:decode}
It is often useful to find a \emph{single} hypothesis $\estimate$ that minimizes the {\em Bayes risk}, i.e., the expected loss with respect to the {\em unknown} ground truth $\truth$.  This procedure is called \defn{minimum Bayes risk (MBR) decoding} and can be approximated with our ensemble of weighted particles:
\begin{subequations}
\begin{align}\label{eqn:mbr}
  \estimate
  &=  {\textstyle { \argmin }_{\unobs \in \seqspace} \sum_{\truth \in \seqspace} p(\truth \mid \obs) \loss(\unobs, \truth ) } \\
  &\approx  {\textstyle { \argmin }_{\unobs \in \seqspace} \sum_{m=1}^{M} w_m \loss(\unobs, \unobs_m ) } \label{eqn:decode}
\end{align}
\end{subequations}
where $\loss(\unobs, \truth)$ is the \defn{loss} of $\unobs$ with respect to $\truth$.
This procedure for combining the particles into a single prediction is sometimes called \defn{consensus decoding}.
We now propose a specific loss function $\loss$ and an approximate decoder.

\subsection{Optimal Transport Distance}\label{sec:otd}
The loss of $\unobs$ is defined as the minimum cost of editing $\unobs$ into the ground truth $\truth$. To accomplish this edit, we must identify the best \defn{alignment}---a one-to-one partial matching $\ali$---of the events in the two sequences.  
We require any two aligned events to have the same type $k$.  
We define $\ali$ as a collection of alignment edges $(t, t^*)$ where $t$ and $t^*$ are the times of the aligned events in $\unobs$ and $\truth$ respectively. 
An alignment edge between a predicted event at time $t$ (in $\unobs$) and a true event at time $t^*$ (in $\truth$) incurs a cost of $|t-t^*|$ to move the former to the correct time.  Each unaligned event in $\unobs$ incurs a deletion cost of 
$C_{\text{delete}}$, 
and each unaligned event in $\truth$ incurs an insertion cost of 
$C_{\text{insert}}$. Now 
\begin{align}\label{eqn:editdist}
\loss(\unobs, \truth) = \min_{\ali \in \mathcal{A}(\unobs, \truth)} D(\unobs, \truth, \ali)
\end{align}
where $\mathcal{A}(\unobs, \truth)$ is the set of all possible alignments between $\unobs$ and $\truth$, and $D(\unobs, \truth,\ali)$ is the total cost given the alignment $\ali$.
Notice that if $|\unobs|\neq|\truth|$, \emph{any} alignment leaves some events unaligned; also, rather than align two faraway events, it is cheaper to leave them unaligned if $C_{\text{delete}}+C_{\text{insert}} < |t-t^*|$.
\cref{alg:dp} in \cref{sec:dpdetails} uses
dynamic programming to compute the loss \eqref{eqn:editdist} and its corresponding alignment $\ali$, 
similar to edit distance \citep{levenshtein-1965-binary} or dynamic time warping \citep{sakoe-71-dtw,listgarten-05-cpm}. 
In practice we symmetrize the loss by specifying equal costs $C_{\text{insert}} = C_{\text{delete}} = C$. 

\subsection{Consensus Decoding}\label{sec:mbr}
Since aligned events must have the same type, consensus decoding \eqref{eqn:decode} decomposes into {\em separately} choosing a set $\estimate\supk$ of type-$k$ events for \emph{each} $k = 1, 2, \ldots, K$, based on the particles' sets $\unobs_m\supk$ of type-$k$ events.  Thus, we simplify the presentation by omitting $\supk$ throughout this section. 
The loss function $L$ defined in \cref{sec:otd} warrants:
\begin{theorem}\label{thm:comb}
  Given $\{\unobs_m\}_{m=1}^{M}$, if we define $\zunion=\bigsqcup_{m=1}^{M} \unobs_m$, then $\exists \estimate \subseteq \zunion$ such that
  \begin{align*}
    {\textstyle \sum_{m=1}^{M} w_m L( \estimate, \unobs_m ) } 
    = 
    {\textstyle \min_{\unobs \in \seqspace} \sum_{m=1}^{M} w_m L(\unobs, \unobs_m) }
  \end{align*}
  That is to say, there exists one subsequence of $\zunion$ that achieves the minimum Bayes risk.
\end{theorem}
The proof is given in \cref{sec:consensus_details}: it shows that if $\estimate$ minimizes the Bayes risk but is \emph{not} a subsequence of $\zunion$, then we can modify it to either improve its Bayes risk (a contradiction) or keep the same Bayes risk while making it a subsequence of $\zunion$ as desired.

Now we have reduced this decoding problem to a combinatorial optimization problem:
\begin{align}
    \estimate = \argmin_{\unobs \subseteq \zunion} 
    {\textstyle \sum_{m=1}^{M} w_m L( \unobs, \unobs_m) }
\end{align}
which is probably NP-hard, by analogy with the Steiner string problem \cite{gusfield-97-algorithms}.

Our heuristic (\cref{alg:mbr} of \cref{sec:mbr_details}) seeks to iteratively improve $\estimate$ by (1)~using \cref{alg:dp} to find the optimal alignment $\ali_m$ of $\estimate$ with each $\unobs_m$, and then (2)~repeating the following sequence of 3 phases until $\estimate$ does not change.  Each phase tries to update $\estimate$ to decrease the weighted distance $\sum_{m=1}^{M} w_m D( \estimate, \unobs_m, \ali_m)$ which by \cref{thm:comb} is an upper bound of the Bayes risk $\sum_{m=1}^{M} w_m L(\estimate, \unobs_m)$:\footnote{Note these phases compute $D(\estimate, \unobs_m, \ali_m)$ but not $L(\estimate, \unobs_m)$, so they need not call the dynamic programming algorithm.}
\begin{description}[noitemsep,align=left]
\item [Move Phase]
For each event in $\estimate$, move its time to the weighted median (using weights $w_m$) of the times of all $\leq M$ events that $\ali_m$ aligns it to (if any), while keeping the alignment edges.  This selects the new time that minimizes $\sum_{m=1}^{M} w_m D(\estimate, \unobs_m, \ali_m)$.
\item [Delete Phase] For each event in $\estimate$, delete it (together with any related edges in each $\ali_m$) if this decreases $\sum_{m=1}^{M} w_m D(\estimate, \unobs_m, \ali_m)$.
\item [Insert Phase] If we inserted $t$ into $\estimate$, we would also modify each $\ali_m$ to align $t$ to the closest unaligned event in $\unobs_m$ (if any) provided that this decreased $D(\estimate, \unobs_m, \ali_m)$.  Let $\Delta(t)$ be the resulting reduction in $\sum_{m=1}^{M} w_m D(\estimate, \unobs_m, \ali_m)$.  Let $t^* = \argmax_{t \in \zunion, t \notin \estimate} \Delta(t)$.  While $\Delta(t^*) > 0$, insert $t^*$.
\end{description}
The move or delete phase can consider events in any order, or in parallel; this does not change the result.

\section{Experiments}\label{sec:exp}
We compare our particle smoothing method with the strong particle filtering baseline---our neural version of \citet{linderman-17-bayesian}'s Hawkes process particle filter---on multiple real-world and synthetic datasets. See \cref{sec:data_details} for training details (e.g., hyperparameter selection). PyTorch code can be found at {\small \url{https://github.com/HMEIatJHU/neural-hawkes-particle-smoothing}}.

\subsection{Missing-Data Mechanisms}\label{sec:censorship}\label{sec:miss-mech}
We experiment with missingness mechanisms of the form
\begin{equation}\label{eqn:miss-prob}
\pmiss(\unobs \mid \comp) = \prod_{\kt{k_i}{t_i}\in \unobs} \pmi_{k_i}\prod_{\kt{k_i}{t_i}\in \obs} (1-\pmi_{k_i})
\end{equation}
meaning that each event in the complete stream $\comp$ is independently censored with probability $\pmi_k$ that only depends on its event type $k$.\footnote{%
	\Cref{sec:mcem} discusses how $\vec{\pmi}$ could be imputed when complete and incomplete data are both available.} 
We consider both deterministic and stochastic missingness mechanisms. 
For the deterministic experiments, we set $\pmi_k$ for each $k$ to be either $0$ or $1$, so that some event types are always observed while others are always missing. 
Then $\pmiss(\unobs \mid \comp)= 1$ if $\unobs$ consists of precisely the events in $\comp$ that ought to go missing, and $0$ otherwise.  
For our stochastic experiments, we simply set $\pmi_k = \pmi$ regardless of the event type  $k$ and experiment with $\pmi = 0.1, 0.3, 0.5, 0.7, 0.9$.
Then \cref{eqn:miss-prob} can be written as
$\pmiss(\unobs \mid \comp)= (1-\pmi)^{|\obs|} \pmi^{|\unobs|}$, whose value decreases exponentially in the number of missing events $ |\unobs|$.  As this depends on $\unobs$, the stochastic setting is definitely MNAR (not MCAR as one might have imagined).

\subsection{Datasets}\label{sec:data}\label{sec:synthetic}\label{sec:real}\label{sec:taxi}\label{sec:elevator}\label{sec:mimic}
The datasets that we use in this paper range from short sequences with mean length 15 to long ones with mean length $>$ 300.
For each of the datasets, we possess fully observed data that we use to train the model and the proposal distribution.\footnote{The focus of this paper is on inference (imputation) under a given model, so training the model is simply a preparatory step.  However, inference could be used to help train on incomplete data via the EM algorithm, provided that the missingness mechanism is known; see \cref{sec:mcem} for discussion.}
  For each dev and test example, we censored out some events from the fully observed sequence, so we present the $\obs$ part as input to the proposal distribution but we also know the $\unobs$ part for evaluation purposes.
Fully replicable details of the dataset preparation can be found in \cref{sec:exp_details}, including how event types are defined and 
which event types are missing 
in the deterministic settings.

\paragraph{Synthetic Datasets}
We first checked that we could successfully impute unobserved events that are generated from {\em known} distributions. That is, when the generating distribution actually is a neural Hawkes process, could our method outperform particle filtering in practice?
Is the performance consistent over multiple datasets drawn from different processes?
To investigate this, we synthesized 10 datasets, each of which was drawn from a different neural Hawkes process with randomly sampled parameters.

\paragraph{Elevator System Dataset \textnormal{\citep{crites-96-elevator}}.}
A multi-floor building is often equipped with multiple elevator cars
that follow {\em cooperative} strategies to transport passengers
between floors
\citep{lewis-91-elevator,Bao-94-elevator,crites-96-elevator}. 
In this dataset, the events are which elevator car stops at which floor. 
The deterministic case of this domain 
is representative of many real-world cooperative (or competitive) scenarios---observing the activities of some players and imputing those of the others.

\paragraph{New York City Taxi Dataset \textnormal{\citep{whong-14-taxi}}.}
Each medallion taxi in New York City has a sequence of time-stamped pick-up and drop-off events, where different locations have different event types. 
\Cref{fig:method} shows how we impute the pick-up events given the drop-off events (the deterministic missingness case).

\subsection{Data Fitting Results}\label{sec:eval_fit}

First, as an internal check, 
we measure {\em how probable} each ground truth reference $\truth$ is under the proposal distribution constructed by each method, i.e., $\log q(\truth \mid \obs)$. As shown in \cref{fig:cloud-5}, the improvement from particle smoothing is remarkably robust across 12 datasets, improving \emph{nearly every} sequence in each dataset. The plots for the deterministic missingness mechanisms are so boringly similar that we only show them in \cref{sec:extra-exp} (\cref{fig:cloud}). 
\begin{figure}[t]
	\begin{center}
		\begin{subfigure}[b]{0.31\linewidth}
			\includegraphics[width=\linewidth]{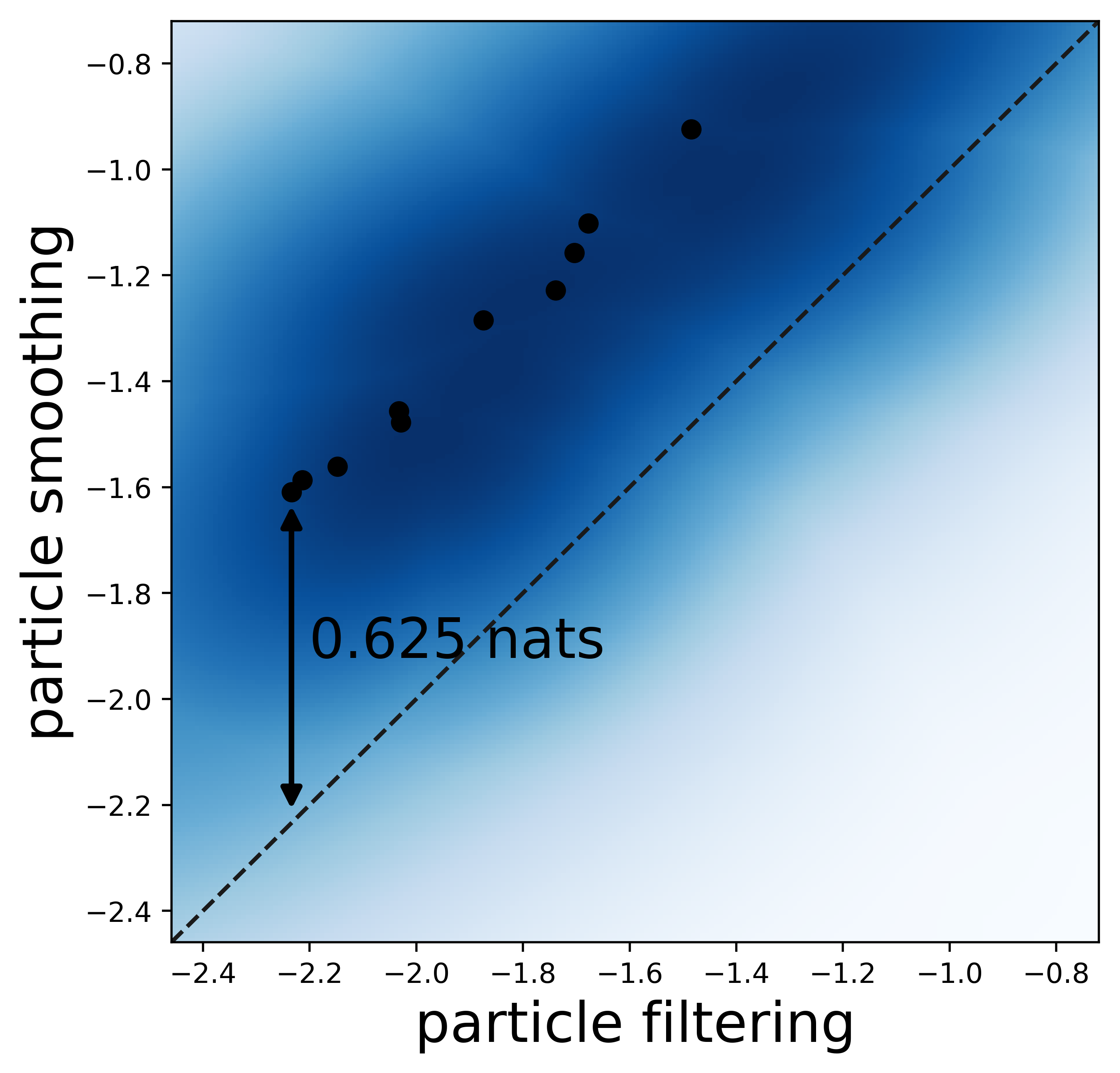}
			\caption{Synthetic Data}\label{fig:nhpcloud-5}
		\end{subfigure}
		~ 
		\begin{subfigure}[b]{0.31\linewidth}
			\includegraphics[width=\linewidth]{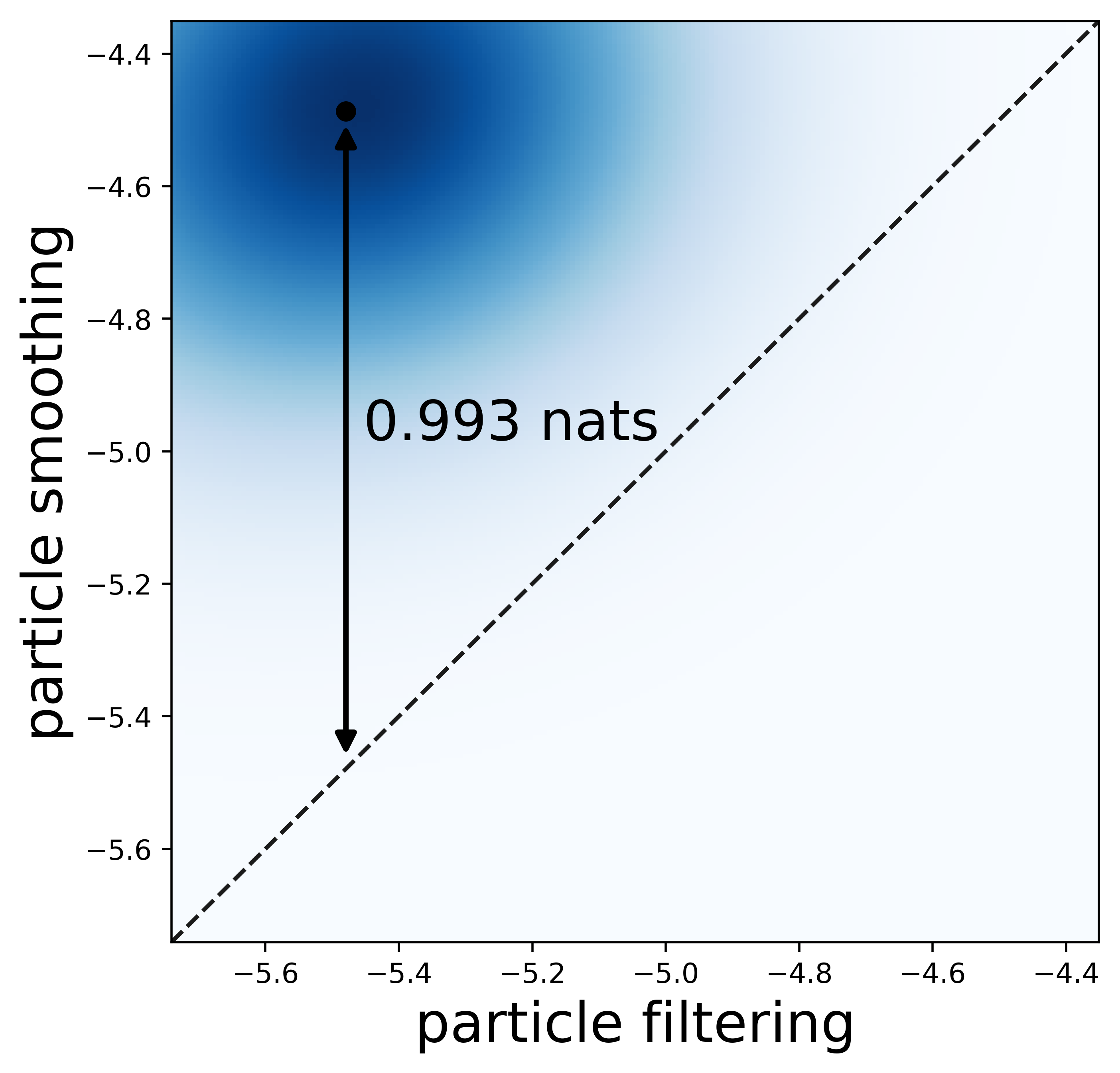}
			\caption{Elevator System}\label{fig:elevatorcloud-5}
		\end{subfigure}
		~ 
		\begin{subfigure}[b]{0.31\linewidth}
			\includegraphics[width=\linewidth]{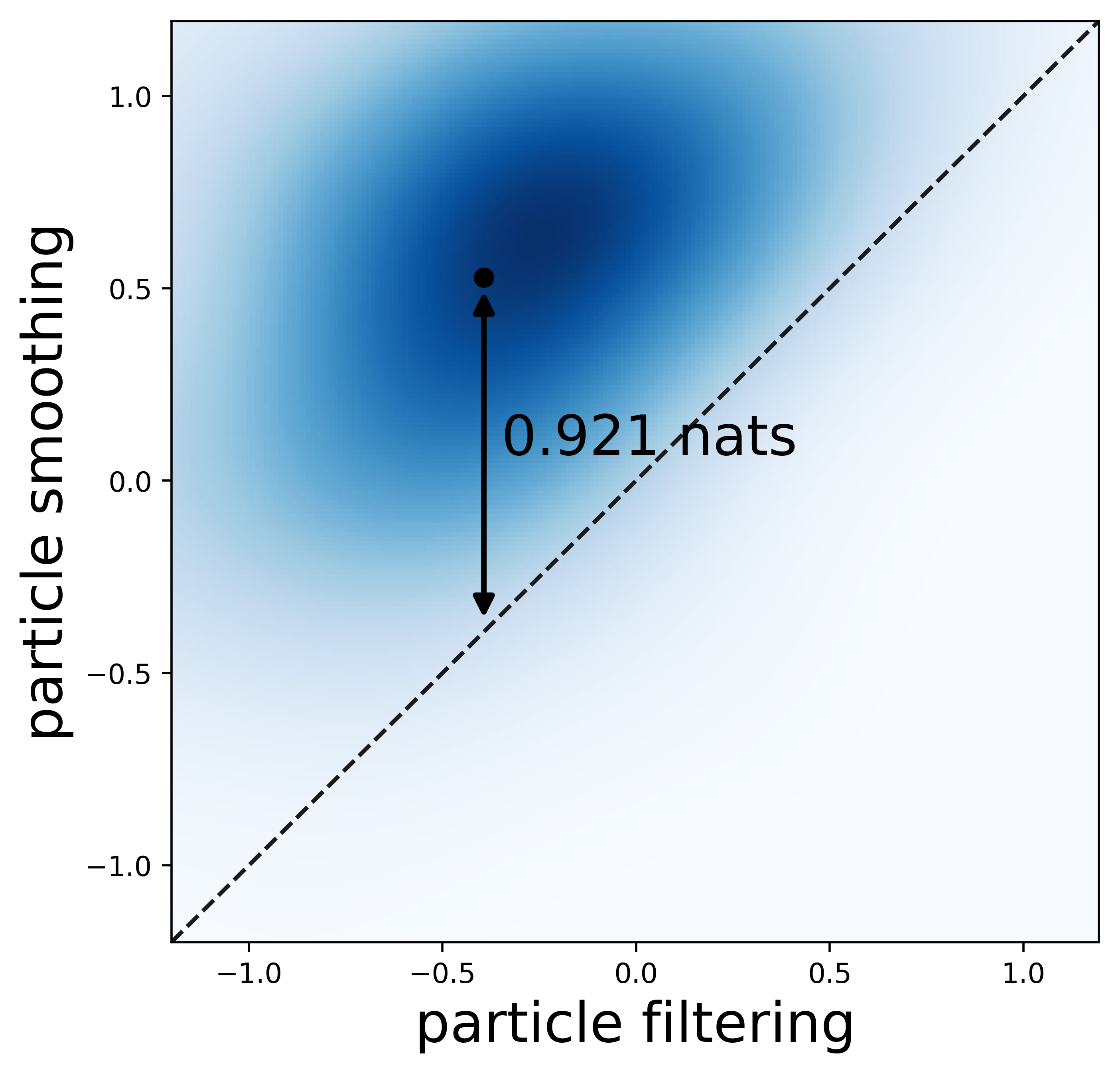}
			\caption{NYC Taxi}\label{fig:taxicloud-5}
		\end{subfigure}
		\caption{
		Scatterplots of neural Hawkes particle smoothing (y-axis) vs.\@ particle filtering (x-axis) with a stochastic missingness mechanism ($\pmi=0.5$).
		Each point represents a single \heldout sequence, and compares the values of $\log q(\truth \mid \obs) \;/\; |\truth|$.  Larger values mean that the proposal distribution is better at proposing the ground truth $\truth$.  Each dataset's scatterplot is converted to a cloud using kernel density estimation, with the centroid denoted by a black dot. A double-arrowed line indicates the improvement of particle smoothing over filtering. For the synthetic datasets, we draw ten clouds on the same figure and show the line for the dataset where smoothing improves the most. As we can see, the density  is always well concentrated above $y = x$.  That is, this is not merely an average improvement: \emph{nearly every} ground truth $\truth$ gets higher proposal probability! Particle smoothing performs well even on datasets where particle filtering performs badly.		}\label{fig:cloud-5}
	\end{center}
\end{figure}

\subsection{Decoding Results}\label{sec:eval_decode}

For each $\obs$, we now make a prediction by sampling an ensemble of $M=50$ particles (\cref{sec:method})\footnote{Increasing $M$ would increase both effective sample size (ESS) and runtime.}
 and constructing their consensus sequence $\estimate$ (\cref{sec:mbr}).  We use multinomial resampling since otherwise the effective sample size is very low (only 1--2 on some datasets).\footnote{Any multinomial resampling step drives the ESS metric to $M$.  This cannot guarantee better samples in general, but resampling did improve our decoding performance on all datasets.} 
We evaluate $\estimate$ by its optimal transport distance (\cref{sec:otd}) to the ground truth $\truth$.
Note that $\forall \ali$, we can decompose $D(\estimate, \truth, \ali)$ as 
\begin{align}\label{eqn:decomp}
C \cdot ( \underbrace{|\estimate| + |\truth| - 2 |\ali|}_{\text{total insertions+deletions}} ) + \underbrace{\textstyle{ \sum_{(t,t^*) \in \ali} |t-t^*| }}_{\text{total distance moved}}
\end{align}
Letting $\ali$ be the alignment that minimizes $D(\estimate, \truth, \ali)$,
the former term measures how well $\estimate$ predicts {\em which} events happened, and the latter measures how well $\estimate$ predicts {\em when} those events happened. 
Different choices of $C$ yield different $\estimate$ with different trade-offs between these two terms.
Intuitively, when $C \approx 0$, the decoder is free to insert and delete event tokens;
as $C$ increases, $\estimate$ will tend to insert/delete fewer event tokens and
move more of them.

\Cref{fig:eval-5} plots the performance of particle smoothing (\reddot) vs.\@ particle filtering (\bluedot) for the stochastic missingness mechanisms, showing the two terms above as the $x$ and $y$ coordinates.  The very similar plots for the deterministic missingness mechanisms are in \cref{sec:extra-exp} (\cref{fig:eval}).\footnote{We show the 2 real datasets only.  The figures for the 10 synthetic datasets are boringly similar to these.}

\begin{figure}[t]
	\begin{center}
		\begin{subfigure}[b]{0.48\linewidth}
			\includegraphics[width=\linewidth]{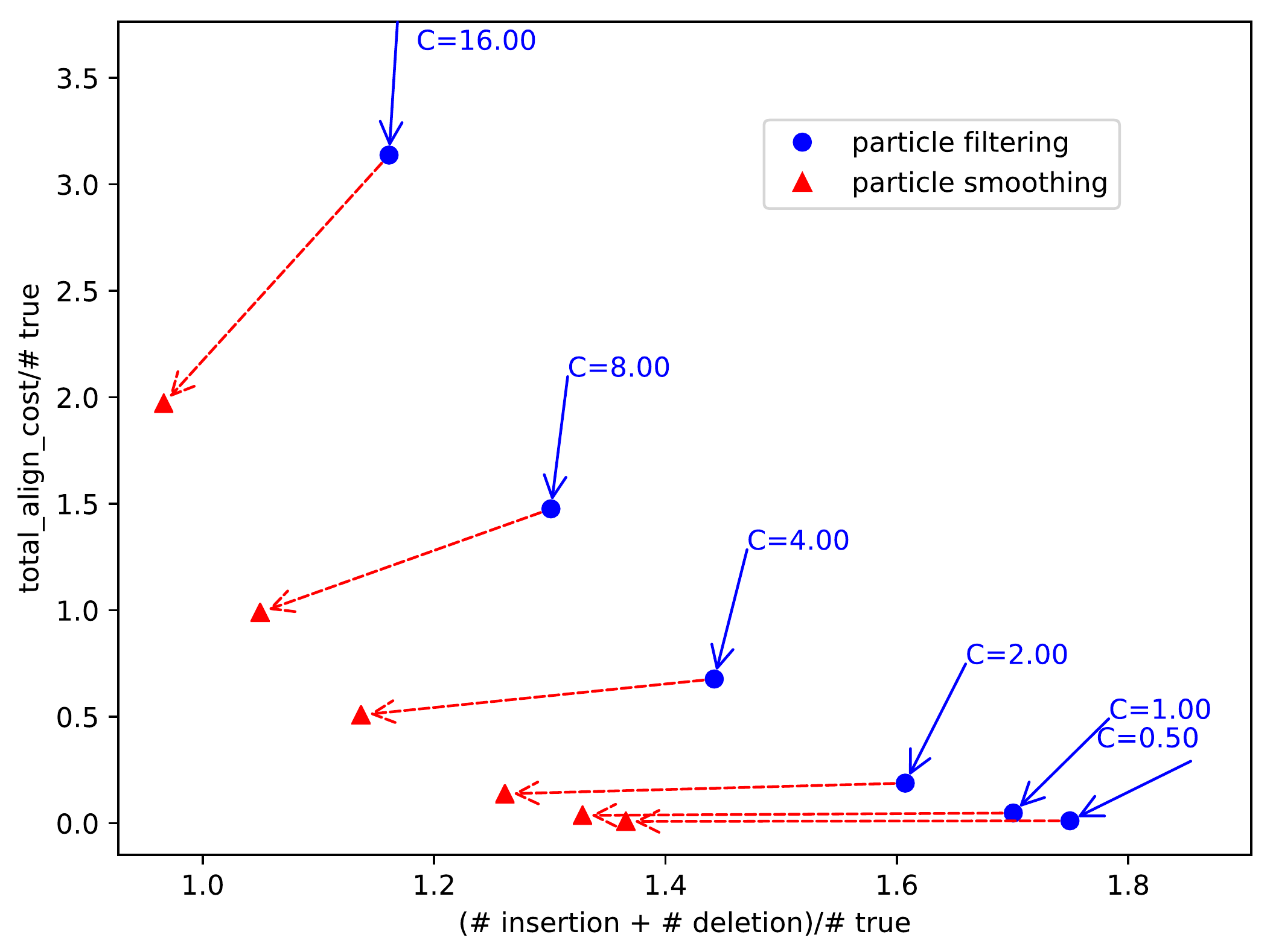}
			\caption{Elevator System}\label{fig:elevator-5}
		\end{subfigure}
		~ 
		\begin{subfigure}[b]{0.48\linewidth}
			\includegraphics[width=\linewidth]{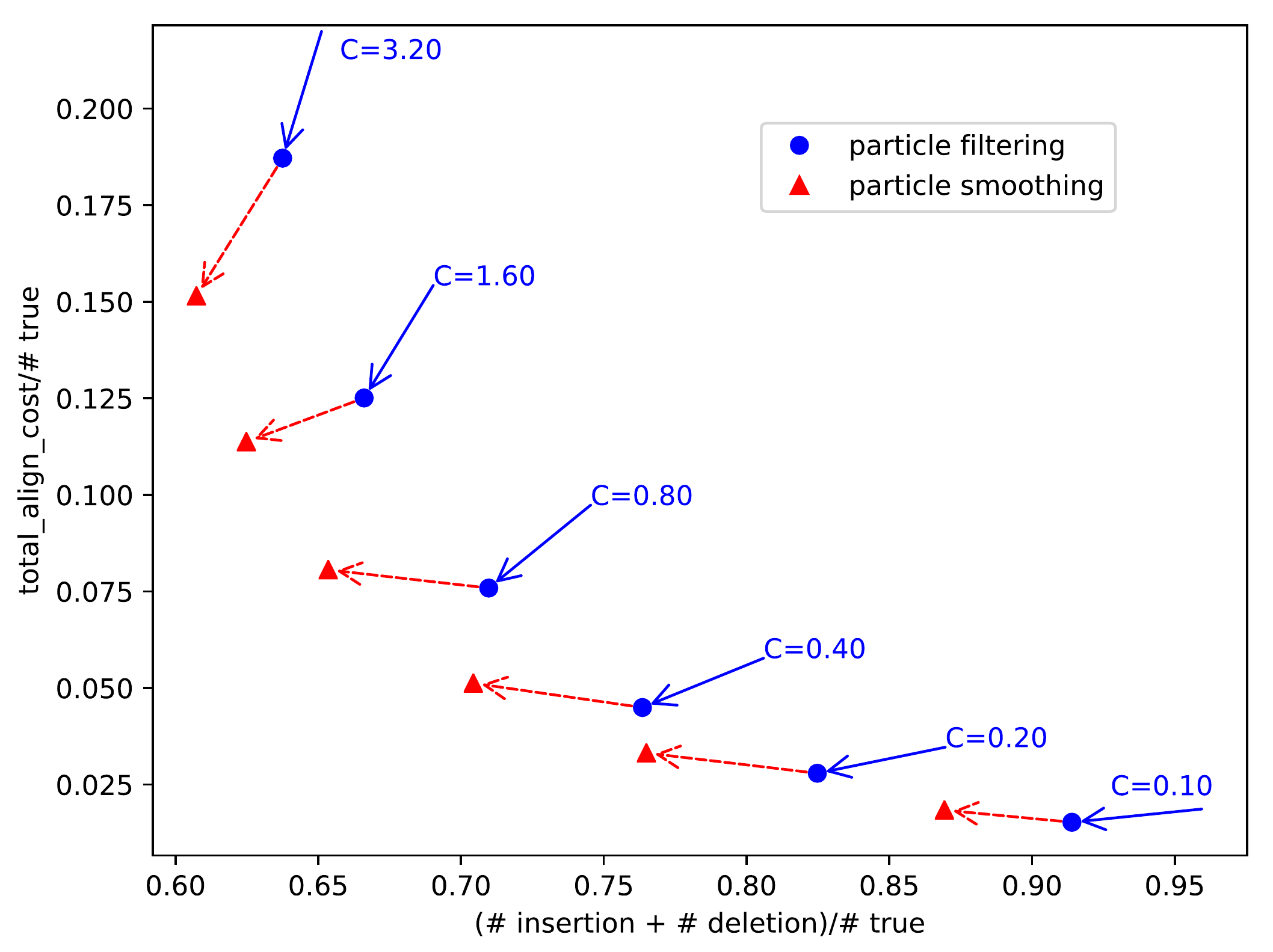}
			\caption{NYC Taxi}\label{fig:taxi-5}
		\end{subfigure}
	\caption{Optimal transport distance of particle smoothing (\reddot) vs.\@ particle filtering (\bluedot) on \heldout data with a stochastic missingness mechanism ($\pmi=0.5$). 
		In each figure, the $x$-axis is the total number of deletions and insertions in the \heldout dataset, $\sum_{n=1}^N ( |\estimate_n|+|\truth_n|-2|\ali_n| )$, and the $y$-axis is the total movement cost, $\sum_{n=1}^N \sum_{(t,t^*)\in \ali_n} |t-t^*|$. Both axes are normalized by the true total number of missing events $\sum_{n=1}^N |\truth_n|$, so the $x$-axis shows a fraction and the $y$-axis shows an average time difference.
		On each dataset, we show one \bluedot per $C$.
		According to \cref{eqn:decomp}, $(C, 1)$, denoted by \blueline, turns out to be the \emph{gradient} of $\sum_{n=1}^N D(\estimate_n, \truth_n, \ali_n)$ at this \bluedot.
		The \redline
		shows the actual improvement obtained by switching to particle smoothing (which is, indeed, an improvement because it has positive dot product with the gradient \blueline). 
		The Pareto frontier (convex hull) of the \reddot symbols dominates the Pareto frontier of the \bluedot symbols---lying everywhere to its left---which means that our particle smoothing method outperforms the filtering baseline.
	}\label{fig:eval-5}
	\end{center}
\end{figure}

\subsection{Sensitivity to Missingness Mechanism}\label{sec:sensitivity}

For the stochastic missingness mechanisms, we also did experiments with different values of missing rate $\pmi=0.1, 0.3, 0.7, 0.9$. Our particle smoothing method consistently outperforms the filtering baseline in all the experiments
(\cref{fig:diff_rhos} in \cref{sec:sensitivity_details}),  similar to \cref{fig:eval-5}. 

\subsection{Runtime}\label{sec:runtime}

The theoretical runtime complexity is $O(MI)$ where $M$ is the number of particles and $I$ is the number of observed events.  In practice, we generate the particles in parallel, leading to acceptable speeds of 300-400 milliseconds per event for the final method. More details about the wall-clock runtime can be found in \cref{sec:wallclock}.

\section{Discussion and Related Work}\label{sec:related}

To our knowledge, this is the first time a bidirectional recurrent neural network has been extended to predict events in continuous time.  Bidirectional architectures have proven effective at predicting linguistic words and their properties given their left \emph{and right} contexts~\citep{graves-13-hybrid,bahdanau-15,peters-18-elmo,devlin-18-bert}: in particular, \citet{lin-eisner-2018-naacl} recently applied them to particle smoothing for discrete-time sequence tagging. 

Previous work that infers unobserved events in continuous time exploits special properties of simpler models, including Markov jump processes \citep{rao-12-mcmc,rao-13-fast}, continuous-time Bayesian networks \citep{fan-10-is} and Hawkes processes \citep{shelton-18-missing}.
Such properties no longer hold for our more expressive neural model, necessitating our approximate inference
method.  

Metropolis-Hastings would be an alternative to our particle smoothing method. The transition kernel could propose a single-event change to $\unobs$ (insert, delete, or move). Unfortunately, this would be quite slow for a neural model like ours, because any proposed change early in the sequence would affect the LSTM state and hence the probability of all subsequent events. Thus, a single move takes $O(|\comp|)$ time to evaluate. Furthermore, the Markov chain may mix slowly because a move that changes only one event may often lead to an incoherent sequence that will be rejected. The point of our particle smoothing is essentially to avoid such rejection by proposing a {\em coherent sequence of events}, left to right but considering future $\obs$ events, from an approximation $\proposal$ to the true posterior.
(One might build a better Metropolis-Hastings algorithm by designing a transition kernel that makes use of our current proposal distribution, e.g., via particle Gibbs \cite{chopin-singh-2015}.)

\vspace{-1pt}
We also introduced an optimal transport distance between event sequences, which is a valid metric. It essentially regards each event sequence as a 0-1 function over times, and applies a variant of Wasserstein distance \citep{villani-08-optimal} or Earth Mover's distance \citep{kantorovitch-58-translocation,levina-01-earth}. 
Such variants are under active investigation \citep{benamou-03-numerical,chizat-15-unbalanced,frogner-15-learning,chizat-18-interpolating}.
Our version allows event insertion and deletion during alignment, where these operations can only apply to an entire event---we cannot align half of an event and delete the other half. Due to these constraints, dynamic programming rather than a linear programming relaxation is needed to find the optimal transport.
\cutforspace{In this sense, our distance is similar to edit distance  or dynamic time warping}
\citet{xiao-17-wgan} also proposed an optimal transport distance between event sequences that allows event insertion and deletion; however, their insertion and deletion costs turn out to depend on the timing of the events in (we feel) a peculiar way.

\vspace{-1pt}
We also gave a method to find a single ``consensus'' reconstruction with small average distance to our particles.  This problem is related to Steiner string \citep{gusfield-97-algorithms}, which is usually reduced to multiple sequence alignment (MSA) \citep{mount-04-msa} and heuristically solved by progressive alignment construction using a guide tree \citep{feng-87-progressive,larkin-07-clustal,notredame-00-coffee} and iterative realignment of the initial sequences with addition of new sequences to the growing MSA \citep{hirosawa-95-comprehensive,gotoh-96-significant}.  These methods might also be tried in our setting.  For us, however, the $i$\textsuperscript{th} event of type $k$ is not simply a character in a finite alphabet such as $\{\texttt{A,C,G,T}\}$ but a time that falls in the infinite set $[0,T)$.  The substitution cost between two events of type $k$ is then their time difference.

On multiple synthetic and real-world datasets, our method turns out to be effective at inferring the ground truth sequence of unobserved events. The improvement of particle smoothing upon particle filtering is substantial and consistent, showing the benefit of training a proposal distribution.

\section*{Acknowledgments}

We are grateful to Bloomberg L.P.\@ for enabling this work through a Ph.D.\@ Fellowship Award to the first author.  The work was also supported by the National Science Foundation under Grant No.\@ 1718846.
We thank the anonymous reviewers, Hongteng Xu at Duke University, and our lab group at Johns Hopkins University's Center for Language and Speech Processing for helpful comments. 
We also thank NVIDIA Corporation for kindly donating two Titan X Pascal GPUs and the state of Maryland for the Maryland Advanced Research Computing Center.
The first version of this work appeared on OpenReview in September 2018.

\bibliographystyle{icml2019}
\vspace{0pt}
\bibliography{particle-smoothing}
\clearpage
\appendix
\appendixpage

\section{\citet{little-rubin-1987}'s Missing-Data  Taxonomy}\label{sec:miss_details}

\citet{little-rubin-1987}'s classical taxonomy of MNAR, MAR, and MCAR mechanisms\footnote{\defn{Missing not at random (MNAR)} makes no assumptions about the missingness mechanism.  \defn{Missing at random (MAR)} is a modeling assumption: determining from data whether the MAR property holds is ``almost impossible'' \citep{mohan-18-graphical}.  \defn{Missing completely at random (MCAR)} is a simple special case of MAR.} was meant for graphical models. A graphical model has a \emph{fixed set} of random variables.  The missingness mechanisms envisioned by \citet{little-rubin-1987} simply decide which of those variables are suppressed in a joint observation.  For them, an observed sample always reveals \emph{which} variables were observed, and thus it reveals \emph{how many} variables are missing.

In contrast, our incomplete event stream is most simply described as a single random variable $\Comp$ that is \emph{partly} missing.  If we tried to describe it using $|\comp|$ random variables with values like  $\kt{k}{t}$, then the observed sample $\obs$ would \emph{not} reveal the number of missing variables $|\unobs|$ nor the total number of variables $|\comp|$.  There would not be a fixed set of random variables.

To formulate our model in \citeauthor{little-rubin-1987}'s terms, we would need a fixed set of uncountably many random variables $K_t$ where $t$ ranges over the set of times.  $K_t = k$ if there is an event of type $k$ at time $t$, and otherwise $K_t = 0$.  For some finite set of times $t$, we observe a specific value $K_t > 0$, corresponding to some observed event.  For all other times $t$, the value of $K_t$ is missing, meaning that we do not know whether or not there is any event at time $t$, let alone the type of such an event.  
A crucial point is that 0 values are never observed in our setting, because we are never told that an event did {\em not} happen at time $t$.  In contrast, a value $> 0$ (corresponding to an event) may be either observed or unobserved.  Thus, the probability that $K_t$ is missing depends on whether $K_t > 0$, meaning that this setting is MNAR.

We preferred to present our model (\cref{sec:partial}) in terms of the finite sequences that are generated or read by our LSTMs.  This simplified the notation later in the paper.  But it does not cure the MNAR property: see \cref{sec:miss-mech}.

Again, our presentation does not allow a \citet{little-rubin-1987} style formulation in terms of a finite fixed set of random variables, some of which have missing values.  That formulation would work if we knew the total number of events $I$, and were simply missing the value $k_i$ and/or $t_i$ for some indices $i$.  But in our setting, the number of events is itself missing: after each observed event $i$, we are missing $J_i$ events where $J_i$ is itself unknown.  In other words, we need to impute even the number of variables in the complete sequence $\comp$, not just their values.

Our definition of MAR in \cref{sec:partial} is the correct generalization of \citet{little-rubin-1987}'s definition: namely, it is the case in which the second factor of \cref{eqn:target} can be ignored. The ability to ignore that factor is precisely why anyone cares about the MAR case.  This was mentioned at \cref{eqn:target},
and is discussed in conjunction with the EM algorithm in \cref{sec:mcem}.

Since missing-event settings tend to violate this desirable MAR property, all our experiments address MNAR problems.  As \citet{little-rubin-1987} explained, the more general case of MNAR data cannot be treated without additional knowledge.  The difficulty is that identifying $\model$ {\em jointly} with $\pmiss$ becomes impossible. If you observe few 50-year-olds on your survey, you cannot know (beyond your prior) whether that's because there are few 50-year-olds, or because 50-year-olds are very likely to omit their age.

Fortunately, we do have additional knowledge in our setting.
Joint identification of $\model$ and $\pmiss$ is unnecessary if either
(1)~one has separate knowledge of the missingness distribution $\pmiss$, or
(2)~one has separate knowledge of the complete-data distribution $\model$.
In fact, both (1) and (2) hold in the MNAR experiments of this paper (\crefrange{sec:miss-mech}{sec:data}).
But in general, if we know at least one of the distributions, then we can still infer the other
(\cref{sec:mcem}). 

\subsection{Obtaining Complete Data}

Readers might wonder why (2) above would hold in a missing-data situation.  In practice, where would we obtain a dataset of complete event streams (as in \cref{sec:data}) for supervised training of $\model(\comp)$?

In some event stream scenarios, a training dataset of complete event streams can be collected at extra expense.  This is the hope in the medical and user-interface scenarios in \cref{sec:intro}.  For our imputation method to work on partially observed streams $\comp$, their complete streams should be distributed like the ones in the training dataset.

Other scenarios could be described as having \defn{eventually complete} streams.  Complete information about each event stream \emph{eventually} arrives, at no extra expense, and that event stream can then be used for training.  For example, in the competitive game scenario in \cref{sec:intro}), perhaps including wars and political campaigns, each game's true complete history is revealed after the game is over and the need for secrecy has passed.  While a game is underway, however, some events are still missing, and imputing them is valuable.  Both (1) and (2) hold in these settings.

An interesting subclass of eventual completeness arises in monitoring settings such as epidemiology, journalism, and sensor networks.  These often have \defn{reporting delays}, so that one observes each event only some time after it happens.  Yet one must make decisions at time $t < T$ based on the events that have been observed so far.  This may involve imputing the past and predicting the future.  The missingness mechanism for these reporting delays says that more recent events (soon before the current time $t$) are more likely to be missing.   The probability that such an event would be missing depends on the specific distribution of delays, which can be learned with supervision once all the data have arrved.

We point out that in all these cases, the ``complete'' streams $\comp$ that are used to train $\model$ do not actually have to be \emph{causally} complete.  It may be that in the real world, there are additional latent events $\vec{w}$ that cause the events in $\comp$ or mediate their interactions.  \Citet[section 6.3]{mei-17-neuralhawkes} found that the neural Hawkes process was expressive enough in practice to ignore this causal structure and simply use $\comp$ streams to directly train a neural Hawkes process model $\model(\comp)$ of the \emph{marginal} distribution of $\comp$, without explicitly considering $\vec{w}$ in the model or attempting to sum over $\vec{w}$ values.  The assumption here is the usual assumption that $\comp$ will have the same distribution in training and test data, and thus $\vec{w}$ will be missing in both, with the same missingness mechanism in both.  By contrast, $\unobs$ is missing only in test data.  It is not possible to impute $\vec{w}$ because it was not modeled explicitly, nor observed even in training data.  However, it remains possible to impute $\unobs$ in test data based on its distribution in training data.

\section{Complete Data Model Details}\label{sec:model_details}

Our complete data model, such as a neural Hawkes process, gives the probability  $\model(\comp)$ that $\comp$ will be the complete set of events on a given interval $[0,T)$.  this probability can always be written in the factored form
\begin{align}\label{eqn:model}
\Big( \prod_{i=0}^{I} \prod_{j=0}^{J_i} p\!\left( \kt{k_{i,j}}{t_{i,j}} \mid \history(t_{i,j}) \right) \Big) \cdot p(\kt{}{\!\geq\!T} \mid \history(T) )
\end{align}
where $p(\kt{k}{t} \mid \history(t))$ denotes the probability density that the first event following $\history(t)$ (which is the set of events occurring \emph{strictly} before $t$) will be $\kt{k}{t}$, and $p(\kt{}{} \geq t' \mid \history(t))$ denotes the probability that this event will fall at some time $\geq t'$.  

Thus, the final factor of \eqref{eqn:model} is the probability that there are no more events on $[0,T)$ following the last event of $\comp$.  The initial factor $p(\kt{k_0}{t_0} \mid \history{t_0})$ is defined to be 1, since the boundary event $\kt{k_0}{t_0}$ is given (see \cref{sec:partial}).

\subsection{Neural Hawkes Process Details}\label{sec:nhp_details}

In this section we elaborate slightly on \cref{sec:nhp}.  Again, $\lambda_k(t \mid \history(t))$ is defined by \cref{eqn:hawkes_mod_a} in terms of the hidden state of a \emph{continuous-time} left-to-right LSTM. We spell out the continuous-time LSTM equations here; more details about them may be found in \citet{mei-17-neuralhawkes}.
\begin{align}
\state(t) = \vec{o}_{i} \odot (2\sigma(2\cell(t))-1) \text{ for } t \in (t_{i-1}, t_{i}]
\end{align}
where the interval $(t_{i-1},t_{i}]$ has consecutive observations $\kt{k_{i-1}}{t_{i-1}}$ and  $\kt{k_{i}}{t_{i}}$ as endpoints.
At $t_i$, the continuous-time LSTM reads $\kt{k_{i}}{t_{i}}$ and updates the current (decayed) hidden cells $\cell(t)$ to new initial values $\cell_{i+1}$, based on the current (decayed) hidden state $\state(t_{i})$, as follows:\footnote{The upright-font subscripts $\mathrm{i}$, $\mathrm{f}$, $\mathrm{z}$ and $\mathrm{o}$ are not variables, but constant labels that distinguish different $\vec{W}$, $\vec{U}$ and $\vec{d}$ tensors. The $\statforgetgate$ and $\statinputgate$ in~\cref{eqn:ct_cell_target_forward} are defined analogously to $\forgetgate$ and $\inputgate$ but with different weights.}
\minipage[t]{0.49\textwidth}
\begin{subequations} 
	\begin{align}
	\inputgate_{i+1}
	&\leftarrow \sigma \left( \vec{W}_{\mathrm{i}} \vec{k}_{i} + \vec{U}_{\mathrm{i}} \state(t_{i}) + \vec{d}_{\mathrm{i}}  \right) \\% \in (0,1)^D \\
	\forgetgate_{i+1}
	&\leftarrow \sigma \left( \vec{W}_{\mathrm{f}} \vec{k}_{i} + \vec{U}_{\mathrm{f}} \state(t_{i}) + \vec{d}_{\mathrm{f}}  \right) \\% \in (0,1)^D \\
	\zgate_{i+1}
	&\leftarrow 2 \sigma \left( \vec{W}_{\mathrm{z}} \vec{k}_{i} + \vec{U}_{\mathrm{z}} \state(t_{i}) + \vec{d}_{\mathrm{z}}  \right) - 1 \\% \in (-1,1)^D \\
	\outputgate_{i+1}
	&\leftarrow \sigma \left( \vec{W}_{\mathrm{o}} \vec{k}_{i} + \vec{U}_{\mathrm{o}} \state(t_{i}) + \vec{d}_{\mathrm{o}}  \right)
	\end{align}
\end{subequations}
\endminipage
\hfill
\minipage[t]{0.49\textwidth}
\begin{subequations}
	\begin{align}
	\cell_{i+1}
	&\leftarrow \forgetgate_{i+1} \odot \cell(t_{i}) + \inputgate_{i+1} \odot \zgate_{i+1}\\
	\statcell_{i+1}
	&\leftarrow \statforgetgate_{i+1} \odot \statcell_{i} + \statinputgate_{i+1} \odot \zgate_{i+1} \label{eqn:ct_cell_target_forward} \\
	\decaygate_{i+1}
	&\leftarrow f \left( \vec{W}_{\mathrm{d}} \vec{k}_{i} + \vec{U}_{\mathrm{d}} \state(t_{i}) + \vec{d}_{\mathrm{d}} \right)
	\end{align}
\end{subequations}
\endminipage

The vector $\vec{k}_{i} \in \{ 0,1 \}^K$ is the $i$\th input: a one-hot encoding of the new event $k_{i}$, with non-zero value only at the entry indexed by $k_{i}$.
Then, $\cell(t)$ { for } $t \in (t_{i-1},t_{i}]$ is given by \eqref{eqn:c_decay_forward}, which continues to control $\state(t)$ except that $i$ has now increased by 1).
\begin{align}
\cell(t) &\defeq \statcell_{i+1} + \left( \cell_{i+1} - \statcell_{i+1} \right) \exp \left( -\decaygate_{i+1} \left( t - t_i \right) \right)
\label{eqn:c_decay_forward}
\end{align}

On the interval $(t_{i},t_{i+1}]$, $\cell(t)$ follows an exponential curve that begins at $\cell_{i+1}$ (in the sense that $\lim_{t \rightarrow t_{i}^+} \cell(t) = \cell_{i+1}$) and decays, as time $t$ increases, toward $\statcell_{i+1}$ (which it would approach as $t \rightarrow \infty$, if extrapolated).

The \defn{intensity} $\lambda_k(t \mid \history(t)) \in \Real_{\geq 0}$ may be
thought of as the instantaneous {\em rate} of events of type $k$ at
time $t$.  
More precisely, as $dt \rightarrow 0^+$, the expected number of
events of type $k$ occurring in the interval $[t,t+dt)$, divided by
$dt$, approaches $\lambda_k(t \mid \history(t))$.  If no event of any
type occurs in this interval (which becomes almost sure as
$dt \rightarrow 0^+$), one may still occur in the next interval
$[t+dt,t+2dt)$, and so on.
The intensity functions $\lambda_k(t \mid \history(t))$ are continuous on intervals during which no event occurs (note that $\history(t)$ is constant on such intervals).  They jointly determine a distribution over the time of the next event after $\history(t)$, as used in every factor of \cref{eqn:model}.  As it turns out \citep{mei-17-neuralhawkes}, $\log \model(\Comp=\comp)$ becomes 
\begin{equation}\label{eqn:logpmodel}
\sum_{\ell} \log \inten{k_\ell}{t_\ell} - \int_{t=0}^{T} \sum_{k=1}^{K}\inten{k}{t} dt
\end{equation}
where the first sum ranges over all events $\kt{k_\ell}{t_\ell}$ in $\comp$. 

We can therefore train the parameters $\vec{\theta}$ of the $\lambda_k$ functions by maximizing log-likelihood on training data.  The first term of \cref{eqn:logpmodel} can be differentiated by back-propagation.  \citet{mei-17-neuralhawkes} explain how simple Monte Carlo integration (see also our \cref{sec:integral}) gives an unbiased estimate of the second term of \cref{eqn:logpmodel}, and how the random terms in the Monte Carlo estimate can similarly be differentiated to give a stochastic gradient.  

\section{Sequential Monte Carlo Details}\label{sec:smc_details}

Our main algorithm is presented as \cref{alg:sis}.  It covers both particle filtering and particle smoothing, with optional multinomial resampling.  

\begin{algorithm*}
\caption{Sequential Monte Carlo --- Neural Hawkes Particle Filtering/Smoothing}\label{alg:sis}
\begin{algorithmic}[1]
	\INPUT observed sequence $\obs=\kt{k_0}{t_0},\ldots, \kt{k_{I+1}}{t_{I+1}}$ with $t_0 =0,t_{I+1}=T$;\newline
                    model $p$; missingness mechanism $p_{\text{miss}}$; proposal distribution $q$; number of particles $M$;\newline
                    boolean flags $\mathit{smooth}$ and $\mathit{resample}$
  \OUTPUT collection $\{(\unobs_1,w_1),\ldots,(\unobs_M,w_M)\}$ of weighted particles
  \Procedure{SequentialMonteCarlo}{$\obs, p, p_{\text{miss}}, q, M, \mathit{smooth}, \mathit{resample}$}
  \For{$m = 1$ {\bfseries to} $M$} \Comment{init weighted particles $(\unobs_m,w_m)$.  History $\history_m$ combines $\unobs_m$ with a prefix of $\obs$}
  \State $\unobs_m \gets \text{empty sequence}$; $w_m \gets 1$; $\history_m \gets \text{empty stack}$; push $\kt{k_0}{t_0}$ onto $\history_m$   \label{line:sis:pushBOS}
  \EndFor
   \State $\future \gets \text{empty stack}$
   \For{$i = I$ {\bfseries downto} $0$} \Comment{stack of all future observed events}
           \State push $\kt{k_{i+1}}{t_{i+1}}$ onto $\future$ \Comment{as we reach these events, we'll pop from $\future$ and push onto $\history_m$ ($\forall m$)}\label{line:sis:pushF}
   \EndFor
  \For{$i = 0$ {\bfseries to} $I$} \Comment{propose unobserved events on interval $(t_i, t_{i+1})$, then observe next event $\kt{k_{i+1}}{t_{i+1}}$} \label{line:sis:outer}
	\For{$m = 1$ {\bfseries to} $M$} \label{line:sis:inner}
    \State $\textsc{DrawSegment}(i,m)$ \Comment{destructively extend $\unobs_m, w_m, \history_m$ with events on $(t_i,t_{i+1}]$}
    \EndFor
    \IfThen{$\mathit{resample}$ \& \textsc{LowESS()}}{\textsc{Resample}()} \Comment{optional multinomial resampling replaces all weighted particles}
  \EndFor
  \State\label{line:return_ensemble} \textbf{return} $\{(\unobs_m, w_m/\sum_{m=1}^{M} w_m) \}_{m=1}^{M}$  \Comment{$M$ particles with weights normalized as in \cref{eqn:weight_b}}
  \EndProcedure
  \Procedure{LowESS}{} \Comment{check if effective sample size is low}
  	\State $\text{ESS} \gets {(\sum_{m=1}^{M} w_m)^2} / {\sum_{m=1}^{M} (w_m)^2}$
  	\IfThen{$\text{ESS} < M/2$}{\textbf{return} \textbf{true} }
  	\State \textbf{return} \textbf{false}
  \EndProcedure
  \Procedure{Resample}{}   \Comment{has access to global variables}
    	\For{$m = 1$ {\bfseries to} $M$} \Comment{often draws multiple copies of good (high-weight) particles, 0 copies of bad ones}
    		\State\label{line:normweights} $\tilde{\unobs}_m \sim \Categorical(\{\unobs_m \mapsto \frac{w_m}{\sum_{m=1}^{M} w_m} \}_{m=1}^{M})$ 
    	\EndFor
    	\For{$m = 1$ {\bfseries to} $M$}
    		\State\label{line:unifweights} ${\unobs}_m \gets \tilde{\unobs}_m$; $w_m \gets 1$ \Comment{update particles and their weights}
    	\EndFor
  \EndProcedure
  \Procedure{DrawSegment}{$i,m$} \Comment{has access to global variables}
    \LineComment{$p$ gives info to define function $\lambda^p_k(t) \defeq \lambda_k(t \mid \history_m)$}
    \LineComment{$q$ gives info to define function $\lambda^q_k(t) \defeq \lambda_k(t \mid \history_m, \future)$, or simply $\lambda^q_k(t)=\lambda^p_k(t)$ if $\mathit{smooth}=\textbf{false}$}\label{line:qdef}
    \LineComment{these functions consult the state of a left-to-right LSTM that's read $\history_m$ and possibly a right-to-left LSTM that's read $\future$}
    \LineComment{we also define the \defn{total intensity functions} $\lambda^p(t) \defeq \sum_{k=1}^K
      \lambda^p_k(t)$ and $\lambda^q(t) \defeq \sum_{k=1}^K
      \lambda^q_k(t)$}
    \State $i' \gets i$; $j \gets 0$; $t \gets t_i$ \Comment{where $t_i$ is time of top element of $\history_m$}
    \Repeat \Comment{each iteration adds a new event with index $\angles{i',j} = \angles{i,1},\ldots,\angles{i,J_i},\angles{i+1,0}$}
      \State $j \gets j+1$
      \Repeat\label{line:thinning} \Comment{thinning algorithm \citep[see][]{mei-17-neuralhawkes}}
        \State\label{line:lambdastar} find any $\lambda^{*} \geq \sup\;\{\lambda^q(t'): t' \in (t,t_{i+1}]\}$ \Comment{e.g., old $\lambda^{*}$ still works if $i$ unchanged; see \cref{sec:lambdastar}}
        \State\label{line:expdraw} draw $\Delta \sim \Exp(\lambda^*)$, $u \sim \Uniform(0,1)$
        \State $t \pluseq \Delta$ \Comment{time of next proposed event (before thinning)}
        \If{$t > t_{i+1}$}\label{line:keeptie} \Comment{where $t_{i+1}$ is time of top element of $\future$}
          \State $\kt{k}{t} \gets \text{pop }\future$; $i' \gets i+1$; $j \gets 0$; \Comment{preempt proposed event by $\kt{k_{i+1}}{t_{i+1}}$ (popped from future into present)}\label{line:sis:pop}
          \State \textbf{break}  \label{line:preempt}  
        \EndIf
      \Until{$u\lambda^* \leq \lambda^q(t)$} \Comment{thinning: accept
        proposed time $t$ only with prob $\frac{\lambda^q(t)}{\lambda^*} \leq 1$}
      \LineComment{we've now chosen next event time $t_{i',j}$ to be $t$; let $t_{\text{prev}}$ denote time of top element of $\history_m$}
      \If{$i'=i$}  \Comment{it's a missing event}
        \State draw $k \in \{1,\ldots,K\}$ where probability of $k$ is proportional to $\lambda^q_k(t)$ \Comment{choose next event type $k_{i',j}$}
        \State append $\kt{k}{t}$ to $\unobs_m$
        \State\label{line:integralq}$w_m \gets w_m / 
              \exp{(-\int_{t'=t_{\text{prev}}}^t \lambda^q(t')\,dt' )} \cdot \lambda^q_k(t)$ \Comment{new factor within $q$ in denominator of \eqref{eqn:weight_b}; see \cref{sec:integral}}      
      \EndIf
      \If{$i' \leq I$} \Comment{skip final boundary event $I+1$ (not
        generated by $\model$)}
        \State\label{line:integralp}\label{line:downweight}$w_m \gets w_m\cdot
              \exp{(-\int_{t'=t_{\text{prev}}}^t \lambda^p(t')\,dt' )} \cdot \lambda^p_k(t)$ \Comment{new factor within $\model$ in numerator of \eqref{eqn:weight_b}; see \cref{sec:integral}}      
        \State push $\kt{k}{t}$ onto $\history_m$ \Comment{event $\angles{i,j}$ just generated now becomes part of the past}\label{line:sis:pushH}
        \State\label{line:missfactor} $w_m \gets w_m \cdot p_{\text{miss}}((\kt{k}{t} \in \Unobs) = (i'=i) \mid \history_m)$ \Comment{new factor within $\pmiss$ in numerator of \eqref{eqn:weight_b}: missing or obs}
       \EndIf
    \Until{$i' = i+1$}
  \EndProcedure
\end{algorithmic}
\end{algorithm*}

In this section, we provide some additional details and notes on the design and operation of the pseudocode.  

\subsection{Explicit Formula for the Proposal Distribution}\label{sec:proposal}

The proposal distribution $\proposal$ factors as follows, and the pseudocode uses this factorization to construct $\unobs$ by sampling its individual events from left to right:
\begin{align}\label{eqn:proposal}
&\prod_{i=0}^I \Big( \prod_{j=1}^{J_i} \big( 
q(\kt{k_{i,j}}{t_{i,j}} \mid \history(t_{i,j}), \future(t_{i,j})) \big) \\
&\hspace{15pt}\cdot q(\kt{}{\!\geq\!t_{i+1}} \mid \history(t_{i+1}), \future(t_{i,J_i}) ) \Big) \nonumber
\end{align}
Here the notation for $q(\cdot \mid \cdot)$ is the same as that for $p(\cdot \mid \cdot)$ in \cref{sec:model_details}.   However, the $q(\cdot \mid \cdot)$ terms are proposal probabilities that condition on different evidence---not only the set $\history(t)$ of all events (observed and unobserved) at times $< t$, but also the set $\future(t)$ of events at times $> t$.\footnote{In particular, the second $q$ factor above is the probability that the event at time $t_{i,J_i}$ is the last one before $t_{i+1}$, given knowledge of all past events up through and including the one at $t_{i,J_i}$, and all future observed events starting with the one at $t_{i+1}$.}
All of the proposal probabilities $q(\cdot \mid \cdot)$ are determined by the intensity functions in \eqref{eqn:both}.  

We can sample $\unobs$ from $\proposal$ in chronological order: for each $0 \leq i \leq I$ in turn, draw a sequence of $J_i$ unobserved events that follow the observed event $\kt{k_i}{t_i}$.  The probabilities of these $J_i$ events are the inner factors in \cref{eqn:proposal}.  This sequence ends (thereby determining $J_i$) if the next proposed event would have fallen after $t_{i+1}$ and thus is preempted by the observed event $\kt{k_{i+1}}{t_{i+1}}$.  The probability of so ending the sequence corresponds to the $q(\kt{}{\!\geq\!t_{i+1}} \mid \cdots)$ factor in \cref{eqn:proposal}.

\Cref{eqn:proposal} resembles \cref{eqn:model},
but it conditions each proposed unobserved event not only on the history but also on the {future}.
\Cref{sec:train} tries to train $\proposal$ to approximate the target distribution $\target$, by making $q(\cdot \mid \history,\future) \approx p(\cdot \mid \history,\future)$.  In other words, at each step, $q$ should draw the next proposed event approximately from the posterior of the model $p$, even though we have no closed form computation for that posterior.

Just as \cref{eqn:model} yields the formula \eqref{eqn:logpmodel} for $\log \model$ when we use a neural Hawkes process model, 
\cref{eqn:proposal} yields the following formula for $\log \proposal$ when we use the proposal intensities from \eqref{eqn:both}:
\begin{align}
&\sum_{\ell} \log \intenbothq{k_\ell}{t_\ell} \nonumber \\
&\qquad - \int_{t=0}^{T} \sum_{k=1}^{K}\intenbothq{k}{t}\,dt \label{eqn:logq}
\end{align}
where the first sum ranges over all events $\kt{k_\ell}{t_\ell}$ in $\unobs$ only. 

\subsection{Managing LSTM State Information}\label{sec:lstm_states}

The push and pop operations shown in the pseudocode must be implemented so that they also have the effect of updating LSTM configurations.

Our $\model$ uses a left-to-right LSTM to construct its state after reading all events so far from left to right (\cref{sec:nhp}).  Since each particle posits a different event sequence, we maintain a separate LSTM configuration for each particle $m=1,2,\ldots,M$.  If $\textit{smooth}=\textbf{true}$, our $q$ additionally uses a right-to-left LSTM whose state has read all future observed events from right to left (\cref{sec:nhps}).  We maintain the configuration of this LSTM as well.  

Specifically, in \cref{alg:sis}, when we push an event to the stack $\history_m$ (\cref{line:sis:pushBOS,line:sis:pushH}), we update the configuration of particle $m$'s left-to-right LSTM (including gates, cell memories and states).  

If $\textit{smooth}=\textbf{true}$, then when we push an event to the stack $\future$ (\cref{line:sis:pushF}), we update the configuration of the right-to-left LSTM.  Moreover, before updating that configuration, we push it onto an parallel stack, so that we can revert the update when we later \emph{pop} the event from $\future$ (\cref{line:sis:pop}).

These LSTM configurations provide the $\state(t)$ and $\stateb(t)$ vectors for the computation of intensities $\lambda^{p}_k(t)$ and $\lambda^{q}_k(t)$ in \cref{eqn:hawkes_mod_a,eqn:both}.  These intensities are needed in \cref{line:integralq,line:integralp} of the algorithm.

\subsection{Integral Computation}\label{sec:integral}

\citet[section B.2]{mei-17-neuralhawkes} construct a Monte Carlo estimator of the $\int_0^T$ integral in \cref{eqn:logpmodel}, by evaluating $\sum_k \inten{k}{t} \cdot T$ at a random $t \sim \Uniform(0,1)$.  While even one such sample would provide an unbiased estimate, they draw $N=O(I)$ such samples, where $I$ is the number of events, and average over these samples.  This reduces the variance of the estimator, which decreases as $O(1/N)$.  Notice that because they sample the $N$ time points uniformly on $[0,T)$, longer intervals between observed events will tend to contain more points, which is appropriate.

\citet{mei-17-neuralhawkes} (Appendix C.2) found that rather few samples could be used to estimate the integral.  Indeed, even sampling at only $I$ time points gave a standard deviation of log-likelihood---for the whole sequence---that was on the order of 0.1\% of absolute (Mei, p.c.).

Our particle methods involve \emph{comparing probabilities}.  For each observed sequence $\obs$, we use \eqref{eqn:weight_b} to reweight the $M$ particles according to their probability under the model divided by their probability under the proposal distribution.  This means contrasting \emph{two} probabilities for each particle (the $p$ and $q$ probabilities).  It also means \emph{comparing} the resulting probability ratios across all $M$ particles, resulting in the normalized weights of \cref{eqn:weight_b}.

For each of the $M$ particles, the $\model$ factor in \cref{eqn:weight_b} is obtained by exponentiating \cref{eqn:logpmodel}, while the $q$ factor is obtained by exponentiating \cref{eqn:logq}.  This means that each of these $2M$ factors contains the $\exp$ of an integral.  To make all of these integral estimates more comparable and thus reduce the variance in the importance weights $w_m$ (\cref{eqn:weight_b}), we evaluate all $2M$ integrals at the same set of $N$ time points (see \cref{sec:wallclock}).  This practice ensures a ``paired comparison'' among particles: $w_m$ and $w_{m'}$ differ only because they have different intensities at the sampled points, and not also because they sample at different points.

In \cref{alg:sis}, these integral estimates are accumulated gradually at \cref{line:integralp,line:integralq}.  The idea is that particle $\unobs_m$ partitions $[0,T)$ into the intervals between successive events of $\comp_m$.  Thus, the (estimated) integral over $[0,T)$ can be expressed by summing the (estimated) integrals over these intervals.  The estimate over an interval uses only the small subset of the $N$ time points that fall into the interval.  When we exponentiate the integrals to convert from log-probabilities into probabilities, this sum turns into a product, as shown at \cref{line:integralp,line:integralq}.

This gradual accumulation method gives the same result as if we had computed each integral ``all at once'' before exponentiating.  However, it is useful to begin weighting the particles before they are complete.  After each event $\kt{k_i}{t_i}$ (for $0 \leq i \leq I+1$), the partial particles up through this event already have partial weights $w_m$.  It is these partial weights that are used by the $\textsc{Resample}$ procedure (when $\mathit{resample}=\textbf{true}$).

In all experiments in this paper, we first sampled $I+1$ points uniformly on $[0,T)$, for an average of only 1 time point per interval.  In addition, for each interval $(t_i,t_{i+1})$, we sampled 1 point uniformly on that interval if it did not yet contain any points.  Thus, $N \in [I+1,2I+1]$.  

Sampling at more points might be wise in settings where there are many missing events per interval (e.g., large $\pmi$ in \cref{sec:miss-mech}).  This is especially true when $\mathit{resample}=\textbf{true}$.  Resampling allows us to try multiple extensions of a high-weight particle; at the next resampling step, we prefer to keep the extensions that fared best.  The danger is that if only a few sampling points happen to fall between resampling steps, then we may make a poor (high-variance) estimate of which extensions fared best.

For our setting, however, we found only negligible changes in the results by increasing to 5 time points per interval (i.e., sampling $5I+5$ points at the first step).  Our evaluation metric (the minimum of \eqref{eqn:decomp} over all alignments $\ali$) became slightly better for some values of $C$ and slightly worse for others, but never by more than 2\% relative.  This is about the same variance that we get across different runs (with different random seeds) that have 1 time point per interval.  Thus, we report only the results of the faster scheme.  We caution that this hyperparameter might be more important in other settings.  

\subsection{Choice of $\lambda^*$}\label{sec:lambdastar}

How do we construct the upper bound $\lambda^*$ (\cref{line:lambdastar} of \cref{alg:sis})? For particle filtering, we follow the recipe in B.3 of \citet{mei-17-neuralhawkes}: we can express $\lambda^* = f_k(\max_{t}g_1(t)+\ldots+\max_{t}g_n(t))$ where each summand $v_{kd} h_d(t) = v_{kd} \cdot o_{id} \cdot (2\sigma(2c_d(t))-1)$ is upper-bounded by $\max_{c \in \{c_{id},\underline{c}_{id}\}} v_{kd} \cdot o_{id} \cdot (2\sigma(2c)-1)$.  Note that the coefficients $v_{kd}$ may be either positive or negative.

For particle smoothing, we simply have more summands inside $f_k$ so $\lambda^* = f_k(\max_{t}g_1(t)+\ldots+\max_{t}g_n(t) + \max_{t}\bar{g}_1(t)+\ldots+\max_{t}\bar{g}_{\bar{n}}(t))$ where each extra summand $u_{kd} \bar{h}_d(t) = u_{kd} \cdot \bar{o}_{id} \cdot (2\sigma(2\bar{c}_d(t))-1)$ is upper-bounded by $\max_{c \in \{\bar{c}_{id},\bar{\underline{c}}_{id}\}} u_{kd} \cdot \bar{o}_{id} \cdot (2\sigma(2c)-1)$ and each $u_{kd}$ is the $d$-th element of vector $\vec{v}_k^{\top} \vec{B}$ (\cref{eqn:both}). Note that the $\bar{o}_{id}, \bar{c}_{id}, \bar{\underline{c}}_{id}$ of newly added summands $\bar{g}$ are actually from the right-to-left LSTM while those of $g$ are from the left-to-right LSTM. 

\subsection{Missing Data Factors in $p$}\label{sec:missp}

Recall that the joint model \eqref{eqn:target} includes a factor $\pmiss(\unobs \mid \comp)$, which appears in the numerator of the unnormalized importance weight \eqref{eqn:weight_b}.  Regardless of the form of this factor, it could be multiplied into the particle's weight $\tilde{w}_m$ at the \emph{end} of sampling (\cref{line:return_ensemble} of \cref{alg:sis}).

However, for some $\pmiss$ distributions, there is a better way.  \Cref{alg:sis} assumes that the missingness of each event $\kt{k}{t}$ depends only on that event and preceding events,\footnote{This assumption could trivially be relaxed to allow it to also depend on the missingness of the preceding events, and/or on the future observed events $\future(t)$.}
so that $\pmiss(\unobs \mid \comp)$ factors as
\begin{align}\label{eqn:missfactors}
 \smashoperator[lr]{\prod_{\ell \in \mathrm{indices}(\unobs)}} 
      p_{\text{miss}}( \kt{k_\ell}{t_\ell} \in Z  \mid \{\kt{k_{\ell'}}{t_{\ell'}}: \ell' \leq \ell\}) \\ 
 \cdot \smashoperator[lr]{\prod_{\ell \in \mathrm{indices}(\obs)}} 
      p_{\text{miss}}( \kt{k_\ell}{t_\ell} \not\in Z  \mid \{\kt{k_{\ell'}}{t_{\ell'}}: \ell' \leq \ell\}) \nonumber
\end{align}
\cref{alg:sis} can thus {\em incrementally} incorporate the subfactors of \cref{eqn:missfactors}, and does so at \cref{line:missfactor} of \cref{alg:sis}.  
For example, with the missingness mechanism in our experiments, \cref{eqn:miss-prob},
the $\pmiss$ factor in \cref{line:missfactor} is $\pmi_k$ if the event is unobserved (that is, $i'=i$) or $1-\pmi_k$ if it is observed.  

These subfactors are therefore taken into account as the particles are constructed, and thus play a role in resampling.

\subsection{Optional Missing Data Factors in $q$}\label{sec:missq}
We can optionally improve the particle filtering proposal intensities to incorporate the $\pmiss$ factor discussed in \cref{sec:missp} (in which case that factor will be multiplied into the denominator of \eqref{eqn:weight_b} and not just the numerator).  This makes $\proposal$ better match $\target$: it means we will rarely posit an unobserved event that would rarely have gone missing.  

Specifically, if a completed-data event $\kt{k}{t}$ would have probability $\pmi_k(t \mid \history(t))$ of going missing given the preceding events $\history(t)$, it is wise to define $\lambda^q_k(t \mid \history(t)) = \lambda^p_k(t \mid \history(t)) \cdot \pmi_k(t \mid \history(t))$.

We do include this extra $\pmi_k$ factor when defining $\lambda^q_k$ for our experiments (\cref{sec:exp}); that is, we modify the definition of $\lambda^q_k$ at \cref{line:qdef}.
The factor is particularly simple in our experiments, where $\pmi_k$ is constant for each $k$.

Was this factor really necessary in the case of particle smoothing?  One might say no: particle smoothing already tries to ensure through training that the proposal distribution will incorporate $\pmiss$.  That is because \cref{sec:train} aims to train $\lambda^q_k(t \mid \history(t),\future(t))$ so that the resulting $\proposal \approx \target$, and the posterior distribution $\target$ does condition on the missingness of $\unobs$.

Still, if the $\pmi_k$ factor is known, why not include it explicitly in the proposal distribution, instead of having to train the BiLSTM to mimic it?  Thus, in effect, we have modified the right-hand side of \cref{eqn:both} to include a factor of $\pmi_k$.  This yields a more expressive and better-factored family of proposal distributions: missingness is now handled by the known $\pmi_k$ factor and the BiLSTM does not have to explain it.  
Additionally, our proposal distribution becomes more conservative about proposing missing events, because having a lot of missing events is \emph{a posteriori} improbable.  In other words, $\pmiss$ as given in \cref{eqn:miss-prob} falls off with the number of missing events $|\unobs|$.

Modifying \cref{eqn:both} in this way is particularly useful in the special case $r_k=0$ (i.e., event type $k$ is never missing and should not be proposed).  There, it enforces the hard constraint that $\lambda^q_k = 0$ (something that the BiLSTM by itself could not achieve); and since this constraint is enforced regardless of the BiLSTM parameters, the events of type $k$ appropriately become irrelevant to the training of the BiLSTM, which can focus on predicting other event types.

\subsection{Events with Equal Times}

In contrast to the notation in the main paper, our pseudocode is written in terms of sequence of events, rather than sets of events.  As a result, it can handle the generalization noted in \cref{fn:setseq}, where a 0 delay is allowed between an event and the preceding event in the complete sequence.  If this occurs, it means that multiple events have fallen at the same time---yet they still have a well-defined order in which they are generated and read by the LSTM.

An unobserved event may have a 0 delay, if \cref{line:expdraw} proposes $\Delta=0$ and the proposal is accepted.  The neural Hawkes model can make in principle make such a proposal, but it has zero probability.  However, it might have positive probability under a slightly different model.

An observed event may also have a 0 delay, if $t=t_{i+1}$ at \cref{line:keeptie} and the proposal is accepted.\footnote{It may seem improbable to propose $t=t_{i+1}$ exactly, but if $t_i=t_{i+1}$, then proposing an unobserved event between these two observed events is just a case of proposing with 0 delay, as in the previous paragraph.}  In this way, it is possible for the proposal distribution to propose any number of unobserved events at time $t_{i+1}$ and immediately before the actual observed event $\kt{k_{i+1}}{t_{i+1}}$.  However, once the proposal distribution happens to propose $\Delta > 0$, the actual observed event $\kt{k_{i+1}}{t_{i+1}}$ will preempt the proposal, ending this sequence of $J_i$ unobserved events.

\section{Right-to-Left Continuous-Time LSTM}\label{sec:r2l}

Here we give details of the right-to-left LSTM from \cref{sec:nhps}. Note that this set of formulas is nearly the same as that of \cref{sec:nhp_details}---after all, it is a continuous-time LSTM that has the same architecture but different parameters from the one in the neural Hawkes process. The difference is that it reads only the observed events, and does so from right to left.

At each time $t \in (0, T)$, its hidden state $\stateb(t)$ is continually obtained from the memory cells $\vec{c}(t)$ as the cells decay:
\begin{align}
	\label{eqn:hawkes_mod_b}
  \stateb(t) = \vec{o}_{i} \odot (2\sigma(2\cellb(t))-1) \text{ for } t \in (t_{i-1}, t_{i}]
\end{align}
where the interval $(t_{i-1},t_{i})$ has consecutive observations $\kt{k_{i-1}}{t_{i-1}}$ and  $\kt{k_{i}}{t_{i}}$ as endpoints.
At $t_i$, the continuous-time LSTM reads $\kt{k_{i}}{t_{i}}$ and updates the current (decayed) hidden cells $\cellb(t)$ to new initial values $\cellb_{i-1}$, based on the current (decayed) hidden state $\stateb(t_{i})$, as follows:
\minipage[t]{0.49\textwidth}
\begin{subequations} \label{eqn:ct_lstm}
\begin{align}
	\inputgateb_{i-1}
	&\leftarrow \sigma \left( \vec{W}_{\mathrm{i}} \vec{k}_{i} + \vec{U}_{\mathrm{i}} \stateb(t_{i}) + \vec{d}_{\mathrm{i}}  \right) \\% \in (0,1)^D \\
	\forgetgateb_{i-1}
	&\leftarrow \sigma \left( \vec{W}_{\mathrm{f}} \vec{k}_{i} + \vec{U}_{\mathrm{f}} \stateb(t_{i}) + \vec{d}_{\mathrm{f}}  \right) \\% \in (0,1)^D \\
	\zgateb_{i-1}
	&\leftarrow 2 \sigma \left( \vec{W}_{\mathrm{z}} \vec{k}_{i} + \vec{U}_{\mathrm{z}} \stateb(t_{i}) + \vec{d}_{\mathrm{z}}  \right) - 1 \\% \in (-1,1)^D \\
	\outputgateb_{i-1}
	&\leftarrow \sigma \left( \vec{W}_{\mathrm{o}} \vec{k}_{i} + \vec{U}_{\mathrm{o}} \stateb(t_{i}) + \vec{d}_{\mathrm{o}}  \right) \label{eqn:ot}
\end{align}
\end{subequations}
\endminipage
\hfill
\minipage[t]{0.49\textwidth}
\begin{subequations} \label{eqn:ct_cell}
\begin{align}
\cellb_{i-1}
&\leftarrow \forgetgateb_{i-1} \odot \cellb(t_{i}) + \inputgateb_{i-1} \odot \zgateb_{i-1}\label{eqn:ct_cell_instant} \\% \in \Real^D
\statcellb_{i-1}
&\leftarrow \statforgetgateb_{i-1} \odot \statcellb_{i} + \statinputgateb_{i-1} \odot \zgateb_{i-1} \label{eqn:ct_cell_target} \\% \in \Real^D \\
\decaygateb_{i-1}
&\leftarrow f \left( \vec{W}_{\mathrm{d}} \vec{k}_{i} + \vec{U}_{\mathrm{d}} \stateb(t_{i}) + \vec{d}_{\mathrm{d}} \right) \label{eqn:ct_cell_rate}
\end{align}
\end{subequations}
\endminipage

The vector $\vec{k}_{i} \in \{ 0,1 \}^K$ is the $i$\th input: a one-hot encoding of the new event $k_{i}$, with non-zero value only at the entry indexed by $k_{i}$.
Then, $\cellb(t)$ { for } $t \in (t_{i-1},t_{i}]$ is given by \eqref{eqn:c_decay}, which continues to control $\stateb(t)$ except that $i$ has now decreased by 1).
\begin{align}
\cellb(t) &\defeq \statcellb_{i-1} + \left( \cellb_{i-1} - \statcellb_{i-1} \right) \exp \left( -\decaygateb_{i-1} \left( t_i - t \right) \right)
\label{eqn:c_decay}
\end{align}

On the interval $[t_{i-1},t_{i})$, $\cellb(t)$ follows an exponential curve that begins at $\cellb_{i-1}$ (in the sense that $\lim_{t \rightarrow t_{i}^-} \cellb(t) = \cellb_{i-1}$) and decays, as time $t$ decreases, toward $\statcellb_{i-1}$.

\section{Optimal Transport Distance Details}\label{sec:dpdetails}

\begin{algorithm*}[t]
	\caption{A Dynamic Programming Algorithm to Find Optimal Transport Distance}\label{alg:dp}
\begin{algorithmic}[1]
    \INPUT {proposal $\estimate$; reference $\truth$}
    \OUTPUT optimal transport distance $d_{\text{}}$; alignment $\ali$
    \Procedure{OTD}{$\estimate, \truth$}
      \State $d \gets 0$; $\ali \gets \text{empty collection}\ \{\}$
      \For{$k \gets 1$ {\bfseries to} $K$}
        \State $d\supk, \ali\supk \gets \textsc{Align}(\estimate\supk, \truth\supk)$
        \State $d \gets d + d\supk$; $\ali \gets \ali \cup \ali\supk$
      \EndFor
    \State \textbf{return} $d, \ali$
    \EndProcedure
    \Procedure{Align}{$\estimate\supk, \truth\supk$}
    \State $\hat{I} \gets |\estimate\supk|$; $I^{*} \gets |\truth\supk|$ \Comment{$\estimate\supk = \hat{t}_1, \ldots, \hat{t}_{\hat{I}}$ and $\truth\supk = {t}^{*}_1, \ldots, {t}^{*}_{{I}^{*}}$}
	  \State $\vec{D} \gets \text{zero matrix with } (\hat{I}+1) \text{ rows and } (I^{*}+1) \text{ columns}$
	  \State $\vec{P} \gets \text{empty matrix with } \hat{I} \text{ rows and } I^{*} \text{ columns}$ \Comment{back pointers}
    \For{$\hat{i} \gets 1$ {\bfseries to} $\hat{I}$} \Comment{transport reference of length $0$ to proposal of length $\hat{i}$}
      \State $\vec{D}_{\hat{i},0} \gets \vec{D}_{\hat{i}-1, 0} + C_{\text{delete}}$ \Comment{delete $\hat{t}_{\hat{i}}$ (and prefixes are matched)}
    \EndFor
    \For{$i^* \gets 1$ {\bfseries to} $I^*$} \Comment{transport preference of length $i^*$ to proposal of length $0$}
      \State $\vec{D}_{0,i^*} \gets \vec{D}_{0, i^*} + C_{\text{insert}}$ \Comment{insert $\hat{t}_{{i}^*} = t^*_{{i^*}}$ to decode (and their prefixes are matched)}
    \EndFor
	  \For{$\hat{i} \gets 1$ {\bfseries to} $\hat{I}$} \Comment{proposal prefix of length $\hat{i}$}
	    \For{$i^{*} \gets 1$ {\bfseries to} $I^{*}$} \Comment{to match reference of length $i^*$}
        \State $D_{\text{delete}} \gets \vec{D}_{\hat{i}-1, i^*}+C_{\text{delete}}$ \Comment{if the event token at $\hat{t}_{\hat{i}}$  is deleted from $\estimate\supk$ }
        \State $D_{\text{insert}} \gets \vec{D}_{\hat{i}, i^*-1}+C_{\text{insert}}$ \Comment{if an event token at $t^{*}_{i^*}$ is inserted to $\estimate\supk$}
        \State $D_{\text{move}} \gets \vec{D}_{\hat{i}-1, i^*-1}+|\hat{t}_{\hat{i}} - t^{*}_{i^*}|$ \Comment{if the event at $t^{*}_{i^*}$ of $\truth\supk$ is aligned to event at $\hat{t}_{\hat{i}}$ of $\estimate\supk$}
        \State $\vec{D}_{\hat{i}, i^*} \gets \min\{ D_{\text{insert}}, D_{\text{delete}}, D_{\text{move}} \}$ \Comment{choose the edit that yields the shortest distance}
        \State $\vec{P}_{\hat{i}, i^*} \gets \argmin_{e\in\{\text{insert}, \text{delete}, \text{move}\}} D_{e}$ \Comment{$e$ represents the kind of edition}
      \EndFor
	  \EndFor
    \State $\hat{i} \gets \hat{I}; i^* \gets I^*; \ali \gets \text{empty collection} \{\}$
      \While{$\hat{i} > 0 \textbf{ and } i^* > 0$} \Comment{back trace}
    \If{$\vec{P}_{\hat{i}, i^*} = \text{delete}$} \Comment{token $\hat{t}_{\hat{i}}$ is deleted.}
        \State $\hat{i} \gets \hat{i} - 1$
    \EndIf
    \If{$\vec{P}_{\hat{i}, i^*} = \text{insert}$} \Comment{a token at $t^*_{i^*}$ is inserted}
        \State $i^* \gets i^* - 1$
    \EndIf
    \If{$\vec{P}_{\hat{i}, i^*} = \text{move}$} \Comment{token $t^*_{i^*}$ is aligned to $\hat{t}_{\hat{i}}$}
        \State $\hat{i} \gets \hat{i} - 1; i^* \gets i^* - 1$
        \State $\ali \gets \ali \cup \{(\hat{t}_{\hat{i}}, t^*_{i^*})\}$
    \EndIf
    \EndWhile
  \State \textbf{return} $\vec{D}_{\hat{I},I^*}, \ali$
  \EndProcedure
\end{algorithmic}
\end{algorithm*}

Pseudocode is presented in \cref{alg:dp} for finding optimal transport distance and the corresponding alignment.
In the remainder of this section, we prove that optimal transport distance is a valid metric.

It is trivial that OTD is non-negative, since movement, deletion and insertion costs are all positive.

It is also trivial to prove that the following statement is true:
\begin{align}
    \loss(\unobs_1, \unobs_2) = 0 \Leftrightarrow \unobs_1 = \unobs_2,
    \label{eqn:zero}
\end{align}
where $\unobs_1$ and $\unobs_2$ are two sequences. If $\unobs_1$ is not identical to $\unobs_2$, the distance of them must be larger than 0 since we have to do some movement, insertion or deletion to make them exactly matched, so the right direction of \cref{eqn:zero} holds. If the distance between $\unobs_1$ and $\unobs_2$ is zero, which means they are already matched without any operations, $\unobs_1$ and $\unobs_2$ must be identical, thus the left direction of \cref{eqn:zero} holds.

OTD is symmetric, that is, $\loss(\unobs_1, \unobs_2) = \loss(\unobs_2, \unobs_1)$, if we set $C_\text{insert} = C_\text{delete}$.
Suppose that $\ali$ is an alignment between $\unobs_1$ and $\unobs_2$.
It's easy to see that the only difference between $D(\unobs_1, \unobs_2, \ali)$ and $D(\unobs_2, \unobs_1, \ali)$
\footnote{We abuse the notation $\ali$, which we think could represent both the movement from $\unobs_1$ to $\unobs_2$ and from $\unobs_2$ to $\unobs_1$.}
is that the insertion and deletion operations are exchanged.
For example, if we delete a token $t_i \in \unobs_1$ when calculating $D(\unobs_1, \unobs_2, \ali)$, we should insert a token at $t_i$ to $\unobs_2$ when calculating $D(\unobs_2, \unobs_1, \ali)$.
If we set $C_\text{insert} = C_\text{delete}$, we have
\begin{align}
    D(\unobs_1, \unobs_2, \ali) =
    D(\unobs_2, \unobs_1, \ali), \quad \forall \ali \in \mathcal{A}(\unobs_1, \unobs_2).
\end{align}
Therefore, we could obtain
\begin{subequations}
\begin{align*}
\loss(\unobs_1, \unobs_2)
&= \min_{\ali^*\in\mathcal{A}(\unobs_1, \unobs_2)}D(\unobs_1, \unobs_2, \ali^*)\\
&=\min_{\ali^*\in\mathcal{A}(\unobs_1, \unobs_2)} D(\unobs_2, \unobs_1, \ali^*)
=\loss(\unobs_2, \unobs_1)
\end{align*}
\end{subequations}

Finally let's prove that OTD satisfies triangle inequality, that is:
\begin{align}
    \loss(\unobs_1, \unobs_2) + \loss(\unobs_2, \unobs_3) \geq \loss(\unobs_1, \unobs_3), \label{eqn:tri}
\end{align}
where $\unobs_1$, $\unobs_2$ and $\unobs_3$ are three sequences.
This property could be proved intuitively.
Suppose that the operations on $\unobs_1$ with minimal costs to make $\unobs_1$ matched to $\unobs_2$ are denoted by $o_1, o_2, \dots, o_{n_1}$,
and those on $\unobs_2$ to make $\unobs_2$ matched to $\unobs_3$ are denoted by $o_1', o_2', \dots, o_{n_2}'$. $o_i$ could be a deletion, insertion or movement on a token.
To make $\unobs_1$ matched to $\unobs_3$, one possible way, which is not necessarily the optimal, is to do $o_1, o_2, \dots, o_{n_1}, o_1', o_2', \dots, o_{n_2}'$ on $\unobs_1$.
Since the total cost is the accumulation of the cost of each operation, and the operations on $\unobs_1$ above to make $\unobs_1$ matched to $\unobs_3$ might not be optimal, the triangle inequality \cref{eqn:tri} holds.

\section{Approximate MBR Details}\label{sec:mbr_details}\label{sec:consensus_details}

\begin{algorithm*}
	\caption{Approximate Consensus Decoding}\label{alg:mbr}
	\begin{algorithmic}[1]
	\INPUT collection of weighted particles $\seqspace_M = \{(\unobs_m, w_m)\}_{m=1}^{M}$
	\OUTPUT consensus sequence $\estimate$ with low $\sum_{m=1}^{M} w_m \loss(\estimate, \unobs_m )$
	\Procedure{ApproxMBR}{$\seqspace_M$}
  \State $\estimate \gets \text{empty sequence}$
  \For{$k = 1$ {\bfseries to} $K$}
    \State $\estimate\supk \gets \textsc{DecodeK}(\{ ( \unobs\supk_{m}, w_m ) \}_{m=1}^{M})$; $\estimate \gets \estimate \sqcup \estimate\supk$ \Comment{decode for type-$k$ by calling \textsc{DecodeK}}
  \EndFor
  \State \textbf{return} $\estimate$
  \EndProcedure
  \Procedure{DecodeK}{$\seqspace_M$}
  \LineComment{$\seqspace_M$ actually means $\seqspace\supk_M = \{(\unobs\supk_m, w_m)\}_{m=1}^{M}$ throughout the procedure; $\unobs_m$ is constant}
  \State $\unobs \gets \argmax_{\unobs \in \{\unobs_m\}_{m=1}^{M}} w_{m}$ \Comment{init decode as highest weighted particle and it is global}
  \Repeat
  	\For{$m = 1$ {\bfseries to} $M$} \Comment{\bf Align Phase}
  		\State $d_m, \ali_m \gets \textsc{Align}(\unobs, \unobs_m )$ \Comment{call method in \cref{alg:dp}; $d_m, \ali_m$ are global}
  	\EndFor
  	\State $r_\text{min} \gets \sum_mw_md_m$ \Comment{track the risk of current $\unobs$}
	\State $\unobs, \{d_m, \ali_m\}_{m=1}^{M} \gets \textsc{Move}(\unobs, \{\unobs_m, d_m, \ali_m\}_{m=1}^{M})$ \Comment{see \cref{alg:mbr_subs}}
	\State $\unobs, \{d_m, \ali_m\}_{m=1}^{M} \gets \textsc{Delete}(\unobs, \{\unobs_m, d_m, \ali_m\}_{m=1}^{M})$ \Comment{see \cref{alg:mbr_subs}}
	\State $\unobs, \{d_m, \ali_m\}_{m=1}^{M} \gets \textsc{Insert}(\unobs, \{\unobs_m, d_m, \ali_m\}_{m=1}^{M})$ \Comment{see \cref{alg:mbr_subs}}
  \Until{$\sum_{m=1}^Mw_md_m = r_\text{min}$}
  \Comment{risk stops decreasing}
  \State \textbf{return} $\unobs$
  \EndProcedure
\end{algorithmic}
\end{algorithm*}

\begin{algorithm*}
	\caption{Subroutines for Approximate Consensus Decoding}\label{alg:mbr_subs}
	\begin{algorithmic}[1]
  \Procedure{Move}{$\unobs, \{\unobs_m, d_m, \ali_m\}_{m=1}^{M}$} \Comment{\bf Move Phase}
  		\For{${t}$ {\bfseries in} $\unobs$} 
        \For{${t'} \in \{ {t'}: ({t}', {t}) \in \bigcup_{m=1}^{M}\ali_{m} \} $}
        \Comment{may replace $t$ with ${t}'$ which is aligned to $t$}
      		\State $(\forall m) {d}_{m}' \gets d_{m}$
      		\For{$(t'', m) \in \{({t}'', m): (t'', t) \in \ali_{m}, m \in \{1,\ldots, M\} \}$}
      		    \State ${d}_{m}' \gets {d}_{m}' - |{t}'' - t| + |{t}'' - {t}'|$
      		\EndFor
      		\If{$\sum_m w_m {d}_{m}' < \sum_m w_m d_m$}
      		    \State ($\forall m$)$d_m\gets {d}_{m}'$; $t \gets t'$ \Comment{$t$ move to $t'$ for lower risk}
      		\EndIf
  		\EndFor
  		\EndFor
  	\State \textbf{return} $\unobs, \{d_m, \ali_m\}_{m=1}^{M}$
  	\EndProcedure
  	\Procedure{Delete}{$\unobs, \{\unobs_m, d_m, \ali_m\}_{m=1}^{M}$} \Comment{\bf Delete Phase}
  		 \For{$t$ {\bfseries in} $\unobs$} \Comment{may delete this event}
  		\For{$m=1$ {\bfseries to} $M$}
  		\Comment{update each $d_m$}
          	\If{$\exists t' \in \unobs_m$ {\bfseries and} $(t', t) \in \ali_m$} \Comment{find the only, if any, $t' \in \unobs_m$ that is aligned to $t$}
                    \LineComment{if we delete $t$ and its alignment $(t',t)$, $d_m$ decreases by the alignment cost (because we do not need to align it)}
                    \LineComment{but increases by an insertion cost (because we need to insert an event at $t$ to match $\unobs_m$)}
          	    \State ${d}_m' \gets d_m + C_\text{insert} - |t' - t|$ 
          	\Else \Comment{otherwise, this event has been deleted when matching with $\unobs_m$ }
          	    \State ${d}_m' \gets d_m - C_\text{delete}$ \Comment{we do not need to pay deletion cost when matching with $\unobs_m$ if we do not have this event at $t$ in $\unobs$}
          	\EndIf
          \If{$\sum_{m}w_m {d}_m' < \sum_m w_m d_m$}
  			\State delete $t$ from $\unobs$; $(\forall m)$ delete $(t',t)$ from $\ali_m$; $d_m \gets {d}_m'$
  		    \EndIf
  		\EndFor
  	\EndFor
  \State \textbf{return} $\unobs, \{d_m, \ali_m\}_{m=1}^{M}$
  \EndProcedure
  \Procedure{Insert}{$\unobs, \{\unobs_m, d_m, \ali_m\}_{m=1}^{M}$} \Comment{\bf Insert Phase}
  \Repeat
  \State $t \gets \text{None}$, $\Delta \gets -\infty$
  \For{$t_c \in \bigsqcup_m \unobs_m$ such that $t_c \notin \unobs$} \Comment{may insert $t_c$ if it is not in $\unobs$ yet}
  \For{$m=1$ {\bfseries to} $M$}
  \State $\unobs_m' \gets \{t' : \forall t'', (t'',t') \notin \ali_m \text{ and } t' \in \unobs_m\}$ \Comment{find $t'$ in $\unobs_m$ that is not aligned yet}
  \If{$\unobs'$ is not empty {\bfseries and} $\min_{t'\in \unobs_m'}|t' - t_c | < C_{\text{insert}} + C_{\text{delete}}$} \Comment{if there is any that is close enough to $t_c$}
  \State $d_m' \gets d_m - C_{\text{insert}} +  \min_{t'\in \unobs_m'}|t' - t_c |$; $\ali_m' \gets \ali_m \cup \{ (t_c, t') \}$ \Comment{align the closest one to $t_c$}
  \Else
  \State $d_m' \gets d_m + C_{\text{delete}}$; $\ali_m' \gets \ali_m$
  \EndIf
  \EndFor
  \If{$\sum_m w_m d_m - \sum_m w_m d_m' > \Delta$}
  \State $t \gets t_c$; $\Delta \gets \sum_m w_m d_m - \sum_m w_m d_m'$
  \EndIf
  \EndFor
  \If{$\Delta > 0$}
  \State $\unobs \gets \unobs \sqcup \{t\}$; $(\forall m)\ali_m \gets \ali_m'$; $d_m \gets d_m'$
  \EndIf
  \Until{$\Delta \leq 0$}
  \State \textbf{return} $\unobs, \{d_m, \ali_m\}_{m=1}^{M}$
  \EndProcedure
\end{algorithmic}
\end{algorithm*}

Our approximate consensus decoding algorithm is given as \cref{alg:mbr}.  In the remainder of this section, we prove \cref{thm:comb} from \cref{sec:mbr}, namely:

\setcounter{theorem}{0}
\begin{theorem}
	Given $\{\unobs_m\}_{m=1}^{M}$, if we define $\zunion=\bigsqcup_{m=1}^{M} \unobs_m$, then $\exists \estimate \subseteq \zunion$ such that
	\begin{align*}
	{\textstyle \sum_{m=1}^{M} w_m L( \estimate, \unobs_m ) } 
	= 
	{\textstyle \min_{\unobs \in \seqspace} \sum_{m=1}^{M} w_m L(\unobs, \unobs_m) }
	\end{align*}
	That is to say, there exists one subsequence of $\zunion$ that achieves the minimum Bayes risk.
\end{theorem}
\begin{proof}
Here we assume that there is only one type of event. Since the distances of different types of events are calculated separately, our conclusion is easy to be extended to the general case.

Suppose $\estimate$ is an optimal decode, that is,
$$\sum_{m=1}^{M} w_m \loss(\unobs_m, \estimate)=\min_{\unobs \in \seqspace}\sum_{m=1}^{M} w_m \loss(\unobs_m, \unobs).$$
If $\estimate \subseteq \zunion$, the proof is done. If not, we can choose some $t_i \notin \zunion$.  Let $t_l = \max \{t \in \zunion: t < t_i\}$ and $t_r = \min \{t \in \zunion: t > t_i\}$.  (These sets are nonempty because $\zunion$ always contains the endpoints 0 and $T$.)
We will show that if we move $t_i$ around, as long as $t_i \in [t_l, t_r]$, the weighted optimal transport distance, i.e. $\sum_{m=1}^{M} w_m \loss(\unobs_m, \estimate)$, will neither increase nor decrease.

Suppose $\alihat = \argmin_{\ali_m \in \mathcal{A}(\unobs_m, \estimate)}\sum_{m=1}^{M} w_m D(\unobs_m, \estimate, \ali_m)$. Let's use $r(t)$ to indicate the weighted transport distance of $\estimate$ with fixed alignment if we move $t_i$ to $t$, that is,
$$r(t) \defeq \sum_{m=1}^{M} w_m D(\unobs_m, \estimate(t), \alihat),$$
where $\estimate(t)$ is the sequence $\estimate$ with $t_i$ moved to $t$.
Because $\estimate(t_i)$ is an optimal decode, and $\alihat$ is the optimal alignment for $\estimate(t_i)$, we should have
$$r(t_i) = \min_{t}r(t).$$
Note that the transport distance is comprised of three parts: deletion, insertion and alignment costs. Since every $\alihat$ is fixed, if we change $t$, only the alignment cost that related to token $t$ will affect $r(t)$. This part of $r(t)$ is linear to $t$, since we have a constraint $t \in [t_l, t_r]$, which guarantees that it will not cross any other tokens in $\zunion$.

Since $r(t)$ is linear to $t\in[r_l, t_r]$ and $r(t)$ gets minimized at $t_i \in (t_l, t_r)$, we conclude that
$$r(t) = r(t_i)=\text{Const}, \forall t \in [t_l, t_r].$$

Since $r(t)$ is the upper bound of the weighted optimal transport distance $\sum_{m=1}^{M} w_m \loss(\unobs_m, \estimate(t))$, which also gets the same minimal value at $t_i \in (t_l, t_r)$ as $r(t)$, we could conclude that $\forall t\in[t_l,t_r]$:
$$\sum_{m=1}^{M} w_m \loss(\unobs_m, \estimate(t)) = \sum_{m=1}^{M} w_m \loss(\unobs_m, \estimate(t_i)) = \text{Const}$$

Therefore we could move token $t_i$ to either $t_l$ or $t_r$ without increasing the Bayes risk. We could do this movement for each $t_i \notin \zunion$ to get a new decode $\hat{\unobs} \subseteq \zunion$, which is also an optimal decode.

\end{proof}

\section{Experimental Details}\label{sec:exp_details}\label{sec:data_details}
In this section, we elaborate on the details of data generation, processing, and experimental results.

In all of our experiments, the distribution $\model$ is trained on the complete (uncensored) version of the training data.  The system is then asked to complete the incomplete (censored) version of the test (or dev) data.  For particle smoothing, the proposal distribution is trained using both the complete and incomplete versions of the training data, as explained at the end of \cref{sec:train}. We used the Adam algorithm with its default settings \citep{kingma-15}. Adam is a stochastic gradient optimization algorithm that continually adjusts the learning rate in each dimension based on adaptive estimates of low-order moments.
Each training example for Adam is a complete event stream $\comp$ over some time interval $[0,T)$.  We stop training early when we detect that log-likelihood has stopped increasing on the held-out development dataset.  We do no other regularization.

\subsection{Dataset Statistics}\label{sec:data_stats}
\Cref{tab:stats_dataset} shows statistics about each dataset that we use in this paper.
\begin{table*}[t]
\begin{center}
\begin{small}
\begin{sc}
\begin{tabularx}{1.00\textwidth}{l *{1}{S}*{3}{R}*{3}{S}}
\toprule
Dataset & \multicolumn{1}{r}{$K$} & \multicolumn{3}{c}{\# of Event Tokens} & \multicolumn{3}{c}{Sequence Length} \\
\cmidrule(lr){3-8}
  &  & Train & Dev & Test & Min & Mean & Max \\
\midrule
Synthetic & $4$ & $\approx 74967$ & $\approx 7513$ & $\approx 7507$ & $10$ & $\approx 15$ & $20$ \\
NYCTaxi & $10$ & $157916$ & $15826$ & $15808$ & $22$ & $32$ & $38$ \\
Elevator & 10 & $313043$ & $31304$ & $31206$ & $235$ & $313$ & $370$ \\
\bottomrule
\end{tabularx}
\end{sc}
\end{small}
\end{center}
 \caption{Statistics of each dataset.  We write ``$\approx N$'' to indicate that $N$ is the average value over multiple datasets of one kind (synthetic); the variance is small in each such case.}
\label{tab:stats_dataset}
\end{table*}

\subsection{Training Details}\label{sec:training}
We used single-layer LSTMs \citep{hochreiter-97-lstm}, selected the number $D$ of hidden nodes of the left-to-right LSTM, and then $D'$ of the right-to-left one from a small set $\{16, 32, 64, 128, 256, 512, 1024\}$ based on the performance on the dev set of each dataset. The best-performing $(D,D')$ pairs are $(256,128)$ on {Synthetic}, $(256, 256)$ on {Elevator}
 $(256,256)$ on {NYC Taxi}, but we empirically found that the model performance is robust to these hyperparameters.
 For the chosen $(D,D')$ pair on each dataset, we selected $\beta$ based on the performance on the dev set, and $\beta=1.0$ yields the best performance across all the datasets we use. 
For learning, we used Adam with its default settings \citep{kingma-15}.

Our Monte Carlo integral estimates are in fact unbiased (\cref{sec:integral}).
As a result, our stochastic gradient estimate is also unbiased, as required (assuming that the complete data is distributed according to $\model$). 
Why? Since $\beta=1$, our stochastic gradient is simply \cref{eqn:in_kl}. 
No particle filtering or smoothing is used to estimate \cref{eqn:in_kl}, because we train it using complete data, as explained in the last long paragraph of \cref{sec:train}. 
The only randomness is the integral over $[0,T)$ (similar to the one in \cref{eqn:logpmodel}) that is required to estimate the term $\log \proposal$ in \cref{eqn:in_kl}: as just noted, this integral estimate is unbiased.

It is true that if $\beta < 1$, we would compute the exclusive KL gradient using particle filtering or smoothing with $M$ particles, and this would introduce bias in the gradient.  Nonetheless, since the bias vanishes as $M\rightarrow \infty$, it would be possible to restore a theoretical convergence guarantee by increasing $M$ at an appropriate rate as SGD proceeds \citep[page 107]{spall-05-opt}.\footnote{SGD methods succeed, both theoretically and practically, with even high-variance estimates of the batch gradient (e.g., where each stochastic estimate is derived from a single randomly chosen training example).  Thus, one should be fine with a noisy sampling-based gradient as long as it is unbiased.}

\subsection{Details of the Synthetic Datasets}\label{sec:syn_details}
Each of the ten neural Hawkes processes has its parameters sampled from $\Uniform[-1.0, 1.0]$. Then a set of event sequences is drawn from each of them via the plain vanilla thinning algorithm \citep{mei-17-neuralhawkes}.
For each of the ten synthetic datasets, we took $K=4$ as the number of event types. To draw each event sequence, we first chose the sequence length $I$ (number of event tokens) uniformly from $\{11, 12, \ldots, 20\}$ and then used the thinning algorithm to sample the first $I$ events over the interval $[0,\infty)$. For subsequent training or testing, we treated this sequence (appropriately) as the complete set of events observed on the interval $[0,T)$ where $T=t_I$, the time of the last generated event. 

We generate $5000$, $500$ and $500$ sequences for each training, dev, and test set respectively.
For the missingness mechanism: in the deterministic settings, we censor all events of type 3 and 4---in other words, we set $\pmi_1=\pmi_2 = 0$ and $\pmi_3 = \pmi_4 = 1$; in the stochastic settings, we set $\pmi_k=0.5$ for all $k$. 

\subsection{Elevator System Dataset Details}\label{sec:elevator_details}
We examined our method in a simulated 5-floor building with 2 elevator cars.
During a typical afternoon down-peak rush hour (when passengers go from floor-2,3,4,5 down to the lobby), elevator cars travel to each floor and pick up passengers that have (stochastically) arrived there according to a traffic profile \citep{Bao-94-elevator}. Each car will also avoid floors that already are or will soon be taken care of by the other.
Having observed when and where car-1 has stopped (to pick up or drop off passengers) over this hour, we are interested in when and where car-2 has stopped during the same time period.
In this dataset, each event type is a tuple of (car number, floor number) so there are $K=10$ in total in this simulated 5-floor building with 2 elevator cars. 

Passenger arrivals at each floor are
assumed to follow a inhomogeneous Poisson process, with arrival rates that vary during the course of the day. The simulations we use follows a human-recorded traffic profile \citep{Bao-94-elevator} which dictates arrival rates for
every 5-minute interval during a typical afternoon down-peak rush hour. \cref{tab:profile} shows the mean number of passengers (who are going to the lobby) arriving at floor-2,3,4,5 during each 5-minute interval.

We simulated the elevator behavior following a naive baseline strategy documented in \citet{crites-96-elevator}.\footnote{We rebuilt the system in Python following the original Fortran code of \citet{crites-96-elevator}.}
In details, each car has a small set of primitive actions. If it is stopped at a floor, it must either ``move up'' or ``move down''. If it is in motion between floors, it must either ``stop at the next floor'' or ``continue past the next floor''. Due to passenger expectations, there are two constraints on these actions: a car cannot pass a floor if a passenger wants to get off there and cannot turn until it has serviced all the car buttons in its current direction. Three additional action constraints were made in an attempt to build in some primitive prior knowledge: 1) a car cannot stop at a floor unless someone wants to get on or off there; 2) it cannot stop to pick up passengers at a floor if another car is already stopped there; 3) given a choice between moving up and down, it should prefer moving up (since the down-peak traffic tends to push the cars toward the bottom of the building). Because of this last constraint, the only real choices left to each car are the stop and continue actions, and the baseline strategy always chooses to continue. The actions of the elevator cars are executed asynchronously since they may take different amounts of time to complete.

We repeated the (one-hour) simulation 700 times to collect the event sequences, each of which has around 300 time-stamped records of which car stops at which floor. We randomly sampled disjoint train, dev and test sets with 500, 100 and 100 sequences respectively.
\begin{table*}[t]
\begin{center}
\begin{small}
\begin{sc}
\begin{tabularx}{0.99\textwidth}{l *{12}{R}}
\toprule
Start Time (min) & $00$ & $05$ & $10$ & $15$ & $20$ & $25$ & $30$ & $35$ & $40$ & $45$ & $50$ & $55$ \\
Mean \# Passenger & $1$ & $2$ & $4$ & $4$ & $18$ & $12$ & $8$ & $7$ & $18$ & $5$ & $3$ & $2$ \\
\bottomrule
\end{tabularx}
\end{sc}
\end{small}
\end{center}
\caption{The Down-Peak Traffic Profile}
\label{tab:profile}
\end{table*}

For the missingness mechanism: in the deterministic settings, we set $\pmi_k=0$ for $k=1,\ldots, 5$ and $\pmi_k=1$ for $k=6,\ldots,10$ (meaning that the events (of arriving at floor 1, 2, \ldots, 5) of car 1 are all observed, but those of car 2 are not); in the stochastic settings, we set $\pmi_k=0.5$ for all $k$. 

\subsection{New York City Taxi Dataset Details}\label{sec:taxi_details}
The New York City Taxi dataset (\cref{sec:taxi}) includes 189,550 taxi pick-up and drop-off records in the city of New York in 2013. Each record has its medallion ID, driver license and time stamp. Each combination of medallion ID and driver license naturally forms a sequence of time-stamped pick-up and drop-off events. Following the processing recipe of previous work \citep{du-16-recurrent}, we construct shorter sequences by breaking each long sequence wherever the temporal gap between a drop-off event and its following pick-up event is larger than six hours. Then the left boundary of this gap is treated as the \eos of the sequence before it, while the right boundary is set as the \bos of the following sequence.

We randomly sampled a month from 2013 and then randomly sampled disjoint train, dev and test sets with 5000, 500 and 500 sequences respectively from that month.

In this dataset, each
event type is a tuple of (location, action).  The location
is one of the 5 boroughs $\{$Manhattan,
Brooklyn, Queens, The Bronx, Staten Island$\}$.
The action can be
either pick-up or drop-off. Thus, there are $K=5\times 2=10$ event types in total.  

For the missingness mechanism: in the deterministic settings, we set $\pmi_k=0$ for $k=1,\ldots, 5$ and $\pmi_k=1$ for $k=6,\ldots,10$ (which means that all drop-off events but no pick-up events are observed); in the stochastic settings, we set $\pmi_k=0.5$ for all $k$. 

\subsection{Experiments with Deterministic Missingness Mechanisms}\label{sec:extra-exp}

We show our experimental results for the deterministic missingness mechanisms in \cref{fig:cloud,fig:eval}.

\begin{figure}[t]
	\begin{center}
		\begin{subfigure}[b]{0.31\linewidth}
			\includegraphics[width=\linewidth]{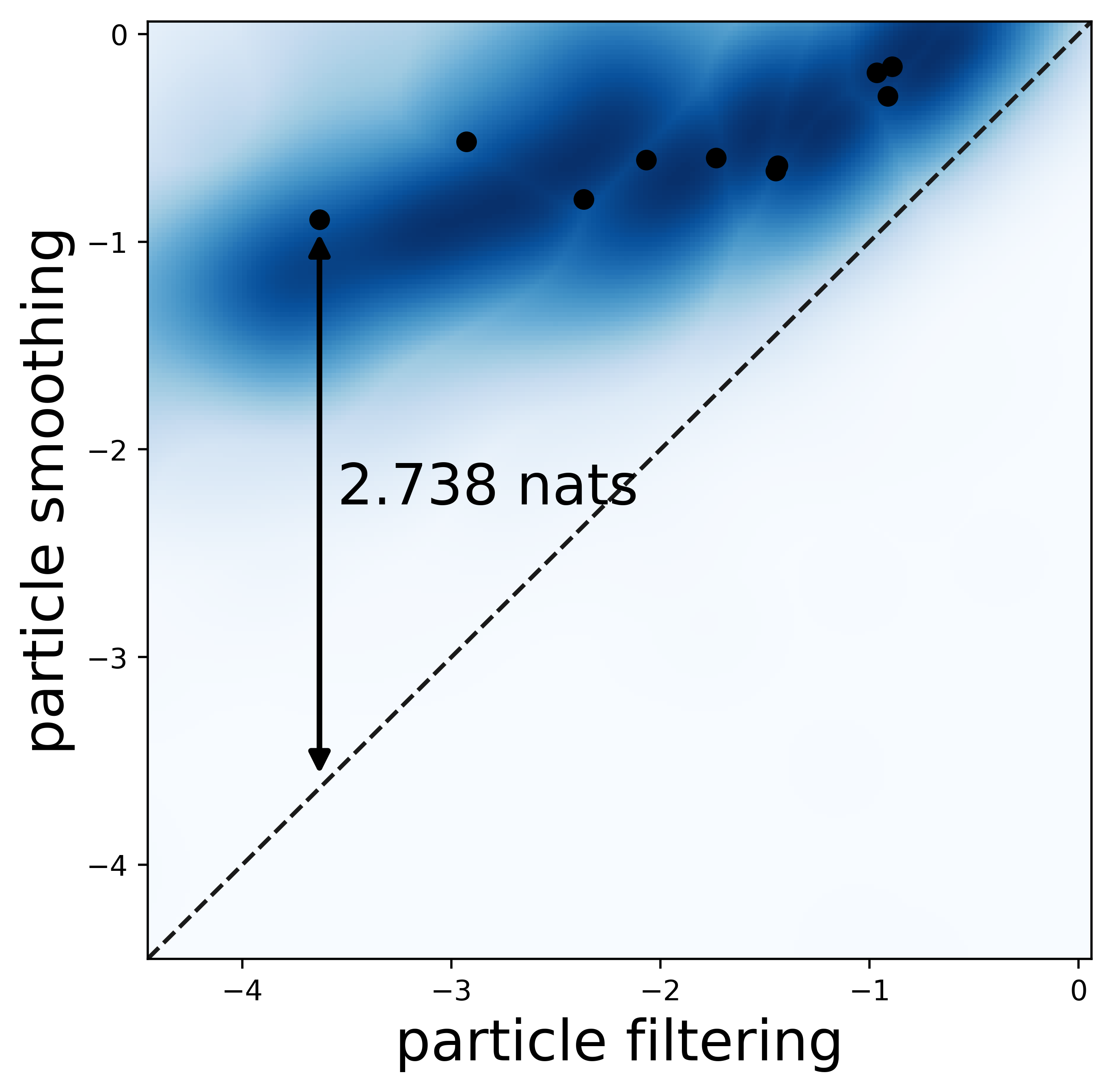}
			\caption{Synthetic}\label{fig:nhpcloud}
		\end{subfigure}
		~
		\begin{subfigure}[b]{0.31\linewidth}
			\includegraphics[width=\linewidth]{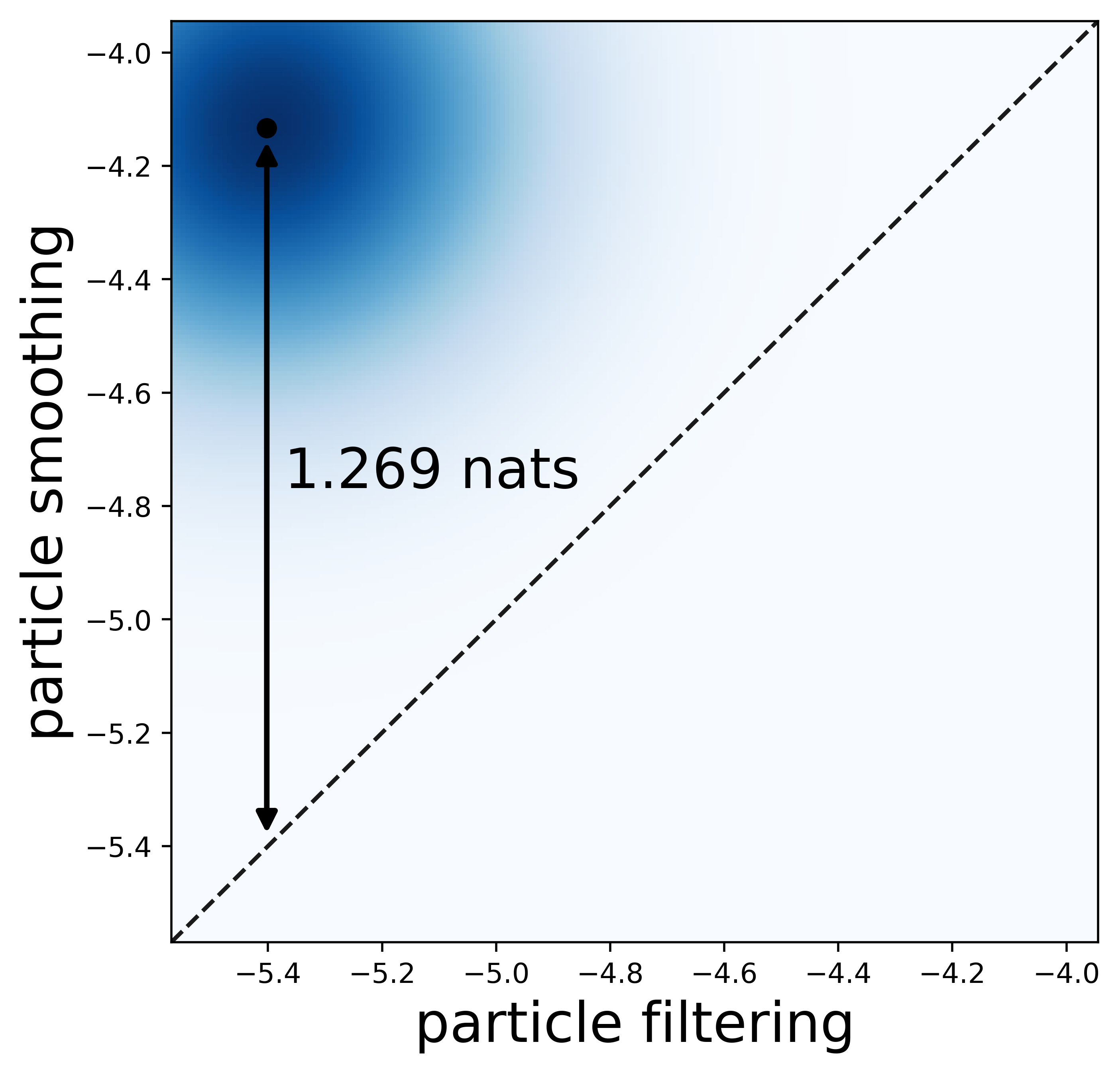}
			\caption{Elevator System}\label{fig:elevatorcloud}
		\end{subfigure}
		~
		\begin{subfigure}[b]{0.31\linewidth}
			\includegraphics[width=\linewidth]{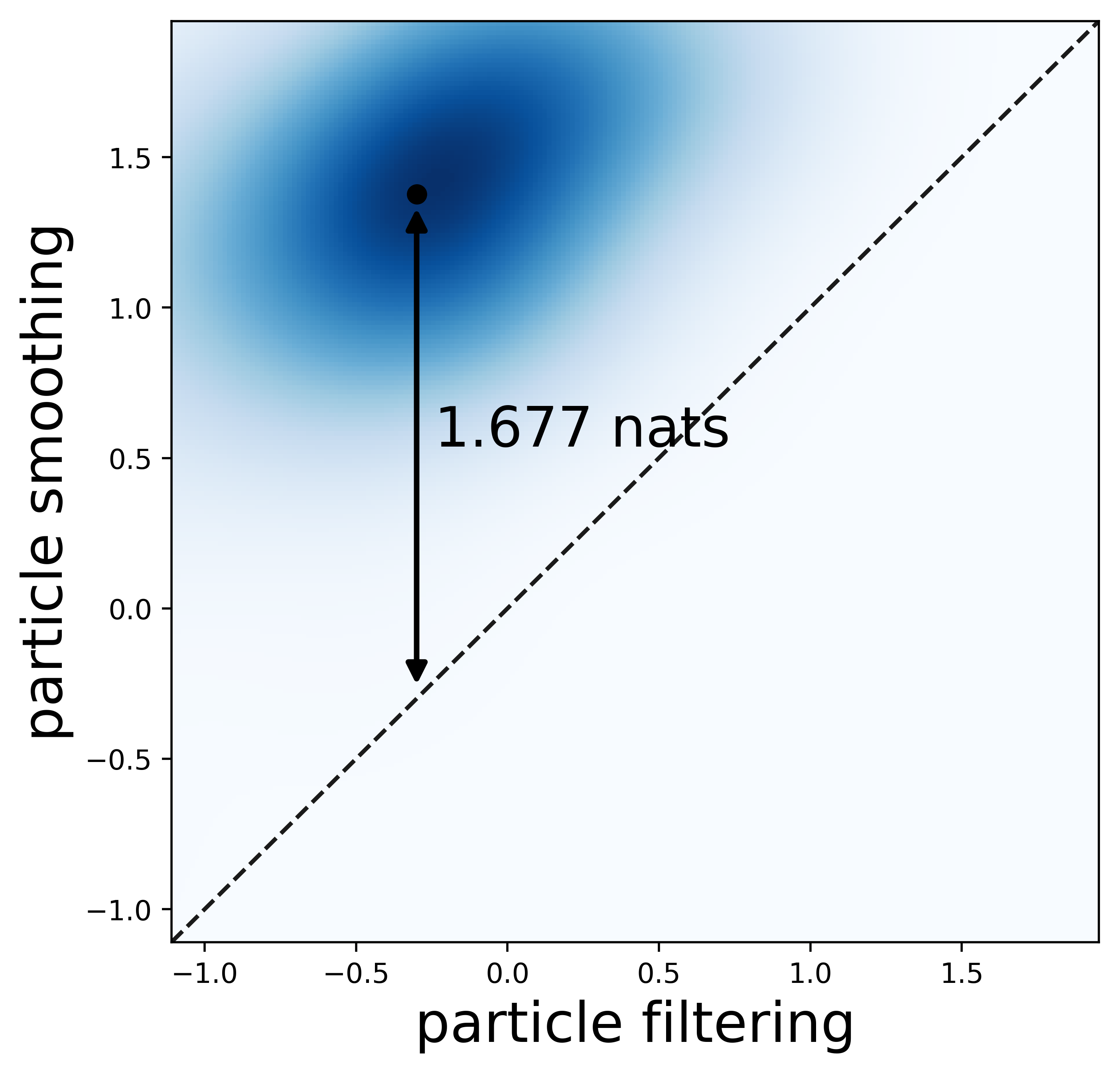}
			\caption{NYC Taxi}\label{fig:taxicloud}
		\end{subfigure}
		\caption{Scatterplots with a deterministic missingness mechanism. Again, the method works, with very similar qualitative behavior to \cref{fig:cloud-5}.
		}\label{fig:cloud}
	\end{center}
\end{figure}

\begin{figure}[t]
	\begin{center}
		\begin{subfigure}[b]{0.46\linewidth}
			\includegraphics[width=\linewidth]{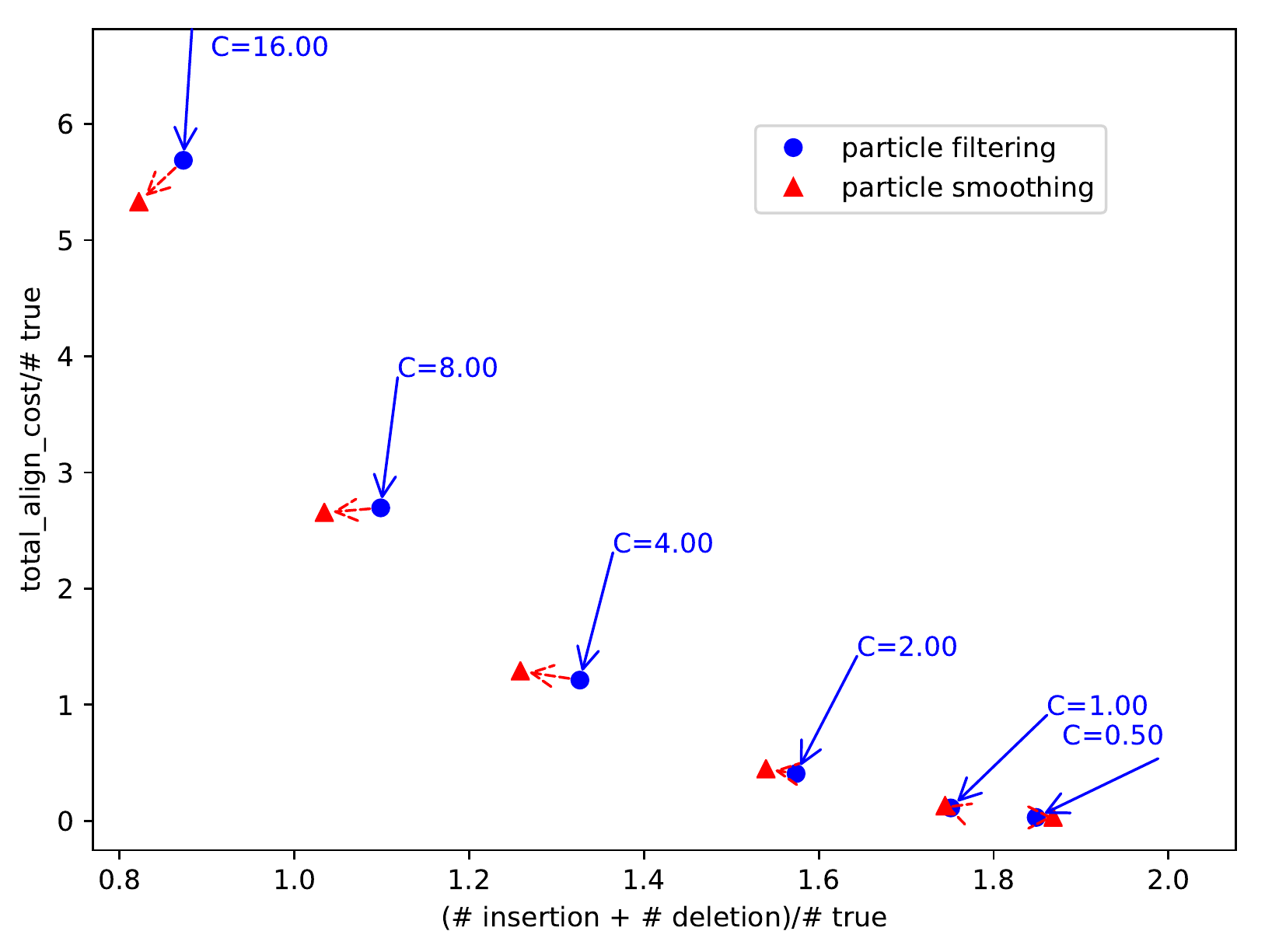}
			\caption{Elevator System}\label{fig:elevator}
		\end{subfigure}
		~
		\begin{subfigure}[b]{0.46\linewidth}
			\includegraphics[width=\linewidth]{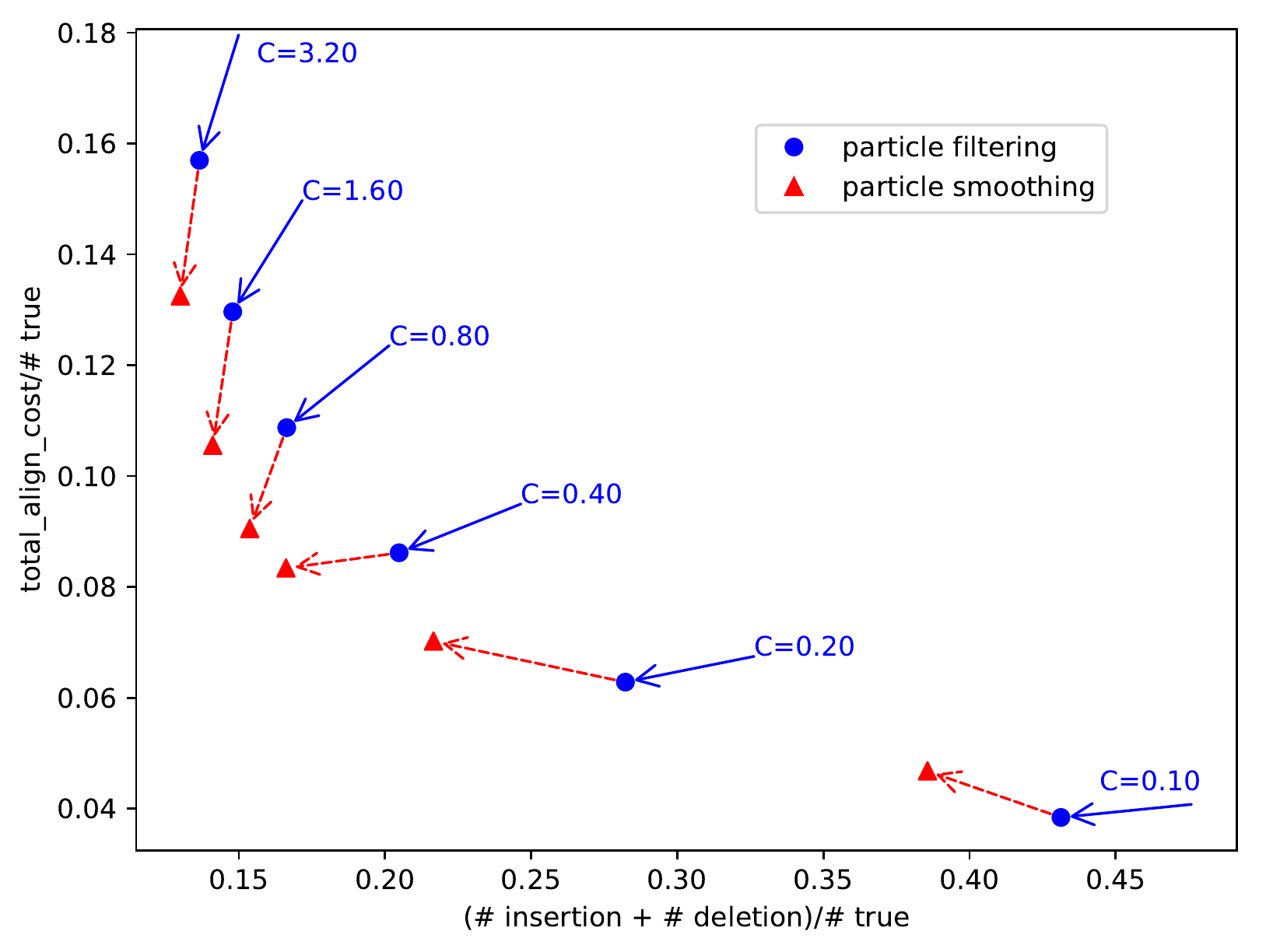}
			\caption{NYC Taxi}\label{fig:taxi}
		\end{subfigure}
		\caption{Optimal transport distance results with a deterministic missingness mechanism. Again, the method works, with very similar qualitative behavior to \cref{fig:eval-5}.
		}\label{fig:eval}
	\end{center}
\end{figure}

\subsection{Sensitivity Experiment Details}\label{sec:sensitivity_details}

\Cref{fig:diff_rhos} displays the optimal transport distance with various values of $\pmi$: our particle smoothing method consistently outperforms the filtering baseline.

\begin{figure}[!ht]
	\begin{center}
		\begin{subfigure}[b]{0.48\linewidth}
			\includegraphics[width=0.99\linewidth]{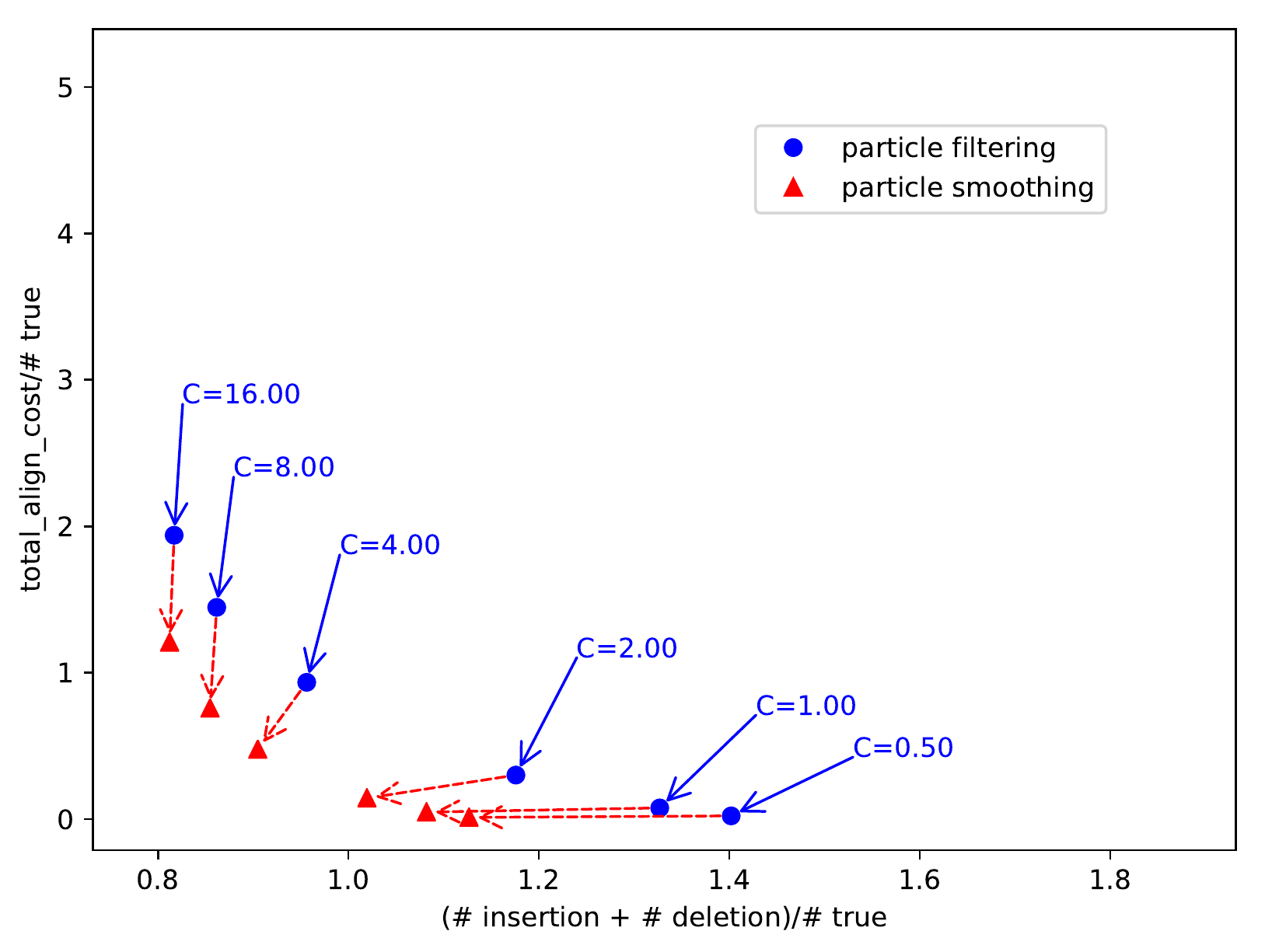}
			\includegraphics[width=0.99\linewidth]{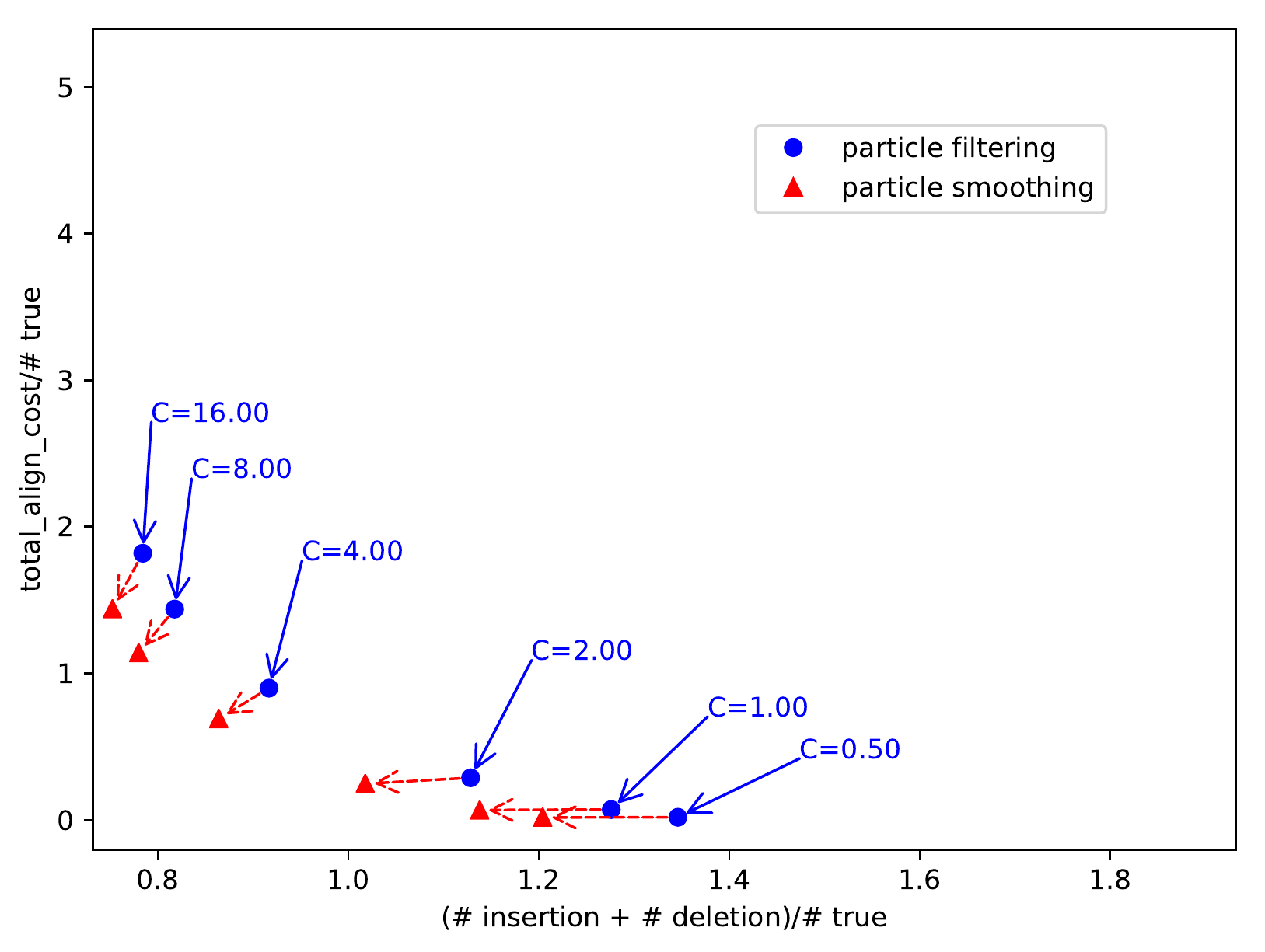}
			\includegraphics[width=0.99\linewidth]{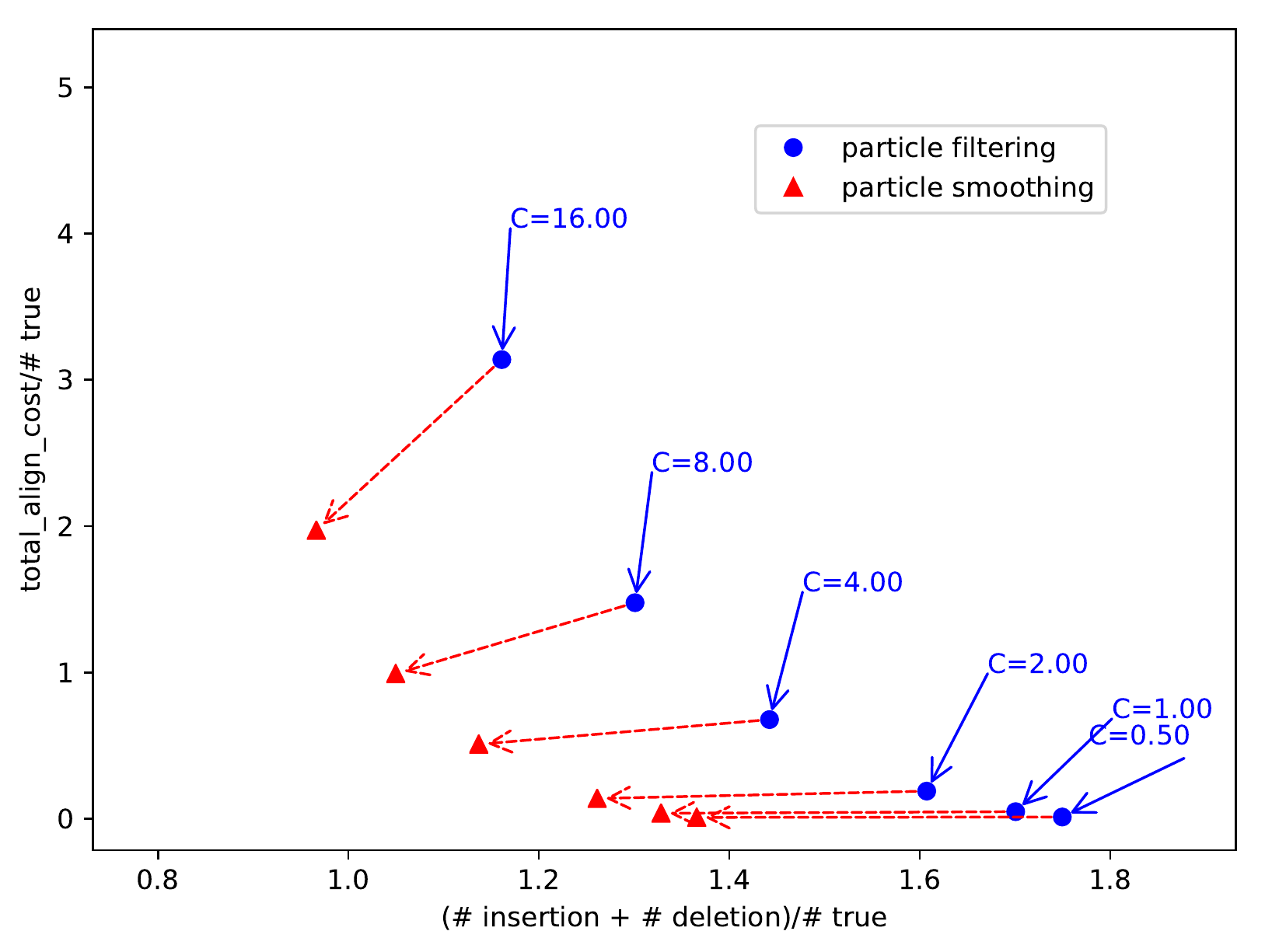}
			\includegraphics[width=0.99\linewidth]{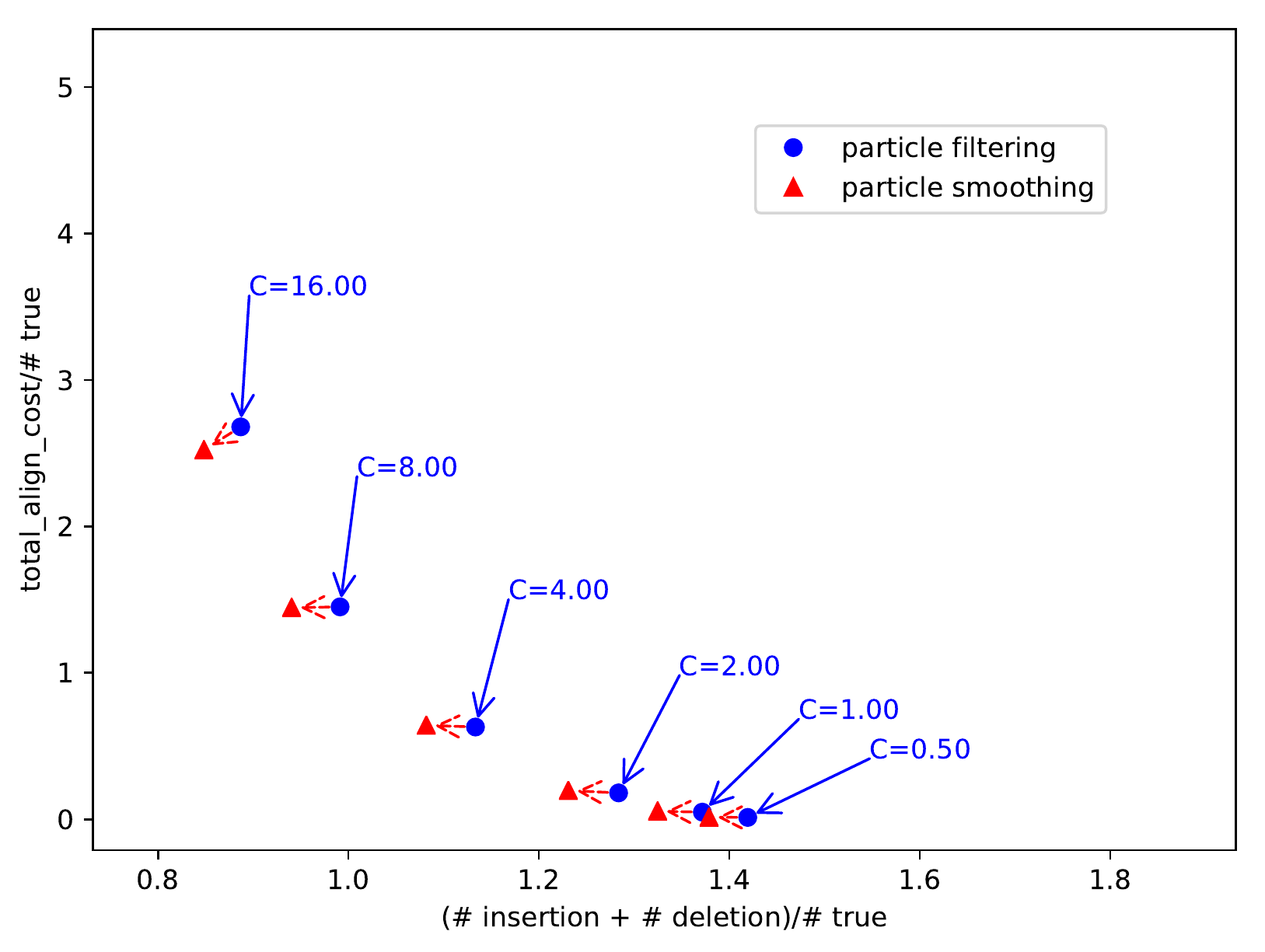}
			\includegraphics[width=0.99\linewidth]{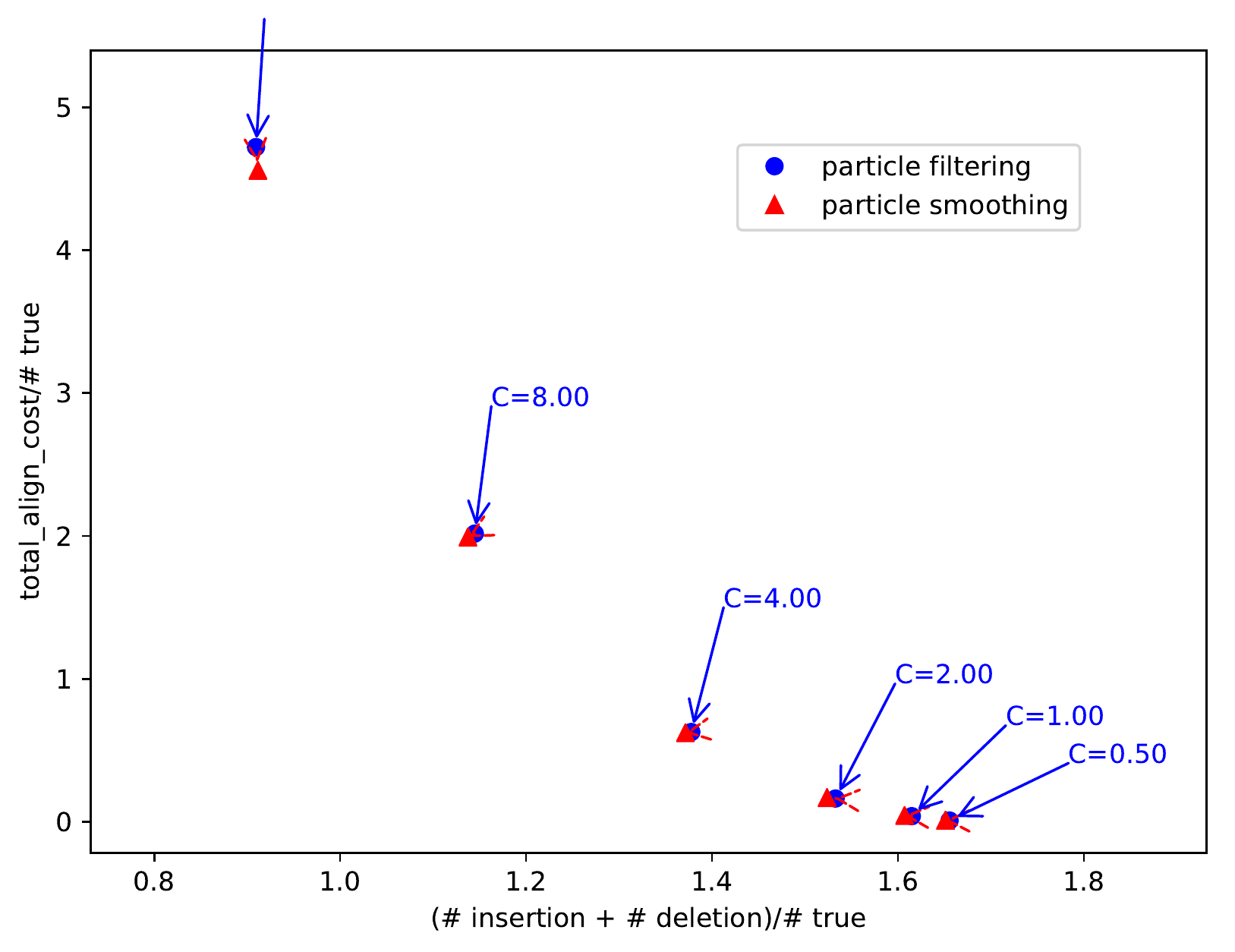}
			\caption{Elevator System}\label{fig:rhos_elevator}
		\end{subfigure}
		\begin{subfigure}[b]{0.48\linewidth}
			\includegraphics[width=0.99\linewidth]{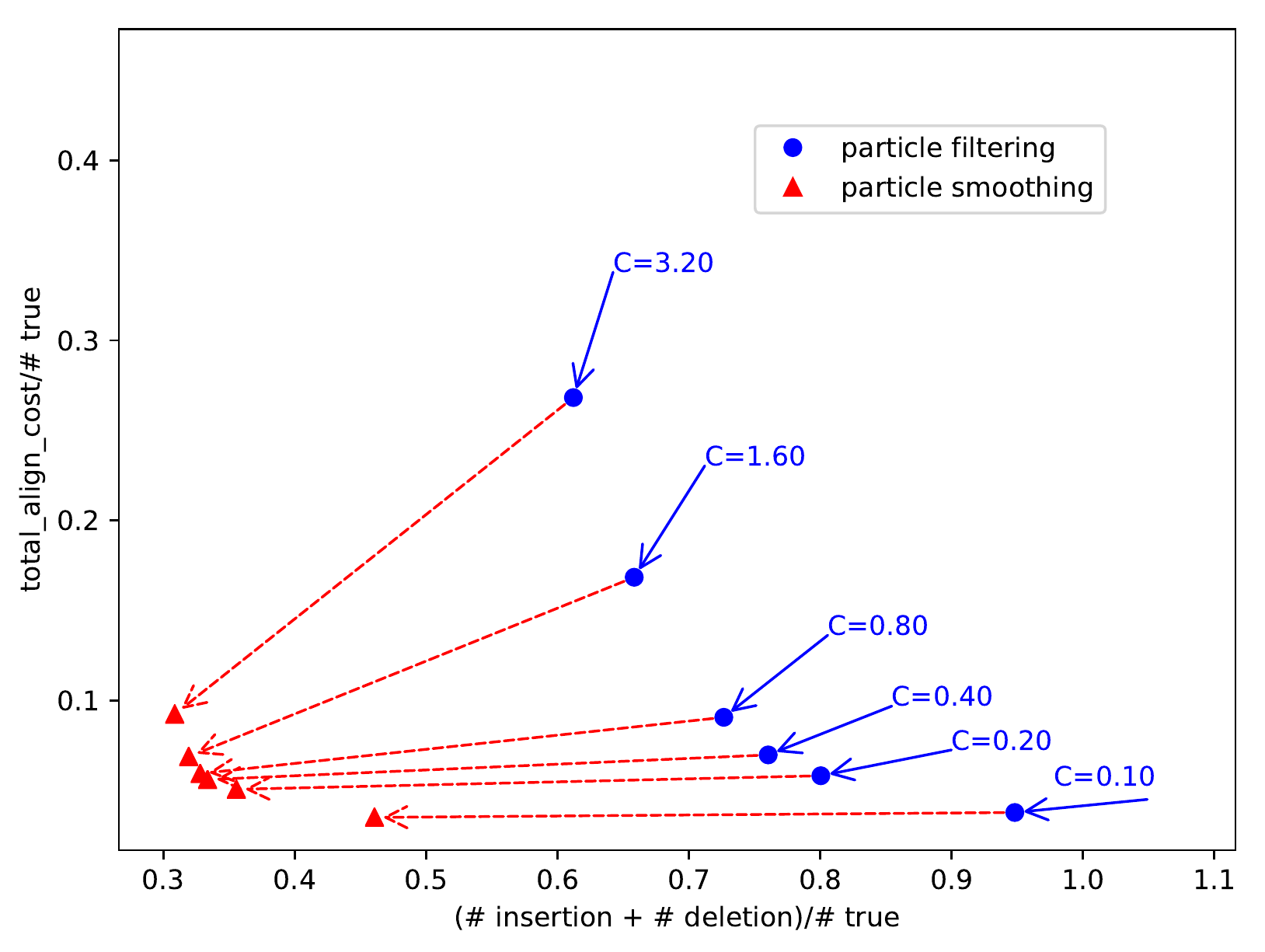}
			\includegraphics[width=0.99\linewidth]{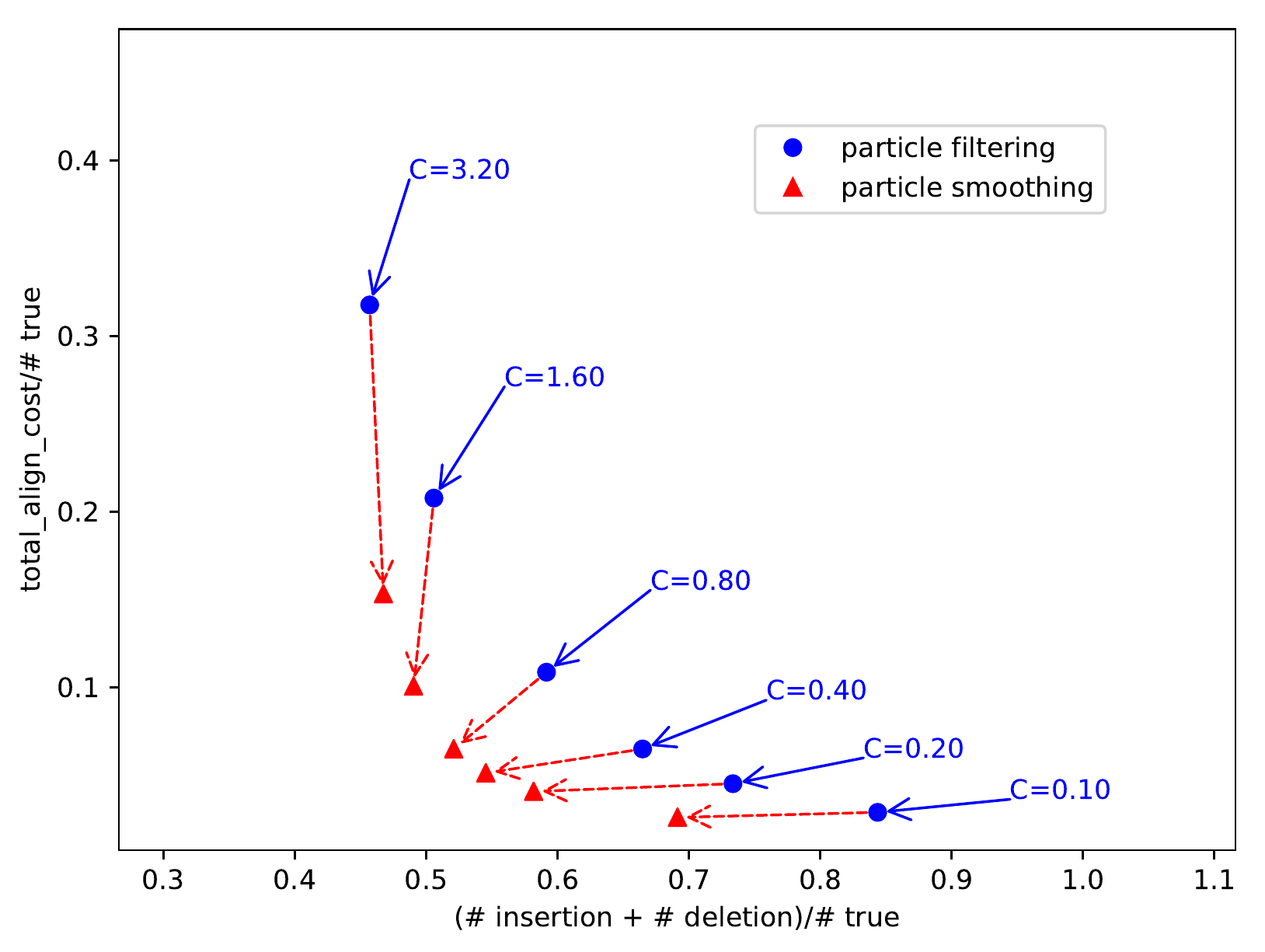}
			\includegraphics[width=0.99\linewidth]{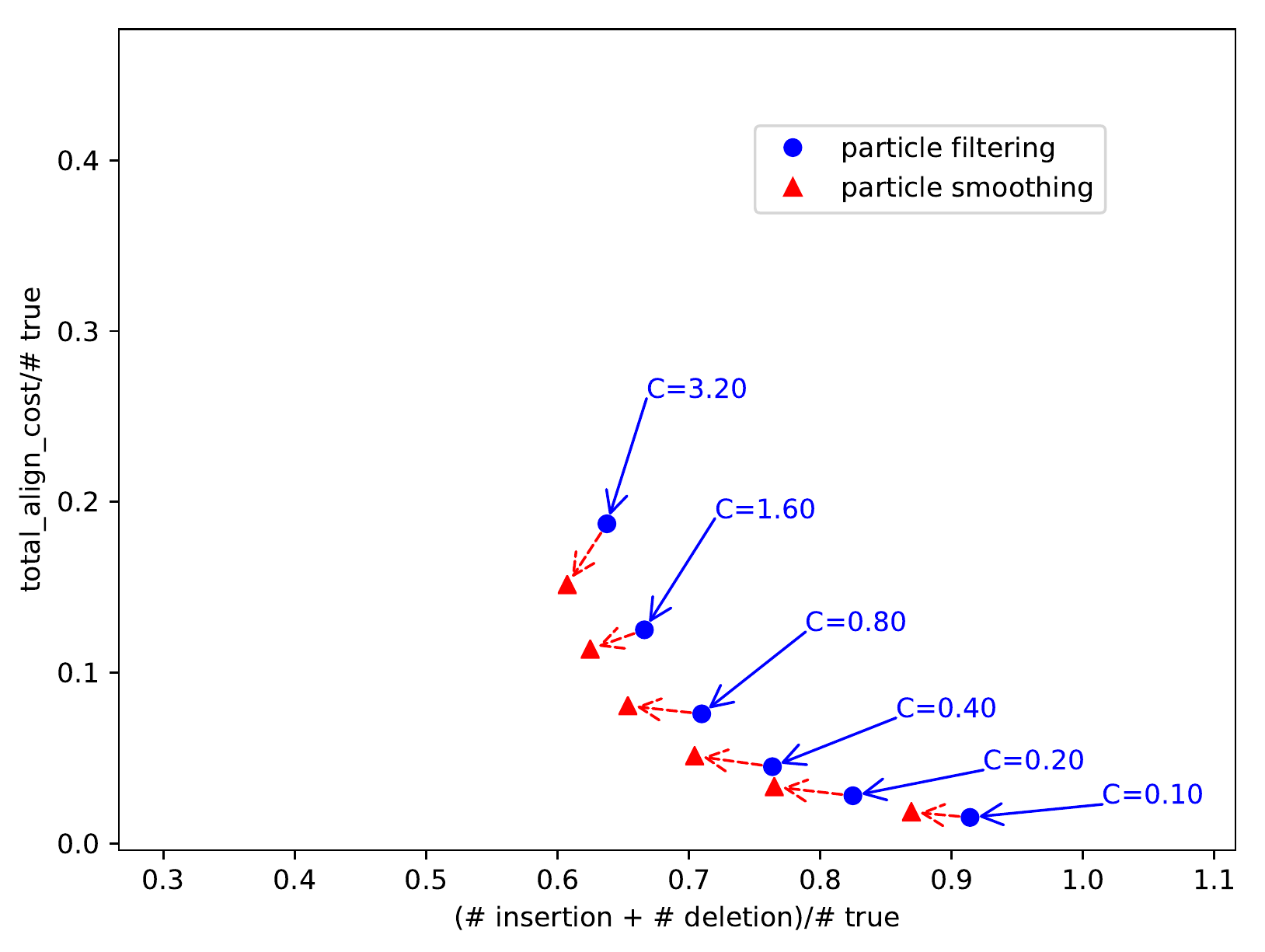}
			\includegraphics[width=0.99\linewidth]{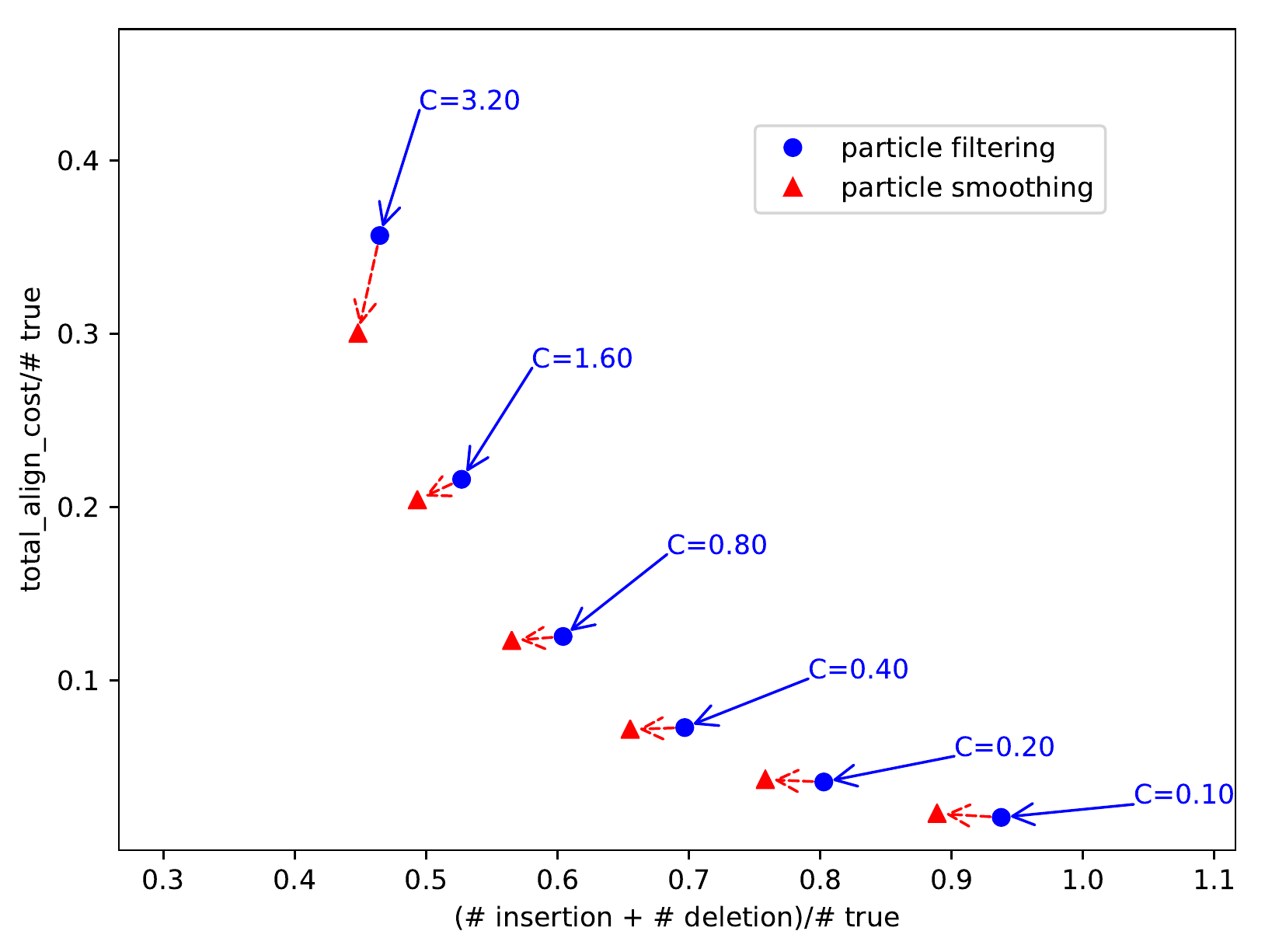}
			\includegraphics[width=0.99\linewidth]{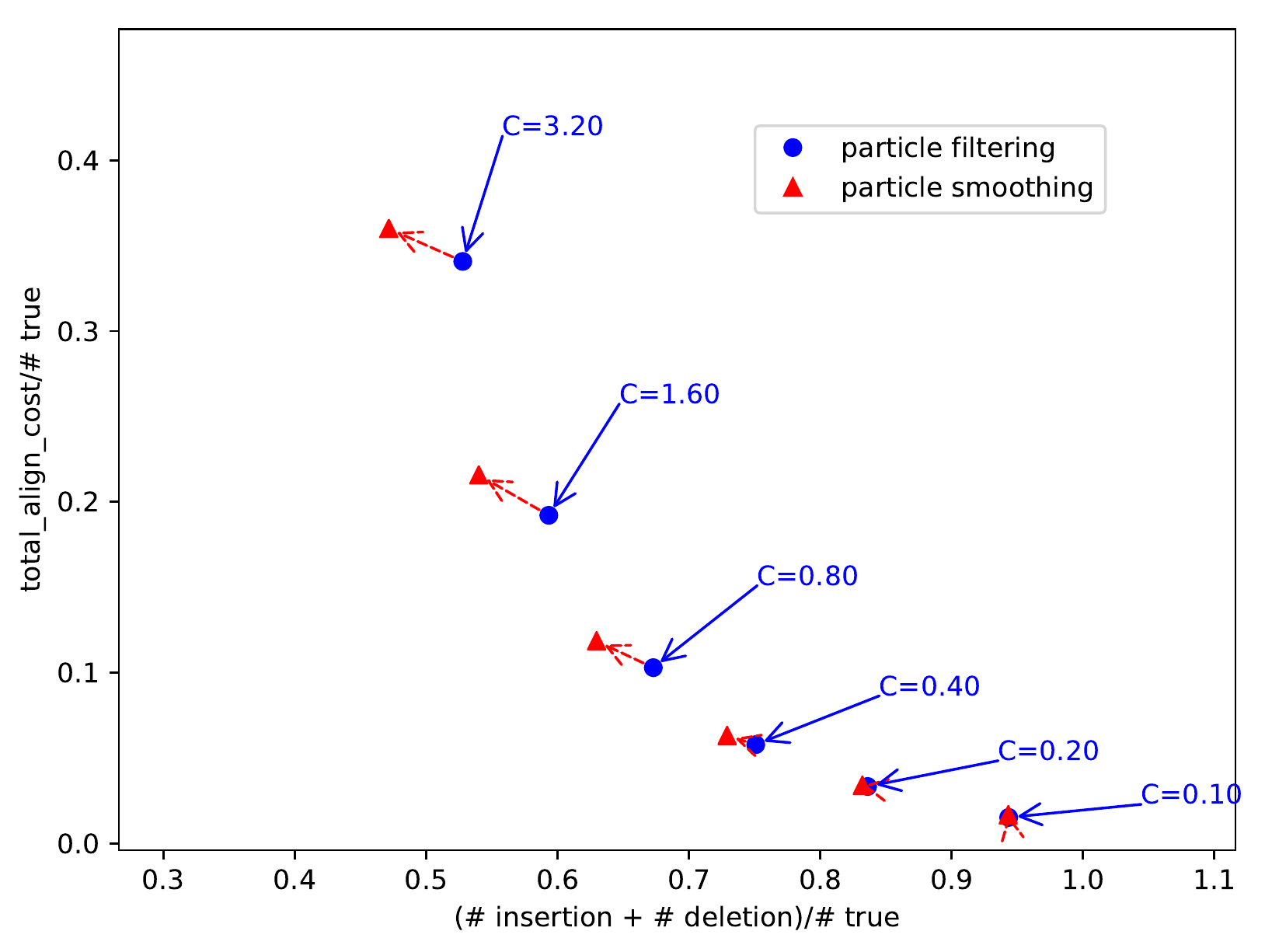}
			\caption{NYC Taxi}\label{fig:rhos_taxi}
		\end{subfigure}
		\caption{
			Optimal transport distance results with varying missingness rate $\pmi$. Rows (top-down) are results with $\pmi= 0.1, 0.3, 0.5, 0.7, 0.9$. As we can see, our particle smoothing consistently outperforms the filtering baseline with different $\pmi$, although no clear trend with increasing $\pmi$ is found on either dataset.
		}\label{fig:diff_rhos}
	\end{center}
\end{figure}

\subsection{Wall-Clock Runtime Details}\label{sec:wallclock}

A given run of particle smoothing begins by drawing $O(I)$ time points from $\Uniform([0,T))$, where $I$ is the number of observed events. All particles are evaluated using integrals that are estimated by evaluating the function at these time points (\cref{sec:integral}).  

The theoretical runtime complexity is $O(MI)$ because drawing a particle requires $O(I)$ time---the outer loop over time steps (\cref{line:sis:outer} of \cref{alg:sis})---and we draw $M$ particles in total---the inner loop over particles (\cref{line:sis:inner} in \cref{alg:sis}).
Our GPU implementation (which we will release) parallelizes the inner loop over particles. We sample 50 particles in parallel in these experiments, but we have tested with 1000 particles in parallel as well. So this is not a real problem with a GPU.

We reported experiments that we performed to demonstrate the practicality.  
On average, drawing an ensemble of 50 particles takes 5 seconds per example on the synthetic datasets (average length 15 events), 12 seconds per example on the NYC Taxi dataset (average length 32 events) and 100 seconds per example on the Elevator System dataset (average length 313 events)---that is, 300-400 milliseconds per event. Such speeds are acceptable in many incomplete data applications, compared to the cost of collecting complete data---all the applications in \cref{sec:intro} involve real-time decision making at a human timescale.

\section{Monte Carlo EM}\label{sec:mcem}

We normally assume (\cref{sec:train}) that some complete sequences are available for training the neural Hawkes process models.  If incomplete sequences are also available, our particle smoothing method can be used to (approximately) impute the missing events, which yields additional complete sequences for training.  Indeed, if we are willing to make a MAR assumption \citep{little-rubin-1987}, then we can do imputation without modeling the missingness mechanism.  Training on such imputed sequences is an instance of \defn{Monte Carlo expectation-maximization (MCEM)} \citep{dempster-77-em,wei-90-mcem,mclachlan-07-em}, with particle smoothing as the Monte Carlo E-step, and makes it possible to train with incomplete data only.  

In the more general MNAR scenario, we can extend the E-step to consider the not-at-random missingness mechanism (see \cref{eqn:weight_b} below), but then we need both complete and incomplete sequences at training time in order to fit the parameters of the missingness mechanism (unless these parameters are already known) jointly with those of the neural Hawkes process.  Although training with incomplete data is out of the scope of our experiments, we describe the methods here and provide MCEM pseudocode.

In this case, we would like to know the (marginal) probability of the observed data $\obs$ under the target distribution $p$:
\begin{align}
p(\obs)
= \sum_{\unobs} \model(\comp) \pmiss(\unobs \mid \comp)
\end{align}
If we propose $\unobs$ from $\proposal$, then it can be rewritten as:
\begin{subequations}
\begin{align}
p(\obs)
&= \sum_{\unobs} \model(\comp) \pmiss(\unobs \mid \comp) \frac{\proposal}{\proposal} \\
&= \E[\unobs \sim \proposal]{ \frac{\model(\comp)\pmiss(\unobs \mid \comp)}{\proposal}}
\end{align}
\end{subequations}

Given a finite number $M$ of proposed particles $\{\unobs_m\}_{m=1}^{M}$, this expectation can be estimated with empirical average:
\begin{align}
p(\obs)
= \frac{1}{M} \sum_{m=1}^{M} \frac{ \model(\comp_m ) \pmiss(\unobs_m \mid \comp_m) }{q(\unobs_m \mid \obs )}
\end{align}
and it is obvious that
\begin{subequations}
\begin{align}
  \log p(\obs) 
  &\geq \frac{1}{M} \sum_{m=1}^{M} ( b_m - \log{q(\unobs_m \mid \obs} ) )\label{eqn:elbo} \\
  b_m &= \log{\model(\comp_m)} + \log \pmiss(\unobs_m \mid \comp_m)
\end{align}
\end{subequations}
where the right-hand-side (RHS) term of \cref{eqn:elbo} is the \defn{Evidence Lower Bound (ELBO)} that we would maximize in order to maximize the log-likelihood.

The MCEM algorithm is composed of two steps:
\begin{description}
\item[E(xpectation)-step] We train the proposal distribution $\proposal$ using the method in \cref{sec:train} and then sample $M$ weighted particles from $\proposal$ by calling \cref{alg:sis}.
\item[M(aximization)-step] We train the neural Hawkes process $\model(\comp)$ by maximizing the ELBO (\cref{eqn:elbo}).
\end{description}

Note that in the MAR case, $\pmiss(\unobs \mid \comp)$ is constant of $\unobs$ so the it can be omitted from the formulation (and thus the algorithms).
Also note that, for particle filtering, the proposal distribution $\proposal$ is only part of $\model(\comp)$ so we do not need to train $\proposal$ at the E-step.

Maximum-likelihood estimation remains sensible in the MNAR case provided that we know one of the distributions $\model$ or $\pmiss$, in which case we can use EM to estimate the other distribution.  

(1)~If $\pmiss$ is known and fixed, as in our experiments, this gives a minor variant of ordinary EM. Ordinary EM makes the MAR assumption that the $\pmiss$ factor of \cref{eqn:target} can be ignored.  However, if we know $\pmiss$, we can incorporate it rather than ignoring it; then it need not satisfy the MAR assumption.

(2)~Conversely, if $\model$ is known and fixed because we estimated it from a sufficient quantity of {\em complete} data, then we can use incomplete data to learn the MNAR missingness distribution $\pmiss$. This setting would even lets us learn contextual missingness mechanisms in which the probability that an event is censored depends not only on the event itself, but also on the surrounding events and whether they are censored.  For example, one could try to fit $\pmiss$ with an LSTM model or a BiLSTM-CRF model \cite{huang2015bidirectional} that performs structured joint prediction of the missingness of all events in the sequence.  Extending that method to use continuous-time LSTMs would allow it to take timing into account.

The E step of Monte Carlo EM uses the current guesses of $\model$ and/or $\pmiss$ to sample from the posterior distribution $p(\unobs\mid\obs)$ of the missing values. That posterior is uncontroversially defined by the simple Bayesian formula \eqref{eqn:target}.  Notice that even if $\model$ and $\pmiss$ were \emph{both} unknown, we could still run MCEM to locally maximize the likelihood $p(\obs)$, but unfortunately the parameters would be unidentifiable in this case.  Thus, there would be many missing-data models with the same likelihood, as explained in \cref{sec:miss_details}, and they would make different predictions of $\unobs$.

\end{document}